%% file: main.tex
\documentclass{article} \usepackage{c_style,times}

\usepackage{hyperref}
\hypersetup{
  colorlinks=true,
  linkcolor=blue,
  citecolor=blue,
  urlcolor=blue
}
\usepackage{url}
\usepackage{amsmath}
\usepackage{amssymb}
\usepackage{amsthm}
\usepackage{graphicx}
\usepackage{cleveref}
\usepackage{xcolor}
\usepackage{enumitem}
\usepackage{booktabs}
\usepackage{algorithm}
\usepackage{algpseudocode}
\usepackage{float}
\usepackage{subcaption}
\usepackage{wrapfig}
\usepackage{listings}

\lstdefinestyle{lean}{
  basicstyle=\ttfamily\footnotesize,
  backgroundcolor=\color{black!6},
  frame=none,
  breaklines=true,
  columns=fixed,
  basewidth=0.5em,
  keepspaces=true,
  showstringspaces=false,
  xleftmargin=0.5em,
  framexleftmargin=0.5em,
  aboveskip=0.6em,
  belowskip=0.6em,
  literate={∧}{{$\wedge$}}1 {∨}{{$\vee$}}1 {∀}{{$\forall$}}1 {∃}{{$\exists$}}1
           {→}{{$\to$}}1 {←}{{$\leftarrow$}}1 {↔}{{$\leftrightarrow$}}1
           {¬}{{$\lnot$}}1 {≠}{{$\neq$}}1 {≤}{{$\le$}}1 {≥}{{$\ge$}}1
           {∈}{{$\in$}}1 {∉}{{$\notin$}}1 {⟨}{{$\langle$}}1 {⟩}{{$\rangle$}}1
           {·}{{$\cdot$}}1 {∅}{{$\emptyset$}}1 {ℓ}{{$\ell$}}1 {…}{{...}}3
           {⊢}{{$\vdash$}}1 {↓}{{$\downarrow$}}1
}
\lstnewenvironment{leancode}{\lstset{style=lean}}{}

\usepackage{tcolorbox}
\tcbuselibrary{skins,breakable,listings}
\definecolor{promptframe}{gray}{0.32}
\definecolor{promptback}{gray}{0.965}
\lstdefinestyle{promptstyle}{
  basicstyle=\ttfamily\scriptsize,
  backgroundcolor=\color{promptback},
  numbers=none,
  frame=none,
  breaklines=true,
  breakatwhitespace=false,
  breakindent=0pt,
  columns=fullflexible,
  keepspaces=true,
  showstringspaces=false,
  literate={-}{{-\allowbreak}}1 {§}{{\S}}1 {—}{{\textemdash}}1 {–}{{\textendash}}1
           {→}{{$\rightarrow$}}1 {⇒}{{$\Rightarrow$}}1 {≥}{{$\geq$}}1 {≤}{{$\leq$}}1
           {×}{{$\times$}}1 {·}{{$\cdot$}}1 {…}{{...}}3 {↔}{{$\leftrightarrow$}}1
           {∀}{{$\forall$}}1 {✝}{{$\dagger$}}1 {⟨}{{$\langle$}}1 {⟩}{{$\rangle$}}1
}
\newtcbinputlisting{\promptbox}[2]{
  enhanced, breakable, listing only,
  listing file={#2},
  listing options={style=promptstyle},
  title={#1},
  fonttitle=\bfseries\large, coltitle=white,
  colframe=promptframe, colback=promptback,
  boxrule=1.2pt, arc=2mm,
  left=5mm, right=5mm, top=2mm, bottom=2mm,
}

\newcommand{\ToolName}{CoCo-Prover}
\newcommand{\ToolNameIt}{\textit{\ToolName{}}}
\newcommand{\CostReduction}[2]{\textbf{#1}\,\textcolor{green!60!black}{\ensuremath{\blacktriangledown\,#2\%}}}
\newcommand{\CostIncrease}[2]{\textbf{#1}\,\textcolor{red}{\ensuremath{\blacktriangle\,#2\%}}}

\title{Cost-Efficient Theorem Proving via Agent Orchestration in Program Verification}

\firstpageheader{Preprint.}

\author{\authorname{Shuangjie Yao}{1,3}\And
\authorname{Nikolaus Holzer}{1}\And
\authorname{Mark Paul Santolucito}{2}\AND
\authorname{Baishakhi Ray}{1}\And
\authorname{Suman Jana}{1}\And
\authorname{Dongdong She}{3}
}

\affiliations{\affiliation{1}{Columbia University}\And
\affiliation{2}{Barnard College, Columbia University}\\[0.4em]
\affiliation{3}{The Hong Kong University of Science and Technology}
}

\newcommand{\ie}{\emph{i.e.},\ }
\newcommand{\eg}{\emph{e.g.},\ }

\newtheorem{proposition}{Proposition}[section]
\newtheorem{lemma}[proposition]{Lemma}

\crefname{proposition}{proposition}{propositions}
\Crefname{proposition}{Proposition}{Propositions}
\crefname{lemma}{lemma}{lemmas}
\Crefname{lemma}{Lemma}{Lemmas}
\crefname{remark}{remark}{remarks}
\Crefname{remark}{Remark}{Remarks}

\begin{document}

\maketitle

\input{chapters/00_abstract.tex}

\input{chapters/01_introduction.tex}

\input{chapters/02_related_work.tex}

\input{chapters/04_methodology.tex}

\input{chapters/11_evaluation.tex}

\input{chapters/21_conclusion.tex}

\bibliography{miscellaneous/reference}
\bibliographystyle{c_style}

\appendix
\input{appendix/theory.tex}

\input{appendix/operational_details.tex}

\input{appendix/additional_results.tex}

\input{appendix/case_study.tex}

\input{appendix/prompts.tex}

\input{appendix/limitations.tex}

\end{document}

%% file: chapters/00_abstract.tex
\begin{abstract}
    Program verification establishes software correctness through machine-checkable proofs constructed in theorem provers. It's a guarantee especially valuable for code generated by large language models (LLMs), which is fluent but carries no assurance of correctness.
    Almost all existing provers, however, pursue pass rates alone at whatever sampling or search budget it takes, and overlook the \emph{success-vs-cost frontier}; yet real software often carries hundreds of interdependent proof obligations, so what matters at scale is not whether one theorem can be proved, but how many can be proved economically.
    We introduce \ToolNameIt{}, which formalizes cost-efficient program proving as \emph{metalevel decision-making under cost, grounded on two-level proof graphs}: an AND/OR proof hypergraph within each declaration is joined to a lemma-dependency graph across declarations; and at each step, it answers two questions: which open goals to select, and which actions to purchase on these goals.
    Selection stays symbolic as a topological pass over the proof graphs.
    Action choice is agent orchestration via metalevel decision-making: an agentic router treats every bounded specialist invocation as a separately priced, best-effort computation, matching heterogeneous specialist agents together with configurations, under evolved routing rules as evidence accumulates.
    On five program verification benchmarks in Lean 4 including function-level CLEVER, VERINA, and AlgoVeri, and repository-level NTP4VC and Vero, we show that \ToolName{} achieves a better success-vs-cost frontier than baselines including frontier coding agents and state-of-the-art LLM-based provers: it achieves the best solve rate on every benchmark and up to 100\% on two benchmarks. It also reduces cost by up to 30.9\% compared to the strongest baseline with the strongest LLM in our evaluation.
\end{abstract}

%% file: chapters/01_introduction.tex
\section{Introduction}
\label{sec:introduction}

Program verification uses formal mathematical reasoning to establish that a program satisfies specified correctness properties~\cite{floyd1967assigning,hoare1969axiomatic}, providing guarantees that testing, which examines only finitely many executions, cannot offer. Historically, obtaining such guarantees has been expensive, requiring substantial expert effort to construct machine-checkable proofs in interactive theorem provers (ITPs) such as Lean~4~\cite{demoura2021lean4}, Rocq~\cite{rocq2025manual}, and Isabelle/HOL~\cite{nipkow2002isabelle}. Large language models (LLMs) create an opportunity to reduce this burden: while they can now generate sophisticated programs from natural-language specifications~\cite{chen2021evaluatingllm, alharbi2025acgt}, they may also automate the costly process of constructing formal proofs.
This possibility has motivated \emph{verified code generation}~\cite{sun2024clover, yang2024formal}, where a model generates a program with a machine-checkable proof proving it satisfies the specification. Nevertheless, proof generation remains considerably harder than code generation shown in recent benchmarks VERINA~\cite{ye2026verina}, CLEVER~\cite{thakur2025clever}, and AlgoVeri~\cite{zhao2026algoveri}. Proof generation has therefore emerged as a central bottleneck in program verification and verified code generation.

Real software magnifies this bottleneck.
A repository can contain hundreds or thousands of proof obligations with sharply different proof economics~\cite{xin2026verisoftbench,ye2026vero}.
At the same time, a generated proof serves only to verify its program: the checker returns exactly one of two outcomes, and every accepted proof certifies equally.
Thus, the natural objective for scalable program verification is the \emph{success--vs--cost frontier}---the Pareto frontier between proof success and token cost (\emph{or} monetary expenditure practically), measured as success under matched cost limits and cost at matched success.

Recent LLM-based theorem provers have shown substantial promise, particularly in mathematical reasoning~\cite{ren2025dsproverv2,lin2026goedelproverv2,achim2025aristotle,sumanth2026hilbert,chung2026goedelarchitect}. Code proving, however, presents a different challenge~\Cref{fig:moti-proving}. Each task introduces a fresh formal context (\eg the implementation, specification, and often new data types and definitions) and proving correctness requires reasoning about the program's intermediate states. While program syntax often determines the overall proof structure and many side conditions can be discharged mechanically, a few critical obligations require discovering facts not explicit in the program or specification (\eg the right loop invariant, a sufficiently generalized induction hypothesis). These discoveries often determine whether the entire proof succeeds, yet they are precisely the steps that require substantial semantic reasoning.

This asymmetry creates a distinctive opportunity for code proving: most obligations are routine, but a small number are both difficult and decisive. Existing code provers~\cite{he2026nspg,tu2026agentic,zhao2026lemmanet,li2026goedelcodeprover}, however, largely ignore this asymmetry, pursuing provability through large sampling or iterative repair budgets and applying similar model-driven procedures across subgoals. Such strategies are particularly costly with frontier general-purpose LLMs and waste substantial computation on obligations that theorem-prover automation or inexpensive model calls could solve directly (\Cref{fig:motivation}). Effective code proving therefore requires not simply more reasoning, but identifying where deep reasoning matters and concentrating model computation there.

\input{figures/motivation.tex}

To address this challenge, we introduce \ToolName{}, a cost-aware prover that adapts its strategy and reasoning effort to the evolving proof state. Rather than applying the same proof routine to every obligation, \ToolName{} repeatedly decides what to prove next and how to prove it: whether an obligation can be discharged mechanically, needs direct proof generation, benefits from introducing a reusable helper lemma, or requires recovery from an unproductive proof path. It also controls how much reasoning effort to spend on each action. To make these decisions globally, \ToolName{} maintains a two-level proof graph that tracks both fine-grained proof progress and dependencies among lemmas.
We formalize this process as \emph{metalevel decision-making under cost}: each proving action is a computation with a cost and an uncertain, verifier-judged outcome, and the prover must decide which computation is worth performing given the current proof state~\cite{russell1991principles}. \ToolName{} realizes this formulation through an agentic router and four specialized proving agents, using the proof graph and accumulated verifier feedback to adapt its strategy as proving progresses. This turns code proving from a fixed proof-search routine into adaptive orchestration of heterogeneous proving strategies under cost.

We evaluate \ToolName{} on five Lean program verification benchmarks (function-level CLEVER~\cite{thakur2025clever}, VERINA~\cite{ye2026verina}, and AlgoVeri~\cite{zhao2026algoveri}, and repository-level NTP4VC~\cite{xu2026ntp4vc} and Vero~\cite{ye2026vero}) against state-of-the-art open-source agentic provers, and each model's official coding agent, with three frontier LLMs under matched caps.
No baseline exceeds \ToolName{}'s solve rate on any benchmark and backbone: it solves up to 100\% of function-level benchmarks and 60.5\% of the large repository-level benchmark Vero while reduces cost by up to 30.9\% against the strongest baseline, Humanize on the strongest backbone LLM, GPT-5.6-Terra.

In summary, our contributions are:
i) \textbf{Cost-efficient code proving.} We pose cost efficiency, rather than provability alone, as the objective for theorem proving in program verification at the scale of real-world software.
ii) \textbf{Metalevel formulation under cost over two-level proof graphs.} We are the first to formalize cost-efficient program proving as metalevel decision-making under cost, grounded on two-level proof graphs. And we instantiate it in \ToolName{} with agent orchestration.
iii) \textbf{Success--cost evaluation.} On five Lean benchmarks and three LLMs, no baseline exceeds \ToolName{}'s solve rate in any setting, and it reduces cost by up to 30.9\% against the strongest baseline.

%% file: figures/motivation.tex
\begin{figure}[t]
    \centering
    \begin{subfigure}[b]{0.55\linewidth}
        \centering
        \includegraphics[height=5.15cm]{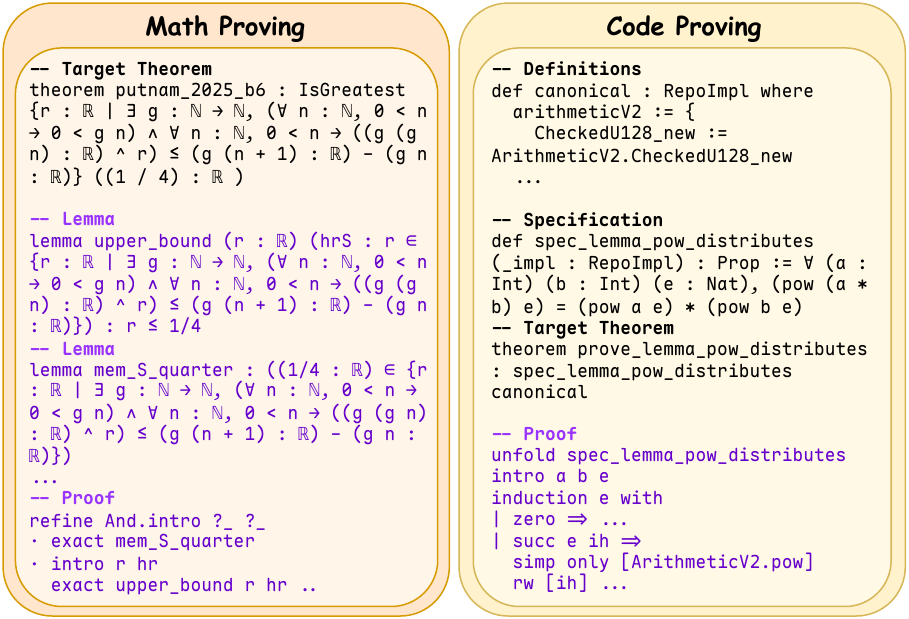}
        \caption{Math proving vs.\ code proving.}
        \label{fig:moti-proving}
    \end{subfigure}
    \hfill
    \begin{subfigure}[b]{0.44\linewidth}
        \centering
        \includegraphics[height=5.15cm]{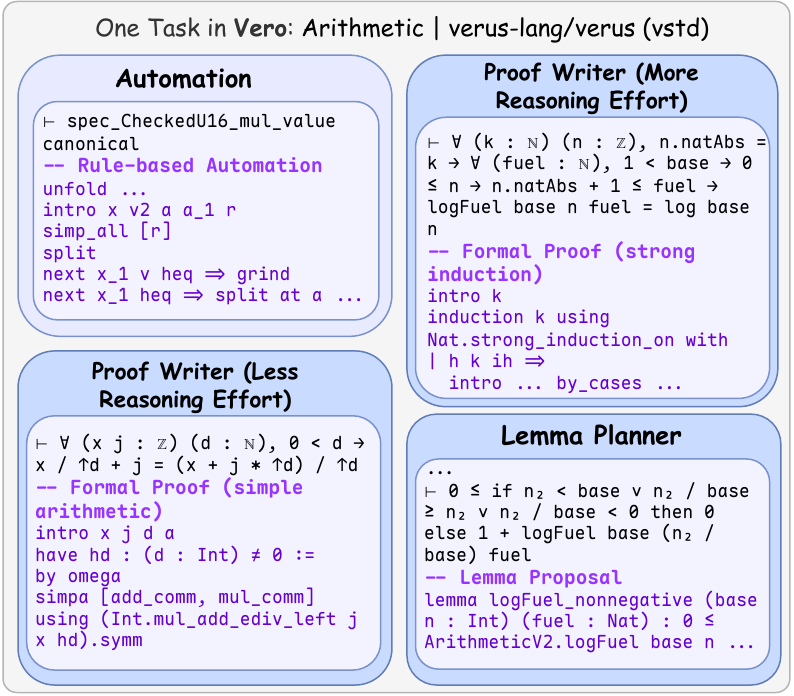}
        \caption{Goals within one code-proving task.}
        \label{fig:moti-diff}
    \end{subfigure}
    \caption{\textbf{The difference between math proving and code proving; and the difference between goals within one code proving task.}
        (\subref{fig:moti-proving}) A theorem in math proving (theorem from PutnamBench~\cite{tsoukalas2024putnambench}, proof from GoedelArchitect~\cite{chung2026goedelarchitect}) is stated and proved in the vocabulary of an existing library and strongly dependent on lemma decompostion, whereas a theorem in code proving (theorem from Vero~\cite{ye2026vero}) is stated over definitions the task itself introduces so the proof follows program syntax rather than library priors.
        (\subref{fig:moti-diff}) Obligations of a single Vero task need different strategies and reasoning efforts: one goal closes under rule-based automation with no model call, two need formal proofs at different reasoning efforts (simple arithmetic versus strong induction), and a recurring fact is better stated as a shared helper lemma.}
    \label{fig:motivation}
    \vspace{-0.9em}
\end{figure}

%% file: chapters/02_related_work.tex
\section{Related Work}
\label{sec:related-work}

\textbf{LLM-based provers.}
Fine-tuned provers generate whole Lean proofs and buy success through \emph{distributional sampling}: correct proofs are spread across many samples and extracted at pass@$k$ budgets by redrawing on unchanged context~\cite{lin2025goedelprover,wang2025kiminaproverpre,lin2026goedelproverv2}, with decomposition, where present, confined to a single generation's chain of thought~\cite{ren2025dsproverv2}; refinement agents temper this with verifier feedback yet still spend fixed per-theorem iteration budgets~\cite{thakur2024incontextlearningagentformal,axproverbase2026,zhou2025solving,ma2026oprover,liu2026numina}.
\ToolName{} instead treats every bounded invocation as best-effort given its evidence and state, and resamples only as a separate decision justified by changed evidence.
Structured provers including stepwise search over AND/OR tactic states~\cite{lample2022htps,hubert2026alphaproof,limperg2023aesop} and lemma decomposition through pools, sketches, or blueprints~\cite{chen2025seedprover,chen2025seedprover15,achim2025aristotle,sumanth2026hilbert,chung2026goedelarchitect,zhang2026planning} suffer from \emph{coarse rollback or rare reuse}: a falsified conjecture triggers wholesale sketch or blueprint regeneration or outright branch failure, the rollback unit being the whole attempt rather than exactly the invalidated dependents, and helpers are rarely shared across distinct targets.
\ToolName{}'s persistent two-level proof graph instead provide fine-grained control over the rollback unit and the helper reuse.
Code provers inherit both modes: NSPG runs one stepwise search per obligation under a flat timeout~\cite{he2026nspg}, AutoRocq maintains a per-lemma proof tree~\cite{tu2026agentic}, LemmaNet discovers helper lemmas from program semantics into a flat corpus~\cite{zhao2026lemmanet}, and Goedel-Code-Prover decomposes each verification task in isolation~\cite{li2026goedelcodeprover}, while the general coding agents merge these techniques but cost much~\cite{sosso2026agenticproving,thakur2025clever}.

\textbf{Cost-aware inference and agent orchestration.}
Query-level routers and cascades choose a model from surface features of an isolated query~\cite{chen2023frugalgpt,ong2024routellm}, and budget-aware tree search anneals exploration within a single query's reasoning tree under an LLM critic~\cite{li2026budgetaware}.
In the area of formal reasoning and theorem proving: Optimizing-Lean-Agent~\cite{kari2026optimizingcost} aims for math proving by routing over a binary action space using a logistic model over surface features of the failed attempts. This coarse continue-or-restart routing selects no goal and no proof strategy, and never conditions on decomposition structure or on other obligations. EconProver~\cite{li2026econprover} also targets math proving, but on the training side: reinforcement learning teaches a single prover to switch reflective chain-of-thought off on easy theorems and to diversify its parallel samples, so they decide neither which obligation to work on nor which kind of action to buy. Danus~\cite{liu2026danus} maintains a fact graph with transitive revocation for mathematical research agents but routes work to homogeneous general-purpose agents rather than specialists matched to goal types and makes no cost decision over that structure.

%% file: chapters/04_methodology.tex
\section{Methodology}
\label{sec:methodology}

\input{chapters/03_problem_formulation.tex}

\input{figures/agents.tex}

\subsection{System Design of \ToolName{}}
\label{sec:architecture-control-loop}

Guided by the metalevel analysis of \Cref{sec:problem-formulation}, \ToolName{} is an agent framework built around one principle: LLM agents make the semantic decisions and dispatch the decisions to the specialist agents as shown in \Cref{fig:router-agent-loop}.
\ToolName{} maintains the two-level proof graph together with a shared memory of artifact status, failure records, and the realized resource ledger.
A \emph{scheduler} selects the frontier goals on AND nodes symbolically according to discussion in \Cref{sec:problem-formulation}.
An \emph{agentic router} reasons about the selected goals and dispatches them to one of four specialist action agents. Each action agent writes its verified transitions, diagnostics, and realized cost back to the graph and shared memory.

\textbf{Evidence-grounded agentic router.}\label{sec:agent-routing}
The scheduler hands the router the whole dependency-ready frontier of open AND nodes. The router first groups it using $k$-means clustering as described in \Cref{app:goal-clustering}, so that semantically related goals are routed together.
A static-analysis check then runs before any model call: a goal that is \emph{definition-nested} is sent directly to Deterministic Automation.
At each decision point it receives an \emph{evidence packet}, instantiating $\mathrm{ev}(s,U)$: the selected Lean goals with their tactic states and relevant definitions, snapshots of earlier proof states, the dependent provisional lemmas, and the execution history with Lean diagnostics and counterexamples from the shared memory, together with its current routing rules.
The router approximates the metalevel ordering of \Cref{eq:metalevel-bellman} in the small-$\lambda$ regime of \Cref{prop:small-lambda}.
From this packet it returns $(G,o,\theta,j)$: the group $G$ of the goals it decides to spend on next, drawn from its own clustering of the frontier; a specialist option $o$; a configuration $\theta$ which sets the invocation's reasoning effort and working budget (\eg the maximum number of refinement iterations); and the preliminary analysis $j$ it produces, so that the action agents can continue the reasoning from it.
The routing rules are initialized with one priority, \emph{same provability, cheap-first}: among the actions equally likely to leave every remaining root provable, take the one with the lowest total expected cost of finishing them all. Both terms are task-level rather than goal-level, which means the comparison is predicted over the whole remaining task. This explains why proposing a helper lemma can be the cheapest choice when several goals will use it, even though proving it once costs more than one direct proof.
The initial rules are deliberately simple: a short natural-language description of each specialist including what kind of goal it suits together with the priority and a few worked routing examples.
The rules then \emph{evolve}: after observing an evidence packet and the outcome of its previous decision, the router may amend them. For instance, if the Proof Writer repeatedly fails on a kind of goal, the rules may evolve to increase reasoning effort or propose a shared helper lemma to decompose the goal.

\textbf{Specialist agents.}\label{sec:specialist-agents}
The four specialists exist to handle different kinds of proof goals.
\emph{\textbf{Deterministic Automation:}} this action doesn't involve any model call. It extends Aesop~\cite{limperg2023aesop}, a white-box rule-based proof search methods, which automatically tries different tactics in Lean to solve the current goal.
\emph{\textbf{Proof Writer:}} it's an agent interacting with Lean checker. It first continue on the reasoning results from the router and then write formal Lean code under a iterative refinement budget from the router. When the budget is exhausted, it keeps the partial successful proof and update the proof graphs.
\emph{\textbf{Lemma Planner:}} it proposes typed helper lemmas with explicit root-relevant consumers. A candidate is added to the dependency graph only after it elaborates and preserves acyclicity, it is checked by both random testing and LLM-based structural checking. Its proof obligations enter the frontier as new open goals, and control returns to the router with analysis tags.
\emph{\textbf{Recovery Reviewer:}} it diagnoses why an attempt failed, distinguishing two cases. First, the goal may be \emph{unprovable}: a wrong proof step or an incorrect helper lemma can leave a goal that is false, so the reviewer constructs concrete instances and evaluates them to look for a counterexample to disprove the goal. Second, the goal may be provable but on a \emph{hard path}: the current route has grown expensive while a cheaper one likely exists, so the reviewer compares continuing with switching on the work each still requires: the expected remaining cost, revalidation included, and whether either route looks less likely to close at all. The judgment is the model's rather than a certified optimality test, since a cheaper route can also be a less provable one. When it decides to recover, it backtracks to the best point in the proof graphs (\ie the nearest state whose verified work does not depend on the failure) so only the invalidated work is discarded.

%% file: chapters/03_problem_formulation.tex
\subsection{Problem Formulation}
\label{sec:problem-formulation}

We formulate cost-efficient program proving as \emph{metalevel decision-making under cost}.
At each step, the prover must determine (1) which goals to work on next, and (2) what proving strategy and configuration to use.
The first is structural and the latter one requires semantic reasoning and incur model cost.
Importantly, cost is considered globally at the task level: an expensive action on one goal may reduce total cost, for example by producing a helper lemma reused by several others.
We measure cost as monetary expenditure rather than token count, since different token types are priced differently.

\input{figures/proof_graphs.tex}

\textbf{Proof state and structural control.}
Given a fixed Lean environment $\Gamma_0$ and target theorems
$L_\star=\{\ell_1,\ldots,\ell_n\}$, the prover maintains the two-level proof graph in
\Cref{fig:coupled-proof-graph}.
Within each declaration, an AND/OR hypergraph records verified proof steps and their resulting subgoals; across declarations, a dependency graph connects target and helper lemmas.
Together, they preserve partial progress, expose reusable results, and track which work depends on provisional lemmas.

At decision epoch $t$,
\begin{equation}
  \begin{aligned}
    s_t & = (\mathfrak G_t,\mathcal F_t,\sigma_t,\mathcal M_t,\mathcal L_t), \\
    U_t & = \mathrm{Sched}(s_t),
  \end{aligned}
  \label{eq:controller-state-action}
\end{equation}
where $\mathfrak G_t=(\mathcal D_t,\{\mathcal H_{\ell,t}\}_{\ell\in L_t})$
is the two-level graph, $\sigma_t$ records artifact status,
$\mathcal M_t$ is shared memory, $\mathcal L_t$ is the computation ledger,
and $\mathcal F_t$ is the \emph{dependency-ready frontier}: open goals whose active prerequisites have been constructed.
A lemma becomes stable once all of its prerequisites are stable; goals depending on provisional lemmas may still proceed, with those dependencies tracked for rollback (\Cref{app:sharing-rollback}).
An incremental topological pass in the style of Kahn's algorithm~\cite{kahn1962topological}
maintains $\mathcal F_t$ as lemmas are added, stabilized, or retracted.
The scheduler exposes $U_t=\mathcal F_t$, restricting the action space without spending model computation.

\textbf{Action selection as purchased computation.}
Given $U_t$, the router decides \emph{where and how to spend computation}: it selects a group of goals, a proving strategy, and its reasoning budget.
Following metareasoning principles~\cite{russell1991principles}, we represent this as
\begin{equation}
  c_t=(G_t,o_t,\theta_t)\in\mathcal C(s_t,U_t),
  \qquad
  Y_t\sim P_{c_t}(\cdot\mid s_t,U_t),
  \qquad
  s_{t+1}=T(s_t,U_t,c_t,Y_t),
  \label{eq:purchased-computation}
\end{equation}
where $G_t\subseteq U_t$ is the selected goal group, $o_t$ is a specialist agent,
and $\theta_t$ specifies its configuration, such as reasoning effort or refinement budget.
The outcome $Y_t$ incurs cost
$\kappa_t=\kappa(s_t,U_t,c_t,Y_t)$.
The router's own inference cost is accounted for separately.

\textbf{Metalevel objective.}
The cheapest immediate action is not necessarily the most cost-efficient: a more expensive computation may increase the chance of completing the task or reduce future work.
For $\lambda\geq0$, we therefore score a run by
$\mathbf 1[\mathrm{solve}]-\lambda C$, where $C$ is total realized cost.
The optimal value satisfies
\begin{equation}
  Q_\lambda^*(s,U,c)=\mathbb E_Y\bigl[-\lambda\kappa+V_\lambda^*(s')\bigr],
  \quad
  V_\lambda^*(s)=\max_{c\in\mathcal C(s,U)}Q_\lambda^*(s,U,c),
  \label{eq:metalevel-bellman}
\end{equation}
writing $\kappa=\kappa(s,U,c,Y)$ and $s'=T(s,U,c,Y)$, with $U=\mathrm{Sched}(s)$, $V_\lambda^*=1$ on solved states, and a free stop action of value $0$.
This objective values a computation by both its immediate cost and its effect on the remaining verification task.
$Q_\lambda^*$ is normative rather than computed explicitly.
The router uses the current proof state and accumulated evidence to approximate its induced preference over computations.
We operate at a small positive $\lambda$, so successful verification dominates while cost distinguishes actions with comparable prospects of success.
Finally, the regret of a routing decision is
\begin{equation}
  \delta_\lambda(s,U,c)
  =V_\lambda^*(s)-Q_\lambda^*(s,U,c)\geq0.
  \label{eq:metalevel-regret}
\end{equation}
Adaptive routing is beneficial when the decision regret it avoids exceeds its own inference overhead; we formalize this result in \Cref{app:additional-theory}.

%% file: figures/proof_graphs.tex
\begin{wrapfigure}{r}{0.62\textwidth}
    \centering
    \vspace{-1.2em}
    \includegraphics[width=0.60\textwidth]{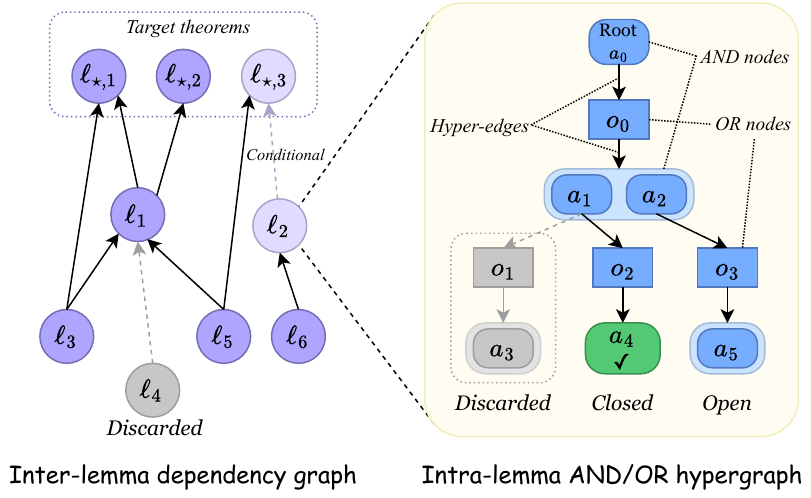}
    \caption{\textbf{The two-level proof graph.}
        \emph{Left}: the lemma-dependency DAG $\mathcal D=(L,E_{\mathcal D})$, with an edge $u\rightarrow v$ when the proof or plan of $v$ uses helper lemma $u$, the targets $L_\star\subseteq L$ as roots, and each lemma stable, provisional, or discarded.
        \emph{Right}: within one declaration, the AND/OR hypergraph $\mathcal H_\ell=(A_\ell,O_\ell,E_\ell)$ over goal states (AND nodes $A_\ell$) and proof steps (OR nodes $O_\ell$), where an executed step attaches a verified hyperedge in $E_\ell$ over the subgoals Lean returns.}
    \label{fig:coupled-proof-graph}
    \vspace{-1.0em}
\end{wrapfigure}

%% file: figures/agents.tex
\begin{figure}[t]
  \centering
  \includegraphics[width=0.98\linewidth]{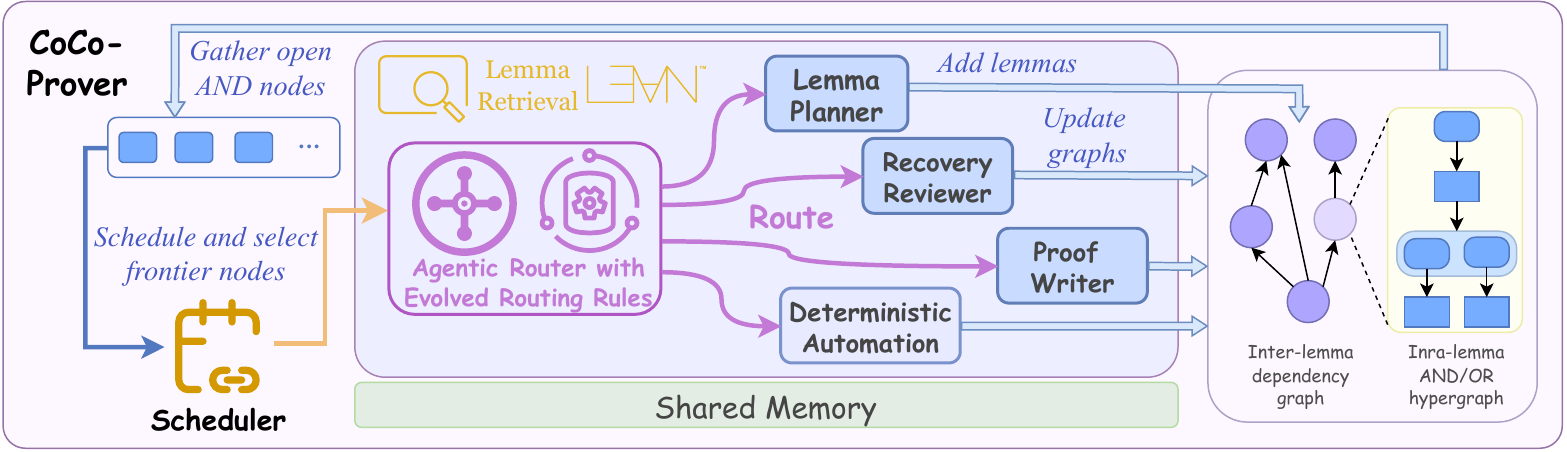}
  \caption{\textbf{Thw workflow of \ToolName{}.}
    \emph{Left}: the scheduler gathers open AND nodes from the proof graphs and hands the router every dependency-ready goal, grouping none of them.
    \emph{Middle}: the agentic router clusters the frontier it receives, then reasons about one group and routes it to one of four specialist agents together with the configuration. The routing rules are revised by the router itself as evidence packets accumulate. The agents alone interact with Lean checker, lemma retrieval, and shared memory.
    \emph{Right}: every agent writes back to the two-level graphs.}
  \label{fig:router-agent-loop}
\end{figure}

%% file: chapters/11_evaluation.tex
\section{Evaluation}
\label{sec:evaluation}

In this section, we first evaluate \ToolName{} on the performance of solve rate, cost and success-vs-cost frontier under a budget sweep. And then we conduct ablation studies to show the contribution of different components. \Cref{app:additional-results} details the experiment setup, and gives further results. Besides, we select some cases to showcase the workflow and benefits of \ToolName{} in \Cref{app:case-study}.

\textbf{Benchmarks.}
We evaluate on five Lean~4 program verification benchmarks.
Three are function-level: CLEVER~\cite{thakur2025clever}, VERINA~\cite{ye2026verina}, and AlgoVeri~\cite{zhao2026algoveri}, which provide function-level verifiable code generation problems.
Two are repository-level, where obligations are derived from upstream code repositories. NTP4VC~\cite{xu2026ntp4vc} focuses on verification conditions, generated from C and Why3 programs, and released in Lean~4, Rocq, and Isabelle, where we use the Lean 4 track. Vero~\cite{ye2026vero} is a repository-level verifiable code generation benchmark, where we use proof-only tasks.
The target theorems are constructed as follows.
VERINA and Vero ship both ground-truth specifications and ground-truth implementations, which we use the formed theorem tasks directly.
CLEVER and AlgoVeri provide ground-truth specifications but no ground-truth implementations, so we generate each implementation with GPT-5.6-Sol at high reasoning effort and manually check its correctness against the specification before any prover runs; the checked implementations are then fixed and shared by all methods.
NTP4VC's theorems are already well-formed verification conditions but it's a large benchmark so we only use the first 300 tasks in alphabetical order.
In every case the theorem statements are fixed inputs and code or specification synthesis is never a part of the measured task.

\textbf{Baselines.}
We compare against three baseline families under identical backbones, providers, and caps. We use three frontier models: Gemini-3.5-Flash, GPT-5.6-Terra, and DeepSeek-V4-Flash.
\emph{Ax-Prover-Base}~\cite{axproverbase2026} is a minimal iterative agent---propose a Lean proof, compile, review, and retry with a memory of recent failures and library search---representative of refinement agents and reported to be strongly cost-effective relative to repeated sampling.
\emph{Humanize}~\cite{humanize2026} is an open-source state-of-the-art agentic Lean~4 prover based on coding agents like Codex~\cite{codex2026} and Claude Code~\cite{cc2026}, which achieves 100\% solve rate on PutnamBench~\cite{tsoukalas2024putnambench}.
\emph{Coding agents} pair each backbone with its official general-purpose agent interface, equipped with tooling (\emph{i.e.,} Lean toolchains, retrieval)---Codex for GPT model, Gemini CLI~\cite{geminicli2026} for Gemini model, and DeepSeek Harness~\cite{deepseekharness2026} for DeepSeek model.
Ax-Prover-Base and Humanize prove one self-contained theorem per episode and do not operate over repository-scale dependency structure, so the repository-level comparison includes only the coding agents (\Cref{tab:provability-performance}).
We do not compare against other open-source provers such as Hilbert~\cite{sumanth2026hilbert}, Goedel-Code-Prover~\cite{li2026goedelcodeprover}, WybeCoder~\cite{gloeckle2026wybecoder}, and $\mathrm{P}^3$~\cite{li2026p3} because each is designed for a different setting, and \Cref{app:additional-experiment-setup} gives the reason for each.

\input{tables/solve_rate.tex}

\subsection{Performance}
\label{sec:eval-performance}

\textbf{Solve rates.}
As shown in \Cref{tab:provability-performance}, no baseline exceeds \ToolName{}'s solve rate in any benchmark--LLM pair.
On function-level benchmarks with a frontier
LLM there is little headroom: Humanize, the strongest baseline, already proves almost every CLEVER and VERINA problem, and \ToolName{} closes the remainder to 100\% on CLEVER and VERINA; on AlgoVeri with Gemini-3.5-Flash the two tie. This also reflects \ToolName{}'s advantage on different LLMs and its scalability.
The second is repository scale: on Vero the coding agents solve at most 7.0\% of tasks, against up to 60.5\% for \ToolName{}.
Vero's obligations are spread across a repository and share definitions and helper lemmas, so a prover that treats all obligations in a single context will be overwhelmed, whereas \ToolName{} uses proof graphs and agent orchestration to handle them effectively.

\textbf{Cost at matched success.}
We compare cost against one baseline at a time, and only on the problems that this baseline and \ToolName{} both solve.
Costs on unsolved problems are not comparable: a method that gives up early, or that reward-hacks a hard problem such as closing it with an axiom, weakening or rewriting the theorem statement (\ie the behavior we observe most often in the coding agents) spends little precisely where an honest proof is expensive, so counting those runs would reward failure.
On the commonly solved problems, \ToolName{} is cheaper than every baseline in all nine function-level benchmark--LLM pairs (\Cref{tab:cost-performance}).
The cost saving is largest against Ax-Prover-Base (up to 85.7\%), which re-samples a whole proof after each failure, and against every baseline under DeepSeek-V4-Flash, where failures are most frequent.
It is smallest against Humanize (from 1.3\% to 30.9\% under GPT-5.6-Terra and Gemini-3.5-Flash), the strongest baseline, since for most easy problems, Humanize and the basic coding agents can simply propose and refine proofs so that our gain is small.
The one case where a baseline is cheaper is Vero with GPT-5.6-Terra, where Codex solves a single task that \ToolName{} also solves. This is because this is only one task problem which only has a few easy obligations so it doesn't weaken \ToolName{}'s advantage.
The two settings therefore differ in what \ToolName{} buys.
On isolated function-level obligations, it principally reduces proof cost.
On repository-level verification, persistent proof structure can additionally change what is practically provable: on Vero, \ToolName{} solves up to 26 of 43 tasks while the coding agents solve at most three, which is also why too few commonly solved tasks remain there for a meaningful cost comparison.

\input{tables/cost.tex}

\textbf{Success--cost frontier under a budget sweep.}\label{sec:budget-sweep}
\Cref{fig:rate-at-budget} asks a complementary question to \Cref{tab:cost-performance}: sweeping the per-problem cap from \$0 to the \$30 of the main evaluation, how many problems does each method solve?
For each LLM we sum the problems solved over the three function-level benchmarks (CLEVER, VERINA, and AlgoVeri) at every budget, so each curve is one method's success--cost frontier.
As shown in \Cref{fig:rate-at-budget}, \ToolName{} achieves the highest final solve count under all three backbones and leads over much of the evaluated budget range.
Moreover, the baselines flatten early.
Under GPT-5.6-Terra, Codex and Ax-Prover-Base gain relatively few additional solves beyond roughly \$5 per problem, while \ToolName{} reaches 419. This is because the extra money buys more of the same discarded work.

\begin{figure}[t]
    \input{figures/rate_at_budget.tex}
    \vspace{1.2em}

    \input{tables/ablation.tex}
\end{figure}

\subsection{Ablation Study}
\label{sec:eval-ablation}

In this section, we conduct some ablation studies to understand the contribution of each component in \ToolName{}. Since rerunning all five benchmarks for every variant is too expensive, we use a fixed subset: 10 problems (randomly chosen, which span difficulty levels from easy to hard) from each of CLEVER, VERINA, and AlgoVeri, 30 in total.

We first conduct a run-to-run variance analysis to understand the stability of \ToolName{} since model outputs are stochastic. We firstly run the complete \ToolName{} three times with Gemini-3.5-Flash on 3 function-level benchmarks and then we run the complete \ToolName{} three times with each backbone LLM on these 30 problems and details are shown in \Cref{app:run-variance}.

\textbf{Component ablation.}
To attribute the gains to individual components, we disable one component at a time and rerun \ToolName{} on the 30 problems with Gemini-3.5-Flash.
\Cref{tab:ablation-cost} compares each variant against the complete system, whose cost is the mean of its three Gemini-3.5-Flash runs in \Cref{tab:stability}.
\emph{\ToolName{} w/o Router} removes adaptive action choice: Deterministic Automation still runs first on every scheduled goal, and the action agents then follow one fixed order over the same two-level graph.
Total cost rises by 39.9\% and one VERINA problem is no longer solved. This is because the fixed routine purchases more unnecessary action invocations.
\emph{\ToolName{} w/o Lemma Planner} removes shared-helper discovery, so every goal must be proved in isolation.
It is the costliest variant, at 41.6\% above the complete system, and it solves the fewest problems (27 of 30), losing two on VERINA and one on AlgoVeri.
This mirrors the failure the case study isolates in the baselines (\Cref{app:case-study}): without lemma proposal, a missing fact cannot become a separately provable subgoal, so it is re-derived inside monolithic attempts instead of proved once and reused.
\emph{\ToolName{} w/o Automation} removes the automation pass.
Its effect is smaller: cost rises by 11.5\% on CLEVER and 12.4\% on VERINA, where many goals close mechanically, but falls by 5.8\% on AlgoVeri, whose algorithmic goals rarely close by automation and frequent automation will split amonolithic goal into multiple meaningless subgoals, which increases cost.
Overall, the component ablation study supports the theoretical analysis and system design of \ToolName{}.

%% file: tables/solve_rate.tex
\begin{table}[t]
    \centering
    \begin{minipage}[c]{0.36\linewidth}
        \captionsetup{justification=raggedright,singlelinecheck=false}
        \caption{\textbf{Solve rates across function- and repository-level benchmarks.}
            A task counts as solved only when Lean accepts the root declaration and no generated axioms are used.}
        \label{tab:provability-performance}
    \end{minipage}\hfill
    \begin{minipage}[c]{0.61\linewidth}
        \centering
        \resizebox{\linewidth}{!}{\input{tables/provability_repo.tex}}
    \end{minipage}
    \vspace{0.75em}

    \input{tables/provability_function.tex}
\end{table}

%% file: tables/provability_repo.tex
\begingroup
\scriptsize
\setlength{\tabcolsep}{3pt}
\renewcommand{\arraystretch}{1.15}
\begin{tabular}{lcccccc}
    \toprule
    \multicolumn{7}{c}{\textbf{Solve-Rate Comparison Across 2 Repository-Level Benchmarks}} \\
    \midrule
                 & \multicolumn{2}{c}{\textbf{GPT-5.6-Terra}}
                 & \multicolumn{2}{c}{\textbf{Gemini-3.5-Flash}}
                 & \multicolumn{2}{c}{\textbf{DeepSeek-V4-Flash}}                                        \\
    \cmidrule(lr){2-3}\cmidrule(lr){4-5}\cmidrule(lr){6-7}
    \textbf{Method}
                 & \textbf{NTP4VC}                                & \textbf{Vero}
                 & \textbf{NTP4VC}                                & \textbf{Vero}
                 & \textbf{NTP4VC}                                & \textbf{Vero}                        \\
    \midrule
    Coding Agent & 85.3\% & 2.3\% & 79.7\% & 7.0\% & 39.0\% & 0\% \\
    \ToolName{}  & \textbf{89.3\%} & \textbf{53.5\%} & \textbf{94.7\%} & \textbf{60.5\%} & \textbf{95.0\%} & \textbf{14.0\%} \\
    \bottomrule
\end{tabular}
\endgroup

%% file: tables/provability_function.tex
\begingroup
    \scriptsize
    \setlength{\tabcolsep}{2pt}
    \renewcommand{\arraystretch}{1.15}
    \begin{tabular*}{\linewidth}{@{\extracolsep{\fill}}lccccccccc@{}}
            \toprule
            \multicolumn{10}{c}{\textbf{Solve-Rate Comparison Across 3 Function-Level Benchmarks}} \\
            \midrule
            & \multicolumn{3}{c}{\textbf{GPT-5.6-Terra}}
            & \multicolumn{3}{c}{\textbf{Gemini-3.5-Flash}}
            & \multicolumn{3}{c}{\textbf{DeepSeek-V4-Flash}} \\
            \cmidrule(lr){2-4}\cmidrule(lr){5-7}\cmidrule(lr){8-10}
            \textbf{Method}
            & \textbf{CLEVER} & \textbf{VERINA} & \textbf{AlgoVeri}
            & \textbf{CLEVER} & \textbf{VERINA} & \textbf{AlgoVeri}
            & \textbf{CLEVER} & \textbf{VERINA} & \textbf{AlgoVeri} \\
            \midrule
            Ax-Prover-Base & 89.4\% & 85.7\% & 81.8\% & 85.7\% & 72.5\% & 76.6\% & 73.9\% & 66.1\% & 58.4\% \\
            Humanize       & 97.5\% & 98.9\% & 85.7\% & 97.5\% & 95.2\% & 80.5\% & 91.9\% & 73.5\% & 66.2\% \\
            Coding Agent   & 91.9\% & 87.3\% & 67.5\% & 94.4\% & 87.8\% & 71.4\% & 78.9\% & 58.2\% & 79.2\% \\
            \ToolName{}    & \textbf{100\%} & \textbf{100\%} & \textbf{89.6\%} & \textbf{100\%} & \textbf{98.4\%} & \textbf{80.5\%} & \textbf{96.9\%} & \textbf{100\%} & \textbf{93.5\%} \\
            \bottomrule
    \end{tabular*}
\endgroup

%% file: tables/cost.tex
\begin{table}[t]
    \centering
    \caption{\textbf{Cost comparison on commonly solved problems.}
        For each baseline, we compare expenditure only on the intersection of problems solved successfully by both \ToolName{} and that baseline, yielding a like-for-like comparison of the cost of valid proofs.}
    \label{tab:cost-performance}
    \input{tables/expe_clever.tex}
    \vspace{0.5em}

    \input{tables/expe_verina.tex}
    \vspace{0.5em}

    \input{tables/expe_algoveri.tex}
    \vspace{0.5em}

    \input{tables/expe_ntp4vc.tex}
    \vspace{0.5em}

    \input{tables/expe_vero.tex}
\end{table}

%% file: tables/expe_clever.tex
\begingroup
    \small
    \setlength{\tabcolsep}{3pt}
    \renewcommand{\arraystretch}{1.15}
    \resizebox{\linewidth}{!}{        \begin{tabular}{lccccccccc}
            \toprule
            \multicolumn{10}{c}{\textbf{Solve-Cost on CLEVER Benchmark (Total 161 Problems)}} \\
            \midrule
            & \multicolumn{3}{c}{\textbf{GPT-5.6-Terra}}
            & \multicolumn{3}{c}{\textbf{Gemini-3.5-Flash}}
            & \multicolumn{3}{c}{\textbf{DeepSeek-V4-Flash}} \\
            \cmidrule(lr){2-4}\cmidrule(lr){5-7}\cmidrule(lr){8-10}
            \textbf{Method}
            & \shortstack{\textbf{\# Commonly}\\\textbf{Solved}}
            & \shortstack{\textbf{Cost in USD}\\\textbf{(Baseline)}}
            & \shortstack{\textbf{Cost in USD}\\\textbf{(\ToolName{})}}
            & \shortstack{\textbf{\# Commonly}\\\textbf{Solved}}
            & \shortstack{\textbf{Cost in USD}\\\textbf{(Baseline)}}
            & \shortstack{\textbf{Cost in USD}\\\textbf{(\ToolName{})}}
            & \shortstack{\textbf{\# Commonly}\\\textbf{Solved}}
            & \shortstack{\textbf{Cost in USD}\\\textbf{(Baseline)}}
            & \shortstack{\textbf{Cost in USD}\\\textbf{(\ToolName{})}} \\
            \midrule
            Ax-Prover-Base & 144 & 69.9  & \CostReduction{63.0}{9.8}  & 138 & 364.6 & \CostReduction{140.3}{61.5} & 115 & 53.7 & \CostReduction{14.3}{73.4} \\
            Humanize       & 157 & 136.6 & \CostReduction{94.4}{30.9} & 157 & 271.9 & \CostReduction{199.6}{26.6} & 143 & 64.9 & \CostReduction{29.0}{55.3} \\
            Coding Agent   & 148 & 90.4  & \CostReduction{72.5}{19.8} & 152 & 258.7 & \CostReduction{176.4}{31.8} & 124 & 56.9 & \CostReduction{24.1}{57.7} \\
            \ToolName{}    & 161 & -- & 100.2 & 161 & -- & 235.6 & 156 & -- & 38.7 \\
            \bottomrule
        \end{tabular}    }
\endgroup

%% file: tables/expe_verina.tex
\begingroup
    \small
    \setlength{\tabcolsep}{3pt}
    \renewcommand{\arraystretch}{1.15}
    \resizebox{\linewidth}{!}{        \begin{tabular}{lccccccccc}
            \toprule
            \multicolumn{10}{c}{\textbf{Solve-Cost on VERINA Benchmark (Total 189 Problems)}} \\
            \midrule
            & \multicolumn{3}{c}{\textbf{GPT-5.6-Terra}}
            & \multicolumn{3}{c}{\textbf{Gemini-3.5-Flash}}
            & \multicolumn{3}{c}{\textbf{DeepSeek-V4-Flash}} \\
            \cmidrule(lr){2-4}\cmidrule(lr){5-7}\cmidrule(lr){8-10}
            \textbf{Method}
            & \shortstack{\textbf{\# Commonly}\\\textbf{Solved}}
            & \shortstack{\textbf{Cost in USD}\\\textbf{(Baseline)}}
            & \shortstack{\textbf{Cost in USD}\\\textbf{(\ToolName{})}}
            & \shortstack{\textbf{\# Commonly}\\\textbf{Solved}}
            & \shortstack{\textbf{Cost in USD}\\\textbf{(Baseline)}}
            & \shortstack{\textbf{Cost in USD}\\\textbf{(\ToolName{})}}
            & \shortstack{\textbf{\# Commonly}\\\textbf{Solved}}
            & \shortstack{\textbf{Cost in USD}\\\textbf{(Baseline)}}
            & \shortstack{\textbf{Cost in USD}\\\textbf{(\ToolName{})}} \\
            \midrule
            Ax-Prover-Base & 162 & 143.1 & \CostReduction{106.0}{25.9} & 137 & 376.5 & \CostReduction{175.5}{53.4} & 125 & 117.2 & \CostReduction{16.7}{85.7} \\
            Humanize       & 187 & 299.5 & \CostReduction{244.5}{18.4} & 180 & 498.9 & \CostReduction{403.4}{19.1} & 139 & 108.2 & \CostReduction{25.4}{76.5} \\
            Coding Agent   & 165 & 127.4 & \CostReduction{111.8}{12.2} & 166 & 407.4 & \CostReduction{237.9}{41.6} & 110 & 50.4  & \CostReduction{15.2}{69.8} \\
            \ToolName{}    & 189 & -- & 245.4 & 186 & -- & 537.4 & 189 & -- & 80.3 \\
            \bottomrule
        \end{tabular}    }
\endgroup

%% file: tables/expe_algoveri.tex
\begingroup
    \small
    \setlength{\tabcolsep}{3pt}
    \renewcommand{\arraystretch}{1.15}
    \resizebox{\linewidth}{!}{        \begin{tabular}{lccccccccc}
            \toprule
            \multicolumn{10}{c}{\textbf{Solve-Cost on AlgoVeri Benchmark (Total 77 Problems)}} \\
            \midrule
            & \multicolumn{3}{c}{\textbf{GPT-5.6-Terra}}
            & \multicolumn{3}{c}{\textbf{Gemini-3.5-Flash}}
            & \multicolumn{3}{c}{\textbf{DeepSeek-V4-Flash}} \\
            \cmidrule(lr){2-4}\cmidrule(lr){5-7}\cmidrule(lr){8-10}
            \textbf{Method}
            & \shortstack{\textbf{\# Commonly}\\\textbf{Solved}}
            & \shortstack{\textbf{Cost in USD}\\\textbf{(Baseline)}}
            & \shortstack{\textbf{Cost in USD}\\\textbf{(\ToolName{})}}
            & \shortstack{\textbf{\# Commonly}\\\textbf{Solved}}
            & \shortstack{\textbf{Cost in USD}\\\textbf{(Baseline)}}
            & \shortstack{\textbf{Cost in USD}\\\textbf{(\ToolName{})}}
            & \shortstack{\textbf{\# Commonly}\\\textbf{Solved}}
            & \shortstack{\textbf{Cost in USD}\\\textbf{(Baseline)}}
            & \shortstack{\textbf{Cost in USD}\\\textbf{(\ToolName{})}} \\
            \midrule
            Ax-Prover-Base & 63 & 113.8 & \CostReduction{55.5}{51.2} & 55 & 189.9 & \CostReduction{90.0}{52.6}  & 45 & 17.4 & \CostReduction{10.7}{38.4} \\
            Humanize       & 66 & 67.4  & \CostReduction{66.6}{1.3}  & 58 & 157.3 & \CostReduction{112.1}{28.7} & 51 & 29.9 & \CostReduction{15.2}{49.1} \\
            Coding Agent   & 52 & 35.3  & \CostReduction{32.0}{9.5}  & 51 & 94.4  & \CostReduction{79.8}{15.5}  & 61 & 58.1 & \CostReduction{29.8}{48.6} \\
            \ToolName{}    & 69 & -- & 74.6 & 62 & -- & 167.7 & 72 & -- & 48.2 \\
            \bottomrule
        \end{tabular}    }
\endgroup

%% file: tables/expe_ntp4vc.tex
\begingroup
\small
\setlength{\tabcolsep}{3pt}
\renewcommand{\arraystretch}{1.15}
\resizebox{\linewidth}{!}{    \begin{tabular}{lccccccccc}
        \toprule
        \multicolumn{10}{c}{\textbf{Solve-Cost on NTP4VC Benchmark (Total 300 Problems)}}                                \\
        \midrule
                     & \multicolumn{3}{c}{\textbf{GPT-5.6-Terra}}
                     & \multicolumn{3}{c}{\textbf{Gemini-3.5-Flash}}
                     & \multicolumn{3}{c}{\textbf{DeepSeek-V4-Flash}}                                                    \\
        \cmidrule(lr){2-4}\cmidrule(lr){5-7}\cmidrule(lr){8-10}
        \textbf{Method}
                     & \shortstack{\textbf{\# Commonly}\\\textbf{Solved}}
                     & \shortstack{\textbf{Cost in USD}\\\textbf{(Baseline)}}
                     & \shortstack{\textbf{Cost in USD}\\\textbf{(\ToolName{})}}
                     & \shortstack{\textbf{\# Commonly}\\\textbf{Solved}}
                     & \shortstack{\textbf{Cost in USD}\\\textbf{(Baseline)}}
                     & \shortstack{\textbf{Cost in USD}\\\textbf{(\ToolName{})}}
                     & \shortstack{\textbf{\# Commonly}\\\textbf{Solved}}
                     & \shortstack{\textbf{Cost in USD}\\\textbf{(Baseline)}}
                     & \shortstack{\textbf{Cost in USD}\\\textbf{(\ToolName{})}}                                         \\
        \midrule
        Coding Agent & 247 & 186.5 & \CostReduction{162.2}{13.0} & 239 & 608.2 & \CostReduction{343.9}{43.5} & 117 & 30.7 & \CostReduction{18.6}{39.6} \\
        \ToolName{}  & 268 & -- & 241.3 & 284 & -- & 619.9 & 285 & -- & 116.1 \\
        \bottomrule
    \end{tabular}}
\endgroup

%% file: tables/expe_vero.tex
\begingroup
    \small
    \setlength{\tabcolsep}{3pt}
    \renewcommand{\arraystretch}{1.15}
    \resizebox{\linewidth}{!}{        \begin{tabular}{lccccccccc}
            \toprule
            \multicolumn{10}{c}{\textbf{Solve-Cost on Vero Benchmark (Total 43 Problems)}} \\
            \midrule
            & \multicolumn{3}{c}{\textbf{GPT-5.6-Terra}}
            & \multicolumn{3}{c}{\textbf{Gemini-3.5-Flash}}
            & \multicolumn{3}{c}{\textbf{DeepSeek-V4-Flash}} \\
            \cmidrule(lr){2-4}\cmidrule(lr){5-7}\cmidrule(lr){8-10}
            \textbf{Method}
            & \shortstack{\textbf{\# Commonly}\\\textbf{Solved}}
            & \shortstack{\textbf{Cost in USD}\\\textbf{(Baseline)}}
            & \shortstack{\textbf{Cost in USD}\\\textbf{(\ToolName{})}}
            & \shortstack{\textbf{\# Commonly}\\\textbf{Solved}}
            & \shortstack{\textbf{Cost in USD}\\\textbf{(Baseline)}}
            & \shortstack{\textbf{Cost in USD}\\\textbf{(\ToolName{})}}
            & \shortstack{\textbf{\# Commonly}\\\textbf{Solved}}
            & \shortstack{\textbf{Cost in USD}\\\textbf{(Baseline)}}
            & \shortstack{\textbf{Cost in USD}\\\textbf{(\ToolName{})}} \\
            \midrule
            Coding Agent   & 1  & 1.6 & \CostIncrease{3.6}{131.2} & 3  & 55.0 & \CostReduction{24.6}{55.2} & 0 & 0.0 & 0.0 \\
            \ToolName{}    & 23 & --  & 515.7 & 26 & -- & 859.5 & 6 & -- & 28.4 \\
            \bottomrule
        \end{tabular}    }
\endgroup

%% file: figures/rate_at_budget.tex
    \centering
    \captionsetup[subfigure]{justification=centering}
    \begin{subfigure}[t]{0.325\linewidth}
        \centering
        \includegraphics[width=\linewidth]{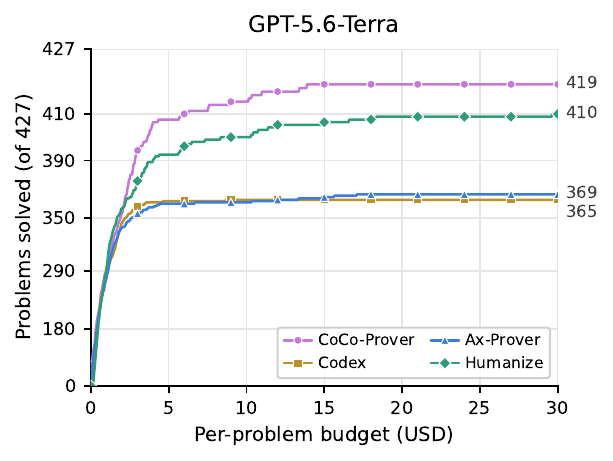}
        \caption{GPT-5.6-Terra.}
        \label{fig:rate-at-budget-gpt}
    \end{subfigure}\hfill
    \begin{subfigure}[t]{0.325\linewidth}
        \centering
        \includegraphics[width=\linewidth]{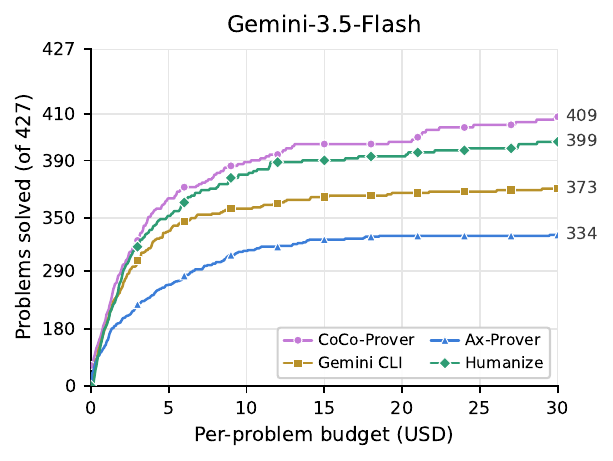}
        \caption{Gemini-3.5-Flash.}
        \label{fig:rate-at-budget-gemini}
    \end{subfigure}\hfill
    \begin{subfigure}[t]{0.325\linewidth}
        \centering
        \includegraphics[width=\linewidth]{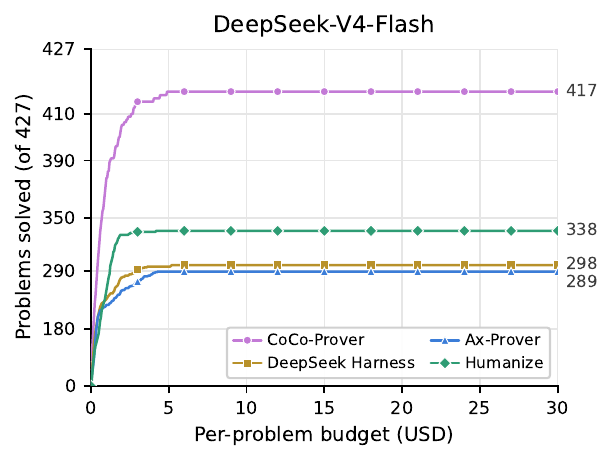}
        \caption{DeepSeek-V4-Flash.}
        \label{fig:rate-at-budget-deepseek}
    \end{subfigure}
    \caption{\textbf{Success--cost frontier on the 3 benchmarks.}
        Each curve gives the number of problems a method solves when every problem is capped at the per-problem budget on the horizontal axis, total 427 problems summed over CLEVER, VERINA, and AlgoVeri.}
    \label{fig:rate-at-budget}

%% file: tables/ablation.tex
\centering
\captionof{table}{\textbf{Component ablation across function-level benchmarks.}
    Using Gemini-3.5-Flash, we report total cost in USD and the number of problems solved (\#S) on a subset of the problems.
    Parenthesized percentages give the cost change relative to the complete system.}
\label{tab:ablation-cost}
\small
\setlength{\tabcolsep}{5pt}
\renewcommand{\arraystretch}{1.15}
\resizebox{\linewidth}{!}{    \begin{tabular}{lrrrrrrrr}
        \toprule
        \textbf{Ablated Version}
                                      & \multicolumn{2}{c}{\textbf{CLEVER}}
                                      & \multicolumn{2}{c}{\textbf{VERINA}}
                                      & \multicolumn{2}{c}{\textbf{AlgoVeri}}
                                      & \multicolumn{2}{c}{\textbf{Total}}                                                                                         \\
        \cmidrule(lr){2-3}\cmidrule(lr){4-5}\cmidrule(lr){6-7}\cmidrule(lr){8-9}
                                      & \textbf{Cost (USD)}                   & \textbf{\#S}
                                      & \textbf{Cost (USD)}                   & \textbf{\#S}
                                      & \textbf{Cost (USD)}                   & \textbf{\#S}
                                      & \textbf{Cost (USD)}                   & \textbf{\#S}                                                                       \\
        \midrule
        \ToolName{} w/o Router        & 41.02 (+43.2\%)                       & 10           & 60.38 (+38.8\%) & 9  & 53.80 (+38.6\%) & 10 & 155.20 (+39.9\%) & 29 \\
        \ToolName{} w/o Automation    & 31.93 (+11.5\%)                       & 10           & 48.89 (+12.4\%) & 9  & 36.57 (-5.8\%)  & 10 & 117.39 (+5.8\%)  & 29 \\
        \ToolName{} w/o Lemma Planner & 37.72 (+31.7\%)                       & 10           & 66.48 (+52.9\%) & 8  & 52.86 (+36.2\%) & 9  & 157.06 (+41.6\%) & 27 \\
        \ToolName{}                   & 28.64                                 & 10           & 43.49           & 10 & 38.82           & 10 & 110.94           & 30 \\
        \bottomrule
    \end{tabular}}

%% file: chapters/21_conclusion.tex
\section{Conclusion}
\label{sec:conclusion}

We presented \ToolName{}, an agent-orchestration framework for cost-efficient theorem proving in program verification with Lean~4.
\ToolName{} formalizes proving as metalevel decision-making under cost, grounded on two-level proof graphs. A symbolic scheduler selects every dependency-ready goal, and an evidence-grounded router groups them and purchases actions from four bounded specialist agents with Lean as the sole proof authority.
Across five Lean program verification benchmarks and three backbones, \ToolName{} achieves highest solve rate in any setting, and it reduces cost on commonly solved problems by up to 30.9\% against the strongest baseline, Humanize. Ablations also demonstrates the advantages of the system design with the shared-lemma discovery and adaptive routing.

%% file: appendix/theory.tex
\section{Additional Theoretical Analysis}
\label{app:additional-theory}

\paragraph{Roadmap.}
\Cref{app:metalevel-process} states the metalevel process of \Cref{sec:problem-formulation} in full, and \Cref{app:regret-proof} proves Theorem~1 and Corollary~2.
\Cref{app:metalevel-process} also names the operating point our protocol targets and shows that, in a finite abstraction of the process, the induced ordering is constant there (\Cref{prop:small-lambda,prop:rules-target-ordering}); that the deployment realizes the abstraction is assumed, not shown.
\Cref{app:faithful-reachability,app:policy-class,app:safe-filtering,app:shared-lemma-separation,app:kernel-soundness} establish structural properties of the two-level representation: it preserves reachability, graph information weakly---and on some instances strictly---raises the attainable value, the hard filters preserve the optimal value under their stated assumptions, sharing a helper can save a factor $K$, and returned proofs are kernel-sound.

\subsection{The metalevel process}
\label{app:metalevel-process}

\paragraph{States, selection, and computations.}
A metalevel state is $s=(\mathfrak G,\mathcal F,\sigma,\mathcal M,\mathcal L)$ as in \Cref{eq:controller-state-action}.
It doesn't invovle the router's rules. The analysis below needs only a common environment to every policy. Theorem~1 conditions on the history, so a router carrying internal state is covered.
$\mathrm{Sched}$ maintains the dependency-ready frontier (\Cref{app:scheduling}) and hands the whole $U=\mathcal F$ to the router.
The admissible computations $\mathcal C(s,U)$ are the triples $(G,o,\theta)$ of a group $G$.
The epoch's update records it in $\mathcal M'$ afterwards.
The contracts fix a finite grid of reasoning-effort and working-budget settings. That grid alone does not make $\mathcal C(s,U)$ finite, since $\theta$ also carries text: we assume in addition that every text-valued field of the action (\eg the analysis $j$ and the prepared context alike) is drawn from a finite vocabulary under a fixed token bound, which the specialist contracts impose.
Under that assumption, $\mathcal C(s,U)$ is finite and the maximum in \Cref{eq:metalevel-bellman} is attained.

\paragraph{Outcomes, costs, and transitions.}
A decision epoch begins at the pre-routing state $s$ and ends at $s'$, and two different things happen inside it.
The first is \emph{deliberation}: the router's calls, the decisions its guards reject, and the amendment of its own rules. This is policy-internal and leaves no trace in the state at all. Its charge $h$ goes to the router's own spending record. Over a run these charges accumulate as $H_\pi$.
The second is the purchased computation itself. The environment returns $Y\sim P_c(\cdot\mid s,U)$: the specialist's output (\eg a proof script, proposed helper lemmas with their elaboration results, a diagnosis, or a recovery proposal) together with Lean's verdict on it. Confining $Y$ to the environment's response is what makes $P_c$ common to every policy.
Its execution cost is $\kappa(s,U,c,Y)$, which never includes deliberation. Total spending is $C=\sum_{t<\tau}\kappa_t+H_\pi$, so each charge is counted once.
The controller and Lean then apply the update $s'=T(s,U,c,Y)$.
We assume:
\begin{itemize}[leftmargin=2.2em,itemsep=1pt,topsep=2pt]
  \item[(A1)] \emph{Markov outputs}: given $(s,U,c)$, the law of $Y$ depends neither on the rest of the history nor on which router produced $c$. This holds because everything the specialist is shown is carried by $c$;
  \item[(A2)] \emph{bounded computations}: $\kappa\le\bar c$ for every computation, and $\kappa\ge\kappa_{\min}>0$ for every paid one, that is, one that invokes a model.
        Some computations have $\kappa=0$ like Deterministic Automation;
  \item[(A3)] \emph{verifier determinism}: given $(s,U,c,Y)$, the update $T$ is deterministic. Neither rule amendment nor deliberation spending needs a clause here.$\mathrm{Sched}$ and the router's grouping are deterministic.
\end{itemize}
The stop computation costs nothing and moves to an absorbing state of value $0$.
Free computations would allow a policy to cycle forever at no charge, so throughout we restrict to \emph{proper} policies, those reaching an absorbing state with probability one and in finite expected time; the global iteration limit of \Cref{alg:cape-loop} enforces this in the implementation.
A solved state, in which every root and every generated dependency of its proof are stable, is absorbing with value $1$.

\paragraph{Value and Bellman equation.}
For a policy $\pi$ choosing $c_t\in\mathcal C(s_t,U_t)$ from the history, let $\tau$ be the number of computations before it enters an absorbing state, and let $C=\sum_{t<\tau}\kappa_t+H_\pi$ be its total realized cost, including its own inference cost $H_\pi$.
Its value is
\begin{equation}
  V^\pi_\lambda(s_0)=\Pr_\pi[\text{solve}]-\lambda\,\mathbb E_\pi[C],
  \label{eq:policy-value}
\end{equation}
and $V^*_\lambda=\sup_\pi V^\pi_\lambda$, taken over policies with zero inference cost, so that the optimum charges nothing for deliberation.
Because stopping is always admissible, $0\le V^*_\lambda\le1$ holds here.
Writing $W=1-V$ turns the problem into one with nonnegative stage costs $\lambda\kappa$ and terminal cost $0$ (solved) or $1$ (stopped). Under this positivity condition, the optimal value satisfies the Bellman equation~\eqref{eq:metalevel-bellman}~\cite{bertsekas2012dynamic}.

\paragraph{The frontier and the evaluation metric.}
Each policy induces a point $(\Pr_\pi[\text{solve}],\mathbb E_\pi[C])$.
Maximizing \Cref{eq:policy-value} for a fixed $\lambda$ selects a point on the success--cost boundary of the achievable set, and sweeping $\lambda$ recovers its supported points, \ie those on its convex hull.
The evaluation instead reports the threshold metric $\mathrm{Pass}_R(B)$ (\ie the fraction of tasks a policy $R$ solves before its cumulative cost reaches an external cap $B$ that is never given to the controller) together with cost on commonly solved problems.
$V^*_\lambda$ is the normative object behind these measurements, not a quantity we estimate.

\paragraph{The router as an ordinal approximation.}
At each state, $Q^*_\lambda$ induces a latent ordering of the admissible computations, $c\succ_s c'$ iff $Q^*_\lambda(s,U,c)>Q^*_\lambda(s,U,c')$, and an optimal policy picks a maximal element.
The router is asked to predict a maximal element of this ordering from $\mathrm{ev}(s,U)$.

\paragraph{The operating point.}
Our protocol does not choose an interior trade-off between success and cost.
The caps of are set generously and the prover works best-effort until it proves the target or exhausts its admissible computations. The preference is a small but positive $\lambda$ one, which means success dominates, and cost separates continuations that are equally likely to succeed.
The cap $B$ is a safety limit rather than a control input. It is never given to the controller, and $\mathrm{Pass}_R(B)$ records how often best-effort proving finishes beneath it.
The \Cref{prop:small-lambda} shows that, in a finite abstraction of this process, it is the regime (not a numerical value of $\lambda$) that determines the ordering the router has to predict.

\begin{proposition}[Constant ordering in the small-$\lambda$ regime]
  \label{prop:small-lambda}
  Assume (A1)--(A3) and restrict to proper policies. Assume further that the reachable state space is finite; $\mathcal C(s,U)$ is finite by the contracts.
  Then there is $\lambda_0>0$ such that for every state $s$, every frontier $U$, and every $c,c'\in\mathcal C(s,U)$, the sign of $Q^*_\lambda(s,U,c)-Q^*_\lambda(s,U,c')$ is the same for all $\lambda\in(0,\lambda_0]$.
  The ordering $\succ_s$ is therefore one fixed ordering on that interval, and its maximal elements are the computations that first maximize the attainable success probability and, among those, minimize expected remaining cost.
\end{proposition}

\begin{proof}
  Under the finiteness hypothesis, the metalevel process is a finite stochastic shortest-path problem over proper policies, for which deterministic stationary policies suffice~\cite{bertsekas2012dynamic}. There are finitely many of them and each one is proper, so its success probability $p_\pi(s)$ and expected cost $c_\pi(s)$ are finite.
  By \Cref{eq:policy-value} each has value $V^\pi_\lambda(s)=p_\pi(s)-\lambda\,c_\pi(s)$, affine in $\lambda$, so $V^*_\lambda(s)=\max_\pi V^\pi_\lambda(s)$ is a maximum of finitely many affine functions: piecewise linear and convex in $\lambda$, with finitely many breakpoints.
  Since $Q^*_\lambda(s,U,c)=-\lambda\,\mathbb E_Y[\kappa]+\mathbb E_Y\bigl[V^*_\lambda(T(s,U,c,Y))\bigr]$ is an affine term plus a finite convex combination of such functions, it is piecewise linear in $\lambda$ with finitely many breakpoints, and so is every difference $Q^*_\lambda(s,U,c)-Q^*_\lambda(s,U,c')$.
  Each difference therefore changes sign at finitely many values of $\lambda$. As there are finitely many tuples $(s,U,c,c')$, taking $\lambda_0$ below the smallest positive such value gives the first claim.
  For the second, note that $Q^*_\lambda(s,U,c)$ is itself the value of the best continuation that begins with $c$, hence of the form $p(s,U,c)-\lambda\,\bar c(s,U,c)$ on each linear piece, where $p$ is the success probability attainable after $c$ and $\bar c$ is the expected remaining cost of the continuation attaining it.
  Letting $\lambda\downarrow0$, a computation whose $p$ is not maximal is strictly beaten for all small enough $\lambda$, and among those attaining the largest $p$ the value is decreasing in $\bar c$, so the maximizers on $(0,\lambda_0]$ are exactly the lexicographic ones.
\end{proof}

\begin{proposition}[The routing rules target this ordering]
  \label{prop:rules-target-ordering}
  Fix $\lambda\in(0,\lambda_0]$.
  \emph{(i)} Among computations that leave the attainable success probability unchanged, $\succ_s$ ranks by expected remaining cost, so when no admissible computation raises that probability the cheapest is maximal.
  In particular, let $c$ close the remaining work outright with probability $p_c$ at cost $\kappa_c$ and, under the proviso of \Cref{app:safe-filtering}, leave the state value-equivalent when it fails.
  Write $J^*(s)$ for the optimal expected remaining cost at $s$, so that $V^*_\lambda(s)=1-\lambda J^*(s)$ when success is attainable with certainty. Then
  \begin{equation}
    \delta_\lambda(s,U,c)=\lambda\bigl(\kappa_c-p_c\,J^*(s)\bigr),
    \label{eq:cheap-first-condition}
  \end{equation}
  so $c$ is maximal exactly when $\kappa_c\le p_c\,J^*(s)$.
  For a paid computation the condition is a genuine restriction, and being the cheapest admissible computation does not imply it.
\end{proposition}

\begin{proof}
  The first sentence of \emph{(i)} is the second claim of \Cref{prop:small-lambda} restricted to the computations that tie on attainable success.
  For \Cref{eq:cheap-first-condition}, the proviso makes the failure successor value-equivalent to $s$, so $Q^*_\lambda(s,U,c)=-\lambda\kappa_c+p_c+(1-p_c)V^*_\lambda(s)$; substituting $V^*_\lambda(s)=1-\lambda J^*(s)$ and subtracting from $V^*_\lambda(s)$ gives the identity.
\end{proof}

Within this abstraction the two initial routing rules are statements about one ordering rather than heuristics attached to a particular $\lambda$.
A working budget $\theta$ is an operational constraint: once it is exhausted, continuing on the same goal is simply outside $\mathcal C(s,U)$, and the controller returns to the router, which may re-decide.
Whether this is also value-preserving is a comparison of whole continuations: writing $h_{\mathrm{ret}}$ for the routing overhead, returning dominates continuing exactly when $p_{\mathrm{ret}}-\lambda\bigl(h_{\mathrm{ret}}+\mathbb E[C_{\mathrm{ret}}]\bigr)\ge p_{\mathrm{cont}}-\lambda\,\mathbb E[C_{\mathrm{cont}}]$.
Retaining the partial proof and the router's analysis, so that work is resumed rather than redone, keeps $p_{\mathrm{ret}}$ and $\mathbb E[C_{\mathrm{ret}}]$ close to their continuing counterparts, but it does not by itself establish the inequality: if one further refinement would have closed the goal for $1$, a return costing $0.2$ that resumes the same refinement totals $1.2$.
We therefore treat the bounded contract as a restriction of the action space, justified operationally by retention, and do not claim it is free.

\subsection{Regret decomposition: proofs of Theorem~1 and Corollary~2}
\label{app:regret-proof}

This subsection states and proves the result summarized at the end of \Cref{sec:problem-formulation}: adaptive routing is beneficial when the decision regret it avoids exceeds its own inference overhead.
Both statements are written in terms of the metalevel regret $\delta_\lambda$ of \Cref{eq:metalevel-regret}.

\noindent\textbf{Theorem 1 (regret decomposition).}
Let router $R$ reach an absorbing state after finitely many computations in expectation, with value $V_\lambda^R$, finite expected inference cost $H_R$ counted in $C$, and the cap encoded as an explicit stop. Then
\begin{equation}
  V_\lambda^*(s_0)-V_\lambda^R(s_0)
  =\mathbb E_R\Bigl[\textstyle\sum_{t<\tau}\delta_\lambda(s_t,U_t,c_t)\Bigr]+\lambda\,\mathbb E_R[H_R],
  \label{eq:regret-decomposition}
\end{equation}
where $\tau$ is the number of computations.

\noindent\textbf{Corollary 2 (when routing pays for itself).}
Router $R$ outperforms a fixed routine $R_0$ exactly when the regret it avoids exceeds its extra decision-making cost:
\begin{equation}
  \mathbb E_{R_0}\Bigl[\textstyle\sum_{t}\delta_\lambda\Bigr]-\mathbb E_{R}\Bigl[\textstyle\sum_{t}\delta_\lambda\Bigr]
  \;>\;
  \lambda\bigl(\mathbb E_R[H_R]-\mathbb E_{R_0}[H_{R_0}]\bigr).
  \label{eq:routing-pays}
\end{equation}
Cost efficiency is therefore the allocation of purchased computations with low cumulative regret.

\begin{proof}[Proof of Theorem~1]
  Fix a router $R$ with $\mathbb E_R[\tau]<\infty$, write $\delta_t=\delta_\lambda(s_t,U_t,c_t)$, and let $\mathcal H_t$ be the history up to the choice of $c_t$.
  By \Cref{eq:metalevel-bellman} and (A1), on $\{t<\tau\}$,
  \[
    \mathbb E_R\bigl[Q^*_\lambda(s_t,U_t,c_t)\,\big|\,\mathcal H_t\bigr]
    =\mathbb E_R\bigl[-\lambda\kappa_t+V^*_\lambda(s_{t+1})\,\big|\,\mathcal H_t\bigr],
  \]
  so that $\mathbb E_R[\delta_t\mid\mathcal H_t]=\mathbb E_R[V^*_\lambda(s_t)-V^*_\lambda(s_{t+1})+\lambda\kappa_t\mid\mathcal H_t]$.
  Summing over $t<\tau$ and taking expectations telescopes to
  \[
    \mathbb E_R\Bigl[\sum_{t<\tau}\delta_t\Bigr]
    =V^*_\lambda(s_0)-\mathbb E_R\bigl[V^*_\lambda(s_\tau)\bigr]+\lambda\,\mathbb E_R\Bigl[\sum_{t<\tau}\kappa_t\Bigr],
  \]
  where exchanging sum and expectation is justified because every summand is bounded ($0\le V^*_\lambda\le1$, and $\kappa\le\bar c$ by (A2)) and $\mathbb E_R[\tau]<\infty$.
  The final state is solved, with value $1$, or stopped, with value $0$, so $\mathbb E_R[V^*_\lambda(s_\tau)]=\Pr_R[\text{solve}]$.
  Substituting $\Pr_R[\text{solve}]=V^R_\lambda(s_0)+\lambda\,\mathbb E_R[\sum_{t<\tau}\kappa_t]+\lambda\,\mathbb E_R[H_R]$ from \Cref{eq:policy-value} gives \Cref{eq:regret-decomposition}.
\end{proof}

\begin{proof}[Proof of Corollary~2]
  Apply \Cref{eq:regret-decomposition} to $R$ and to $R_0$ from the same $s_0$.
  The common term $V^*_\lambda(s_0)$ cancels, so $V^R_\lambda(s_0)>V^{R_0}_\lambda(s_0)$ if and only if \Cref{eq:routing-pays} holds.
\end{proof}

Theorem~1 is the performance-difference identity of reinforcement learning~\cite{kakade2002approximately}, specialized to purchased computations with an explicit charge for deliberation.
What it fixes is the target of the design: a router is worthwhile precisely to the extent that its choices lose less continuation value than a fixed routine's, net of its own inference cost.

\paragraph{Truncated runs.}
The telescoping argument assumes the trajectory reaches the modeled absorbing state and that $\tau$ counts the transition into it.
A run stopped by an external time or expenditure cap ends instead at a state $s_T$ that may still carry positive continuation value, and for the truncated value $V^R_{\mathrm{trunc}}(s_0)=\mathbb E[Z_T]-\lambda\,\mathbb E[\sum_{t<T}\kappa_t+H_R]$, with $Z_T$ the solve indicator, the same calculation gives
\begin{equation}
  V^*_\lambda(s_0)-V^R_{\mathrm{trunc}}(s_0)
  =\mathbb E_R\Bigl[\textstyle\sum_{t<T}\delta_t\Bigr]+\lambda\,\mathbb E_R[H_R]+\mathbb E_R\bigl[V^*_\lambda(s_T)-Z_T\bigr].
  \label{eq:truncated-decomposition}
\end{equation}
The last term is the value abandoned at the cap and cannot be dropped: truncating an unsolved state of value $0.9$ before any computation gives a gap of $0.9$ with no recorded regret.
\Cref{eq:regret-decomposition} is therefore the identity for runs that end in the modeled stop, which we obtain by appending an explicit stop computation at the cap and charging it the abandoned value as its regret.
Both statements additionally require $\mathbb E_R[H_R]<\infty$, which $\mathbb E_R[\tau]<\infty$ and (A2) do not imply: the outer horizon bounds completed epochs, so a router that keeps returning malformed decisions could deliberate without end. We therefore assume a finite cap on router retries within an epoch, with a defined fallback when it is reached.

\subsection{Faithful two-level graph states preserve reachability}
\label{app:faithful-reachability}

This subsection concerns primitive Lean steps which are deterministic.
For a sequence $w=(a_1,\ldots,a_k)$ of primitive steps, let $F^w(s)$ be the state reached from $s$ and $C(w)$ its cost; $\tau$ is reserved for the number of computations in \Cref{app:metalevel-process}.
Following the verifier-defined reachability view~\cite{sonoda2026why}, define
\begin{equation}
  \mathrm{Reach}(s)
  =\mathbf 1\!\left\{\exists w:\ F^w(s)\in\mathcal G\right\},
  \qquad
  C^\star(s)
  =\inf_{w:F^w(s)\in\mathcal G} C(w),
  \label{eq:reachability-cost}
\end{equation}
where $\mathcal G$ is the solved set and $C^\star(s)=\infty$ when no successful sequence exists.

Let $\mathcal R$ be the lower-level Lean search space with verifier transition $F_{\mathcal R}$ and solved set $\mathcal G_{\mathcal R}$.
Let $e:\mathcal R\rightarrow\mathcal S$ encode a lower-level state in \ToolName{}'s two-level state space $\mathcal S$, whose verified primitive transition is $F_{\mathcal S}$ and whose solved set is $\mathcal G_{\mathcal S}$.

\begin{proposition}[Faithful reachability]
  \label{prop:faithful-reachability}
  Suppose that, for every reachable lower-level state $r$ and primitive step $a$, the encoding commutes with the verifier transition:
  \begin{equation}
    e\!\left(F_{\mathcal R}(r,a)\right)
    =F_{\mathcal S}(e(r),a).
    \label{eq:transition-preservation}
  \end{equation}
  Suppose further that solved states are preserved in both directions, that the two spaces admit the same primitive steps at corresponding states, and that the encoding does not alter step costs.
  Then
  \begin{equation}
    \mathrm{Reach}^{\mathcal R}(r)
    =\mathrm{Reach}^{\mathcal S}(e(r)),
    \qquad
    C_{\mathcal R}^{\star}(r)
    =C_{\mathcal S}^{\star}(e(r)).
    \label{eq:faithful-reachability-cost}
  \end{equation}
\end{proposition}

\begin{proof}
  If a lower-level sequence reaches a solved state, induction over the transition-preservation identity in \Cref{eq:transition-preservation} maps every prefix to the two-level representation, and solved-set preservation maps its endpoint to $\mathcal G_{\mathcal S}$.
  Conversely, the same argument lifts a successful sequence in $\mathcal S$ through the faithful encoding.
  Because corresponding sequences have equal cost, taking the infimum over successful sequences gives the cost equality in \Cref{eq:faithful-reachability-cost}.
  This is the faithful-abstraction argument of the verifier-defined reachability view~\cite{sonoda2026why} specialized to the two-level graph.
\end{proof}

Every certificate reachable by primitive Lean steps is reachable in the representation of the two-level proof graphs at the same primitive cost.

\subsection{Graph information raises the attainable value}
\label{app:policy-class}

Let $I_{\mathrm{loc}}(s,G)$ contain the goals of the supplied group and their local attempt history.
A target-local router has the form $R_{\mathrm{loc}}(s,G)=f(I_{\mathrm{loc}}(s,G))$, whereas a graph-informed router may use the full state $s$, including other goals, dependency edges, artifact status, and realized spend.
Let $\Pi_{\mathrm{loc}}$ and $\Pi_{\mathrm{graph}}$ denote the corresponding policy classes, and $V^*_{\lambda,\Pi}(s)=\sup_{\pi\in\Pi}V^\pi_\lambda(s)$.

\begin{proposition}[Information-class dominance]
  \label{prop:information-class-dominance}
  $V^*_{\lambda,\Pi_{\mathrm{graph}}}(s)\ge V^*_{\lambda,\Pi_{\mathrm{loc}}}(s)$ for every state $s$ and every $\lambda\ge0$.
\end{proposition}

\begin{proof}
  A graph-informed router can ignore its additional inputs and process them at no extra inference cost, so $\Pi_{\mathrm{loc}}\subseteq\Pi_{\mathrm{graph}}$, and a supremum over a superset is no smaller. If reading graph evidence were itself more expensive, inclusion of decision rules would not suffice and the comparison would have to charge that difference.
\end{proof}

The gap between the classes shows up when a single policy must serve instances that look alike locally: the value of a helper lemma depends on how many consumers it serves. The next result is a minimax statement about one fixed local policy---on each instance separately, a local policy hardcoded for that instance can be optimal.

\begin{proposition}[Strict separation]
  \label{prop:strict-separation}
  Consider two instances.
  In instance $A$, $K\ge2$ consumers $g_1,\ldots,g_K$ each need a fact that can be supplied by one helper lemma $h$; in instance $B$, only $g_1$ exists.
  The symbolic rule supplies the group $\{g_1\}$ at \emph{every} decision until the router commits---either closing $g_1$ directly or proposing the helper---and every computation before that commitment returns the same output in both instances.
  So the whole observation history $I_{\mathrm{loc}}$ admits up to the commitment, and the actions available along it, are identical in the two instances; what differs is the root set, which only a graph-informed router reads.
  Supplying $\{g_1\}$ only at the first decision would not suffice: a rule that moved to $g_2$ after a free failed attempt would hand a local router the second consumer, and reading it would be enough to choose correctly on both instances.
  Every computation is deterministic: the Proof Writer closes a consumer directly at cost $a$, the Lemma Planner proposes and proves $h$ at cost $b$, each consumer then closes from $h$ at residual cost $d<a$, and every other computation makes no progress and returns the same output on both instances.
  Assume $a<b+d$ and $b+Kd<Ka$, so the helper is a losing trade for one consumer but a winning trade for $K$, and $\lambda Ka<1$, so solving is always worth its cost.
  Then one router in $\Pi_{\mathrm{graph}}$ is optimal on both instances, while every \emph{single} router $R\in\Pi_{\mathrm{loc}}$ satisfies
  \begin{equation}
    \max_{I\in\{A,B\}}\Bigl(V^*_{\lambda,\Pi_{\mathrm{graph}}}-V^R_\lambda\Bigr)(I)
    \;\ge\;
    \lambda\,\frac{\Delta_B\,\Gamma_A}{\Delta_B+\Gamma_A},
    \label{eq:strict-separation}
  \end{equation}
  where $\Delta_B=b+d-a$ and $\Gamma_A=\min\bigl(a-d,\;Ka-b-Kd\bigr)$ are positive by assumption.
\end{proposition}

\begin{proof}
  A graph-informed router reads the number of consumers from $\mathcal D$ and chooses the helper in $A$ (cost $b+Kd$) and the direct proof in $B$ (cost $a$), both optimal under the assumptions.
  By hypothesis the group supplied and every output observed agree across the instances until the commitment, so a local router reaches that decision on $g_1$ with the same observation history in both, and chooses the Lemma Planner there with the same probability $p$.
  In $B$, that choice costs at least $b+d$, an excess of $\Delta_B$ over the optimum.
  In $A$, choosing the Proof Writer on $g_1$ costs $a$ for $g_1$ and at least $\min\bigl(b+(K-1)d,\,(K-1)a\bigr)$ for the rest, an excess of at least $\Gamma_A$ over $b+Kd$.
  Because $\lambda Ka<1$, stopping loses at least as much as either excess, and among solving runs value differences equal $\lambda$ times cost differences, so the local router's worst-case value loss is at least $\lambda\max\bigl(p\Delta_B,(1-p)\Gamma_A\bigr)$, which is minimized at $p=\Gamma_A/(\Delta_B+\Gamma_A)$ and gives \Cref{eq:strict-separation}.
\end{proof}

For example, with $K=2$, $a=5$, $b=6$, and $d=0$, the helper costs $6$ against $5$ from either consumer's view, yet $6$ against $10$ globally, and every local router loses at least $4\lambda/5$ on one of the two instances.
The separation is therefore uniform over the two instances, and it needs $\lambda>0$: at $\lambda=0$ both classes attain success probability one.

\subsection{Safe filtering}
\label{app:safe-filtering}

\ToolName{} shrinks $\mathcal C(s,U)$ with two hard filters.
The first is the \emph{no identical deterministic repetition} clause of $\textsc{Guards}$ (\Cref{alg:cape-loop}), which the Deterministic Automation branch supports by caching a no-progress result against reruns on an unchanged state.
The second is the subtraction of $\textsc{RootDisc}(\mathcal D_t)$ from the frontier, together with the Lemma Planner's requirement that a proposed helper name root-relevant consumers (\Cref{sec:specialist-agents}), so that work with no path to a root is never scheduled.
This subsection asks when discarding those options is free, that is, when it leaves $V^*_\lambda(s)$ unchanged.
Both rely on the same proviso, which we state as \emph{erasability}: the filtered computation can be deleted from any trajectory that contains it, leaving the subsequent admissible actions, their output laws and prices, the scheduling, and the terminal behavior unchanged, so in particular $V^*_\lambda(s')=V^*_\lambda(s)$ for its successor $s'$.
Equality of the optimal value alone is not enough. A free computation with $\delta_\lambda=0$ can be the only route to the optimal continuation, which is a free step that sets a flag on which a cheaper proof action depends ties with the optimum, yet deleting it strictly lowers the value. Thus, the deletion argument needs the stronger, behavioral condition.

\begin{lemma}[Deterministic no-progress elimination]
  \label{lem:no-progress-elimination}
  Suppose a computation $c$ has a deterministic output and an earlier invocation from the same normalized proof state, environment, and configuration left that state unchanged.
  Then, for $\lambda>0$, $\delta_\lambda(s,U,c)\ge\lambda\kappa_{\min}>0$, if $c$ is paid, so removing it from $\mathcal C(s,U)$ does not change $V^*_\lambda(s)$.
  If $c$ is free, $\delta_\lambda(s,U,c)=0$: it is a maximizer, and removing it preserves the value because erasability makes it a self-loop on everything the continuation can observe. It is removed to keep policies proper.
\end{lemma}

\begin{proof}
  By determinism, the repeated invocation again leaves the proof state unchanged, so under the proviso $Q^*_\lambda(s,U,c)=-\lambda\kappa+V^*_\lambda(s)$. For a paid $c$, (A2) gives $\kappa\ge\kappa_{\min}$ and hence $Q^*_\lambda(s,U,c)\le V^*_\lambda(s)-\lambda\kappa_{\min}$; for a free $c$ the same identity gives $Q^*_\lambda(s,U,c)=V^*_\lambda(s)$.
  Thus a paid $c$ is never maximal, and the maximum in \Cref{eq:metalevel-bellman} is attained without it. A free $c$ is maximal; erasability is what lets the optimal continuation be executed in its absence, so dropping it leaves $V^*_\lambda(s)$ unchanged.
\end{proof}

\begin{lemma}[Root-disconnected work elimination]
  \label{lem:disconnected-work-elimination}
  A proposed declaration with no dependency path to an active consumer in any root's proof region does not occur in the dependency closure returned from the \emph{current} graph.
  Assume in addition that such work can be postponed until a consumer is attached, at no greater cost and with the same success distribution. Then, under the proviso, a paid computation that works only on such a declaration has $\delta_\lambda\ge\lambda\kappa_{\min}$, so deferring it does not change $V^*_\lambda(s)$.
\end{lemma}

\begin{proof}
  Every generated lemma used by a returned root proof has, by definition, a path through its consumers to that root, so a lemma with no such path is absent from every closure and cannot affect whether a state is solved.
  A later plan may attach a consumer, which is why the postponement assumption is needed where it makes deferral. Under the proviso the computation then leaves $V^*_\lambda$ unchanged while costing at least $\kappa_{\min}$, and the argument of \Cref{lem:no-progress-elimination} applies.
\end{proof}

Consumer-first admission is therefore a search restriction justified by the postponement assumption. The remaining semantic choices are not certified as safe and are left to the router with bounded trials.

\subsection{Shared-lemma cost separation}
\label{app:shared-lemma-separation}

The previous two subsections compare routers on the same shared graph.
The next result compares representations: a shared lemma-dependency graph against isolated proof hypergraphs.

\begin{proposition}[Shared-lemma cost separation]
  \label{prop:shared-lemma-separation}
  Fix an abstract instance in which $K$ consumers each need a fact supplied only by an absent helper $h$, producing $h$ costs $c_h$, and consumer $i$ then closes at residual cost $d_i$. Compare two \emph{stipulated} strategies: an isolated-graph prover that cannot transfer generated artifacts, and a graph prover that proves $h$ once. Their costs are
  \begin{equation}
    C_{\mathrm{local}}
    =Kc_h+\sum_{i=1}^{K}d_i,
    \qquad
    C_{\mathrm{graph}}
    =c_h+\sum_{i=1}^{K}d_i.
    \label{eq:shared-lemma-separation}
  \end{equation}
  When $c_h>0$ and $d_i=0$ for every consumer, $C_{\mathrm{local}}/C_{\mathrm{graph}}=K$; with $c_h=0$ both totals equal $\sum_i d_i$, so the ratio is $1$, or undefined when the residuals vanish too.
\end{proposition}

\begin{proof}
  The accounting is immediate from the stipulated costs.
  An isolated-graph prover that cannot transfer generated artifacts pays $c_h$ inside each graph.
  The graph-informed prover proves $h$ once, records one node, and connects it to all consumers. Deterministic use of the checked declaration incurs no additional model-generation cost.
  Substitution in \Cref{eq:shared-lemma-separation} gives the stated ratio when the residual costs vanish.
\end{proof}

Consequently, for any external cap $B$ satisfying
\begin{equation}
  c_h+\sum_i d_i
  \le B
  <Kc_h+\sum_i d_i,
  \label{eq:shared-lemma-threshold}
\end{equation}
the stipulated graph strategy completes within the cap while the stipulated isolated strategy does not. This compares the two strategies, not every strategy each representation admits.
Failed helper attempts and rollback appear as the realized term $r_h$ in \Cref{eq:sharing-costs}.

\paragraph{Lemma planning as nonlocal value of computation.}
Measured against $V^*_\lambda$, no purchase can show a gain: $Q^*_\lambda(s,U,c_{\mathrm{LP}})-V^*_\lambda(s)=-\delta_\lambda(s,U,c_{\mathrm{LP}})\le0$, since the optimum already includes the option of buying the helper.
The amortization argument is therefore stated against a \emph{direct-only} reference policy $\pi_{\mathrm{dir}}$, which never proposes a helper and proves every consumer separately.
For two strategies the value difference is $V_{\mathrm{share}}-V_{\mathrm{dir}}=(p_{\mathrm{share}}-p_{\mathrm{dir}})+\lambda\bigl(\mathbb E[C_{\mathrm{dir}}]-\mathbb E[C_{\mathrm{share}}]\bigr)$, so a cost saving settles the comparison only when the two succeed equally often.
Assume therefore that sharing does not change the success probability, write $\Delta J_i$ for the expected cost that $h$ saves consumer $i$ relative to $\pi_{\mathrm{dir}}$, and treat the consumers as otherwise independent. Then proposing and proving $h$ improves on $\pi_{\mathrm{dir}}$ by roughly
\begin{equation}
  \lambda\Bigl(\sum_{i\in\mathrm{Cons}(h)}\Delta J_i-c_h-\mathbb E[r_h]\Bigr),
  \label{eq:lemma-voc}
\end{equation}
where $r_h$ is the realized cost of failed helper attempts and rollback.
The equal-success assumption is a real restriction rather than a formality: a helper that turns out false, too hard, or abandoned lowers $p_{\mathrm{share}}$, and $r_h$ prices those attempts without pricing their effect on whether the task is solved at all. A strategy that spends $1$ and succeeds half the time beats one that spends $10$ and always succeeds on cost, yet their values are $\tfrac12-\lambda$ and $1-10\lambda$, so the cheap strategy loses on value exactly when $\lambda<1/18$ and wins above it.
For instance, in the $K=2$, $a=5$, $b=6$, $d=0$ instance of \Cref{app:policy-class} with $\lambda=0.1$, this reads $0.1(10-6)=0.4$ against the direct-only policy, while the advantage over $V^*_\lambda$ is zero.
One granularity caveat: the implemented Lemma Planner only proposes and elaborates $h$, returning control to the router (\Cref{sec:specialist-agents}), so its successor value still carries the cost of proving $h$. \Cref{eq:lemma-voc} prices the macro-action that includes that proof, and $c_h$ must not be charged twice.
A lemma can thus be a losing purchase from each consumer's view and a winning one globally, as in the $K=2$ example of \Cref{app:policy-class}. This is why the router's evidence packet includes consumers and conditional descendants.

\subsection{Kernel-relative soundness}
\label{app:kernel-soundness}

\begin{proposition}[Kernel-relative soundness]
  \label{prop:kernel-soundness}
  Assume that (i) the declared Lean environment $\Gamma_0$ is accepted, (ii) \textup{\ToolName{}} promotes a generated lemma only after Lean accepts its proof, (iii) every generated dependency of root $\ell\in L_\star$ is stable, and (iv) no prohibited placeholder or unapproved axiom occurs.
  If \textup{\ToolName{}} returns a stable proof of $\ell$, then $\Gamma_0\vdash\ell$.
\end{proposition}

\begin{proof}
  Topologically order the generated dependency DAG.
  Induction over this order shows that each stable generated declaration has a kernel-accepted proof using only $\Gamma_0$ and earlier stable declarations.
  Applying the argument to the accepted proof of $\ell$ yields the result.
\end{proof}

%% file: appendix/operational_details.tex
\section{Operational Details}
\label{app:operational}

\subsection{Two-level proof graphs}
\label{app:proof-graphs}

This subsection gives the full definition of the two-level proof graphs summarized in \Cref{sec:problem-formulation} and drawn in \Cref{fig:coupled-proof-graph}.

\textbf{Inter-lemma dependency graph.}
Across declarations, the lemma-dependency graph $\mathcal D=(L,E_{\mathcal D})$ is a directed acyclic graph that carries an edge $u\rightarrow v$ whenever the current proof or plan for $v$ uses helper lemma $u$.
The targets $L_\star\subseteq L$ are its roots, so helper-lemma work is useful only through a dependency path to some root.

\textbf{Intra-lemma AND/OR hypergraph.}
Within every target theorem $\ell_{\star,i}\in L_\star$ or generated helper lemma $\ell$, a lazily constructed AND/OR proof hypergraph $\mathcal H_\ell=(A_\ell,O_\ell,E_\ell)$ is a tree rooted at the lemma's initial goal.
\emph{AND nodes} $A_\ell$ are normalized Lean goal states, all of whose subgoals must be proved; \emph{OR nodes} $O_\ell$ are candidate proof steps (a tactic or tactic sequence), siblings being alternatives for the same goal.
Executing a step $o$ attaches the verified hyperedge $e=(o,\{g_1,\ldots,g_k\})\in E_\ell$ recording exactly the subgoals Lean returns, with $k=0$ when the step closes its goal.
A goal keeps at most one valid alternative at a time; an alternative discarded by the Recovery Reviewer (\Cref{sec:specialist-agents}) is retained as failure evidence while a new one is added.

\textbf{Artifact status.}
Each artifact is labeled \emph{open}, \emph{provisional}, \emph{conditionally solved}, \emph{stable}, or \emph{rejected}, and a goal with no admissible action is \emph{blocked}.
Only a Lean-accepted declaration whose generated dependencies are stable can itself become stable, and conditional plans guide search but never certify a root.

\textbf{Nonlocal effects.}
This structure makes the two nonlocal effects of our setting explicit: one stable helper lemma serving several consumers, and a failed provisional helper lemma invalidating exactly its conditional descendants, so a failure rolls back to the nearest safe point instead of discarding whole attempts.

\subsection{End-to-end algorithm}
\label{app:end-to-end-algorithm}

\Cref{alg:cape-loop} states the complete control loop of \ToolName{} (\Cref{sec:architecture-control-loop}): each epoch builds the frontier symbolically, lets the router group it and pick one group, purchases one routed action, and writes its verified result back, with initialization, evidence construction, decision validation, the two write-back paths, admissibility failure, and termination made explicit.

\begin{algorithm}[h]
  \caption{\ToolName{}: metalevel decision-making under cost on the two-level proof graph}
  \label{alg:cape-loop}
  \small
  \begin{algorithmic}[1]
    \Require environment $\Gamma_0$; targets $L_\star=\{\ell_1,\ldots,\ell_n\}$; specialist agents $O=\{\mathrm{DA},\mathrm{PW},\mathrm{LP},\mathrm{RR}\}$ with contracts $\{\Theta_o\}_{o\in O}$; horizon $t_{\max}$ (also stopped by the run's wall-clock and expenditure caps)
    \State $\mathcal D_0\gets(L_\star,\emptyset)$;\quad $\mathcal H_{\ell,0}\gets\{\mathrm{root}(\ell)\}$ and $\sigma_0(\ell)\gets\textit{open}$ for each $\ell\in L_\star$;\quad $\mathcal M_0,\mathcal L_0\gets\emptyset$
    \State $s_0\gets(\mathfrak G_0,\mathcal F_0,\sigma_0,\mathcal M_0,\mathcal L_0)$;\quad $\rho_0\gets$ initial routing rules \Comment{$\rho_t$ is the router's own variable, not part of $s_t$}
    \For{$t=0,1,2,\ldots$ \textbf{while} $s_t\notin\mathcal G$ \textbf{and} $t<t_{\max}$}
    \State $\sigma_t\gets\textsc{Propagate}(\mathfrak G_t,\sigma_t)$ \Comment{status closure over $\mathcal D_t$}
    \State $\mathcal F_t\gets\{g\in\textsc{Open}(\mathfrak G_t):\textsc{Prereq}(g)\subseteq\textsc{Elaborated}(\sigma_t)\}\setminus\textsc{RootDisc}(\mathcal D_t)$ \Comment{selection: incremental Kahn pass}
    \If{$\mathcal F_t=\emptyset$} \Return $\textsc{Report}(\sigma_t,\mathcal M_t,\mathcal L_t)$ \EndIf
    \State $U_t\gets\textsc{Sched}(s_t)=\mathcal F_t$ \Comment{the whole dependency-ready frontier, handed over in full}
    \State $\Gamma_t\gets\textsc{Cluster}(U_t)$;\quad $\mathcal C(s_t,U_t)\gets\{(G,o,\theta):G\in\Gamma_t,\ o\in O,\ \theta\in\Theta_o,\ \textsc{Guards}(G,o,\theta,s_t)=\top\}$ \Comment{router-side grouping, no LLM call}
    \If{$\mathcal C(s_t,U_t)=\emptyset$} $\sigma_t(U_t)\gets\textit{blocked}$;\ \textbf{continue} \EndIf
    \If{some $G\in\Gamma_t$ has $\textsc{DefNested}(G)\ \wedge\ (G,\mathrm{DA},\theta_{\mathrm{DA}})\in\mathcal C(s_t,U_t)$} \Comment{static analysis: no model call}
    \State $(G_t,o_t,\theta_t,j_t)\gets(G,\mathrm{DA},\theta_{\mathrm{DA}},\varnothing)$ \Comment{definition-nested goals go straight to automation}
    \Else
    \State $\mathrm{ev}(s_t,U_t)\gets\bigl\langle\,U_t,\ N_{\mathfrak G_t}(U_t),\ \mathcal M_t{\restriction}_{U_t},\ X_t(\cdot),\ \mathcal L_t,\ \{\Theta_o\}_{o \in O}\,\bigr\rangle$ \Comment{environmental evidence; the router adds its own rules $\rho_t$} \Comment{evidence packet; $\rho_t$: routing rules}
    \For{$r=1,\ldots,r_{\max}$} \Comment{action: metalevel decision under cost}
    \State $(G_t,o_t,\theta_t,j_t)\gets\textsc{Router}\bigl(\mathrm{ev}(s_t,U_t),\Gamma_t,\rho_t\bigr)$;\quad charge the call to $H$ \Comment{pick a group, an agent, a configuration; $j_t$: preliminary analysis}
    \State \textbf{break if} $(G_t,o_t,\theta_t)\in\mathcal C(s_t,U_t)$ \Comment{malformed or guard-violating decisions are rejected}
    \EndFor
    \If{no admissible decision after $r_{\max}$ retries} $(G_t,o_t,\theta_t,j_t)\gets\textsc{Default}(\mathcal C(s_t,U_t))$ \Comment{finite deliberation per epoch}\EndIf
    \EndIf
    \State $Y_t\gets\textsc{Exec}_{o_t}(G_t;\theta_t,j_t)$ \Comment{environment response on the chosen group; execution cost $\kappa_t$, deliberation charged to $H$}
    \If{$o_t=\mathrm{DA}$} \Comment{automation: Lean-side search, no model call}
    \State attach the verified steps of $Y_t$ to $\mathcal H_{\ell,t}$; cache a no-progress result against reruns on the unchanged state
    \ElsIf{$o_t=\mathrm{PW}$} \Comment{proof writer: iterative refinement within $\theta_t$}
    \State attach the verified steps of $Y_t$, keeping a partial proof when the refinement budget is exhausted
    \ElsIf{$o_t=\mathrm{LP}$} \Comment{lemma planner: elaborate, then back to the router}
    \State $\mathcal D_{t+1}\gets\mathcal D_t\cup\textsc{Elab}_{\circlearrowleft}(Y_t)$;
    \State new obligations $\to\mathcal F$
    \ElsIf{$o_t=\mathrm{RR}$} \Comment{recovery: nearest-safe-point backtracking}
    \If{$Y_t$ refutes some $g\in G_t$} reject that $g$ alone \textbf{else} switch path when judged no less provable and cheaper to finish \EndIf
    \If{$Y_t$ retracts artifact $u$} reopen $\mathrm{CondDesc}_t(u)$; retain unaffected verified subgraphs \EndIf
    \EndIf
    \State $\rho_{t+1}\gets\textsc{Amend}(\rho_t,U_t,c_t,Y_t)$ if the epoch was routed, else $\rho_t$ \Comment{rules evolve on the outcome just observed}
    \State $s_{t+1}\gets T(s_t,U_t,c_t,Y_t)$ \Comment{attach verified (hyper)edges; update $\sigma,\mathcal M$; append $\kappa_t$ to $\mathcal L$}
    \ForAll{$v$ proposed for promotion}
    \State $\sigma_{t+1}(v)\gets\textit{stable}$ \textbf{iff} $\textsc{Kernel}(v)=\top\ \wedge\ \textsc{Deps}(v)\subseteq\textsc{Stable}(\sigma_{t+1})$
    \EndFor
    \EndFor
    \State \Return $\{\ell\in L_\star:\sigma(\ell)=\textit{stable}\}$ with dependency closures;\ $\textsc{Report}(\sigma,\mathcal M,\mathcal L)$ for the rest
  \end{algorithmic}
\end{algorithm}

The operators are as follows.
$\textsc{Propagate}$ closes artifact status over $\mathcal D_t$: a stabilized helper lemma releases its consumers, a rejected or changed provisional artifact marks its conditional descendants invalidated.
$\textsc{Open}(\mathfrak G_t)$ is the set of open AND nodes; $\textsc{Prereq}(g)$ is the set of active prerequisites in $\mathcal D_t$; $\textsc{Elaborated}(\sigma_t)$ is the set of declarations with an elaborated statement, stable or provisional, against which a proof can be constructed, while $\textsc{Stable}(\sigma_t)$ is the smaller set that may be promoted; $\textsc{RootDisc}(\mathcal D_t)$ is the set of lemmas with no dependency path to a root, which are discarded; $\textsc{Sched}$ is the structural stage of \Cref{eq:controller-state-action}: it maintains the frontier and hands \emph{all} of it to the agentic router, $U_t=\mathcal F_t$. Grouping is the router's: $\textsc{Cluster}$ partitions $U_t$ into $\Gamma_t$ by embedding $k$-means (\Cref{app:goal-clustering}), the static-analysis check sends definition-nested groups straight to Deterministic Automation, and the router then chooses which group $G_t\in\Gamma_t$ to buy a computation for; $T$ is the deterministic update of \Cref{eq:purchased-computation}.
$\textsc{DefNested}$ is the static-analysis check: a goal whose obstruction is only unfolded definitions is dispatched to Deterministic Automation without a routing call, but only when that action itself passes the guards---after a no-progress automation attempt on an unchanged goal the guards exclude it, and the goal is routed instead--- and $\textsc{Amend}$ is the router's revision of its own rules $\rho_t$ after observing an outcome.
$\textsc{Guards}$ evaluates the mechanically checkable admissibility conditions (contract membership, acyclicity, available identifiers, no identical deterministic repetition); $r_{\max}$ bounds the router calls in one epoch, so deliberation is finite even when decisions are repeatedly malformed. $\textsc{Default}$ is the deterministic fallback used when that bound is reached: the admissible triple with the lowest contract cost bound (\ie Deterministic Automation whenever it is admissible) with groups ordered by $d_\sigma$ then statement hash, agents by a fixed order, and an empty analysis $j$.
Everything the chosen specialist sees is part of the decision: besides $j_t$, the excerpt of the router's conversation that the invocation carries is a field of $\theta_t$, and the update records it in $\mathcal M_{t+1}$ afterwards. Deliberation charges go to $H$, so $T$ stays a function of its four arguments.
One pass of the loop is a \emph{decision epoch}: it recomputes the frontier, reclusters it, and purchases exactly one computation on one group.
A \emph{controller iteration} is the reporting unit the case studies count (\Cref{app:case-study}): a boundary in the log that groups consecutive epochs, imposing no restriction on which actions they may take.
In $\mathrm{ev}(s_t,U_t)$, $N_{\mathfrak G_t}(U_t)$ is the graph neighborhood of $U_t$ (prerequisites, consumers, blocking relations, conditional descendants), $\mathcal M_t{\restriction}_{U_t}$ is the failure memory restricted to $U_t$, and $X_t(\cdot)$ is the at-risk spend of \Cref{eq:rollback-exposure}.
$O$ abbreviates the four specialist agents; $\textsc{Elab}_{\circlearrowleft}(Y_t)$ keeps only proposed lemmas that elaborate, preserve acyclicity, and survive random testing and the structural check; $\textsc{Kernel}(v)$ is Lean kernel acceptance of $v$'s proof and $\textsc{Deps}(v)$ are its generated dependent lemmas.
$\textsc{Report}$ emits the structured failure and diagnostics.

\begin{algorithm}[h]
  \caption{Dependency-aware scheduling: incremental Kahn maintenance of the frontier}
  \label{alg:scheduling}
  \begin{algorithmic}[1]
    \Statex \emph{Maintained state:} $d(\ell)=\bigl|\{u:(u\rightarrow\ell)\in E_{\mathcal D,t},\ u\ \text{not elaborated}\}\bigr|$ and $d_\sigma(\ell)=\bigl|\{u:(u\rightarrow\ell)\in E_{\mathcal D,t},\ \sigma(u)\neq\textit{stable}\}\bigr|$ for every $\ell\in L_t$;\quad $\mathcal F=\{g\in\textsc{Open}(\mathfrak G):d(\mathrm{decl}(g))=0\}\setminus\textsc{RootDisc}(\mathcal D)$
    \Procedure{OnEdgeAdded}{$u\rightarrow\ell$}
    \If{$u$ not elaborated} $d(\ell)\gets d(\ell)+1$; remove $\ell$'s goals from $\mathcal F$ \EndIf;\quad \textbf{if} $\sigma(u)\neq\textit{stable}$ \textbf{then} $d_\sigma(\ell)\gets d_\sigma(\ell)+1$
    \EndProcedure
    \Procedure{OnElaborated}{$u$}
    \ForAll{$(u\rightarrow\ell)\in E_{\mathcal D,t}$} $d(\ell)\gets d(\ell)-1$; \textbf{if} $d(\ell)=0$ \textbf{then} add $\ell$'s open goals to $\mathcal F$ \EndFor
    \EndProcedure
    \Procedure{OnStabilized}{$u$}
    \ForAll{$(u\rightarrow\ell)\in E_{\mathcal D,t}$} $d_\sigma(\ell)\gets d_\sigma(\ell)-1$ \Comment{$d_\sigma(\ell)=0$: $\ell$ may be promoted} \EndFor
    \EndProcedure
    \Procedure{OnRetracted}{$u$}
    \State reopen $\mathrm{CondDesc}(u)$, retaining unaffected verified subgraphs; restore $d(\cdot)$ and $d_\sigma(\cdot)$ on the reopened region; update $\mathcal F$
    \EndProcedure
    \Procedure{Sched}{$s$}
    \State \Return $\mathcal F$ in full, ordered by $d_\sigma$ \Comment{grouping and the static check are the router's (\Cref{sec:agent-routing})}
    \EndProcedure
  \end{algorithmic}
\end{algorithm}

\subsection{Dependency-aware scheduling}
\label{app:scheduling}

The scheduler discharges the selection half of the metalevel symbolically: it maintains the dependency-ready frontier and hands it to the router \emph{in full}, so no model call is spent on selection.
\Cref{alg:scheduling} states the maintenance precisely.
The frontier is kept by an \emph{incremental Kahn pass}~\cite{kahn1962topological} over the dynamic dependency graph $\mathcal D_t$: for every lemma/theorem $\ell$, the scheduler keeps the \emph{unelaborated in-degree} $d(\ell)$ and the frontier $\mathcal F_t$ consists exactly of the open AND nodes of declarations with $d(\ell)=0$, minus root-disconnected work.
A second counter, the \emph{unstable in-degree} $d_\sigma(\ell)$, counts every prerequisite that is not yet stable, elaborated or not: it does not gate the frontier, but $d_\sigma(\ell)=0$ marks the declarations whose proofs can be promoted, and the scheduler orders the frontier by it, so unconditional work is preferred to conditional work.
Kahn's algorithm processes zero-in-degree nodes of a static graph but here the graph changes as lemmas are added, elaborated, stabilized, and retracted, so the counters are updated per event at cost proportional to the affected out-degree. Elaborating $u$ decrements $d(\ell)$ for each consumer $\ell$ and releases those reaching zero; stabilizing $u$ decrements $d_\sigma(\ell)$ and enables promotion; retracting $u$ reopens exactly its conditional descendants and restores their counters rather than re-sorting the graph.
Two properties of the frontier follow from its definition and carry the scheduler's cost rationale: work with no dependency path to a root never enters it, and no goal is opened beneath a prerequisite that has not even been elaborated. Work beneath a \emph{provisional} prerequisite is admitted and its exposure is what \Cref{eq:rollback-exposure} measures and the Recovery Reviewer unwinds.

Groups that reach the Lemma Planner are additionally validated for compatibility: a valid group requires compatible environments and similarity in normalized targets, hypotheses, source definitions, or recursion structure.
This restriction prevents surface similarity from becoming an unchecked dependency.

\subsection{Goal clustering by semantic embedding}
\label{app:goal-clustering}

Batching reduces the cost of router and specialist invocations: a single call jointly analyzes several related goals, and proposals for shared helpers draw evidence from multiple potential consumers. These savings, however, depend on grouping goals that are genuinely related.
The router forms such groups, on the frontier the scheduler hands it, by applying $k$-means clustering to semantic embeddings of the normalized goal statements. DreamProver~\cite{zhang2026dreamprover} uses the similar clustering method but they use it to evolve lemmas.

\textbf{Embedding and similarity.}
Each open goal $g \in \mathcal F_t$ is rendered as its normalized statement using the tactic state and embedded with the Lean-statement embedding model of Lean Finder~\cite{lu2026leanfinder}, which is trained on formal Lean statements.
Writing $\phi(g) \in \mathbb R^d$ for the raw embedding, we store the unit-normalized vector $x_g = \phi(g)/\lVert\phi(g)\rVert_2$ in an index keyed by the statement hash, so embeddings are computed once per normalized statement and reused as the frontier evolves. Similarity between goals is cosine similarity: $\mathrm{sim}(g,g') = \langle x_g, x_{g'}\rangle$.

\textbf{$k$-means on the unit sphere.}
Given the current frontier embeddings $\{x_g\}_{g\in\mathcal F_t}$, spherical $k$-means seeks unit centroids $\mu_1,\ldots,\mu_k$, with $\lVert\mu_j\rVert_2=1$, and an assignment into clusters $C_1,\ldots,C_k$ minimizing
\begin{equation}
  \sum_{j=1}^{k}\ \sum_{g\in C_j} \bigl\lVert x_g-\mu_j \bigr\rVert_2^2 ,
  \label{eq:kmeans-objective}
\end{equation}
Because both $x_g$ and $\mu_j$ lie on the unit sphere and an unconstrained mean centroid has $\lVert x_g-\mu_j\rVert_2^2 = 1+\lVert\mu_j\rVert_2^2-2\langle x_g,\mu_j\rangle$, we have $\lVert x_g-\mu_j\rVert_2^2 = 2 - 2\langle x_g,\mu_j\rangle$, so minimizing \Cref{eq:kmeans-objective} is equivalent to maximizing within-cluster cosine similarity, under the similarity the index already uses.
We run Lloyd's iteration: initialize centroids by $k$-means\texttt{++} seeding from a fixed seed derived from the hash of the frontier, so that the grouping is a deterministic function of the state. We repeat: (i) assign each $x_g$ to its nearest centroid and (ii) update each $\mu_j$ to the mean of its cluster, re-normalized to the unit sphere; stop when assignments are unchanged. An empty cluster, or one whose mean vanishes, is reseeded deterministically with the frontier goal furthest from its assigned centroid, ties broken by statement hash, so that the grouping stays a total deterministic function.
The cluster count is set from the batch budget, $k = \lceil |\mathcal F_t| / b \rceil$ for a fixed target batch size $b$, and singleton clusters are permitted.

\subsection{Sharing and rollback accounting}
\label{app:sharing-rollback}

Shared helper lemmas can reduce duplicate generation, but failed helper lemmas create rollback costs.
\ToolName{} exposes both effects without asking the router to fill unknown quantities.
If $K$ consumers would otherwise rediscover helper lemma $h$, define
\begin{equation}
  \begin{aligned}
    C_{\mathrm{iso}}(h)
     & =\sum_{i=1}^{K}(c_{h,i}+d_i), \\
    C_{\mathrm{share}}(h)
     & =c_h+\sum_{i=1}^{K}d_i+r_h,
  \end{aligned}
  \label{eq:sharing-costs}
\end{equation}
where $c_{h,i}$ is the work spent generating helper lemma $h$ for consumer $i$, $c_h$ is the realized cost of proving the shared helper lemma once, $d_i$ is the remaining work for consumer $i$, and $r_h$ is expenditure on failed attempts to prove the helper lemma and on rollback.
By \Cref{eq:sharing-costs}, sharing saves money exactly when $C_{\mathrm{share}}(h)<C_{\mathrm{iso}}(h)$.

The lemma-dependency graph also makes current exposure measurable.
For provisional artifact $u$, \ToolName{} reports
\begin{equation}
  X_t(u)=\sum_{v\in\mathrm{CondDesc}_t(u)} c_t^{\mathrm{sunk}}(v),
  \label{eq:rollback-exposure}
\end{equation}
where $\mathrm{CondDesc}_t(u)$ is the set of conditional descendants that require revalidation if $u$ changes, and $c_t^{\mathrm{sunk}}(v)$ is the realized cost already charged to descendant $v$ by epoch $t$.
A large value in \Cref{eq:rollback-exposure} is evidence for resolving $u$ before further deepening.
This quantity enters the router's evidence packet in \Cref{sec:agent-routing}. When a helper lemma is rejected, \ToolName{} reopens its conditional descendants while retaining verified subgraphs unaffected by the failed dependency.

\subsection{Message passing between the router and the action agents}
\label{app:router-contract}

\ToolName{} treats the router's context as the head of the conversation that the chosen agent continues, rather than as a separate request.
Concretely, a routing decision hands the action agent the router's message history together with the preliminary analysis $j$ that the router produced while making the decision.
The agent appends its own turns to that history instead of opening a fresh one, so the shared prefix is served from the provider's prompt cache at the cache-hit price, and the agent inherits the router's reading of the goals rather than re-deriving it.
The ledger records cached and uncached input tokens separately, so the realized effect of this reuse is visible in the reported costs rather than assumed.

\subsection{Structural checking of proposed lemmas}
\label{app:structural-check}

A helper lemma proposed by the Lemma Planner enters the dependency graph only after it elaborates and preserves acyclicity, but elaboration only establishes that the statement is well-formed, not that it is true.
Committing to a false helper lemma is expensive: consumers are proved against it, and the work beneath it must be discarded when Lean eventually refuses the proof.
\ToolName{} therefore tries to refute a candidate before any proof effort is spent on it, in two escalating steps.

\textbf{Step 1: random testing.}
For a candidate whose statement is executable, we sample concrete values for its universally quantified variables, instantiate the statement, and evaluate it.
This is cheap, requires no model call, and catches the common failures of a plausible-looking conjecture such as a missing side condition, an off-by-one bound, or a hypothesis that the planner assumed but never stated.
A single falsifying instance disproves the lemma outright.

\textbf{Step 2: LLM-proposed structured instances.}
Random sampling explores the space uniformly and rarely reaches the corner cases where verification conditions actually fail: a singleton, an index at the boundary, an overflow value, or the state just before and after a loop's last iteration.
Constructing those instances takes reasoning about the statement rather than sampling, so when Step~1 finds no counterexample, we ask a model to propose a small set of structured instances, and evaluate them the same way.
Step~2 runs only if Step~1 passes, so its cost is paid only for candidates that have already survived the free check.

If either step produces a counterexample, the candidate is rejected. The counterexamples, with the instance that triggered it, are returned to the Lemma Planner, which refines the statement, most often by strengthening a hypothesis or restricting a bound, and proposes a revised lemma.
If both steps pass, the lemma is admitted to the dependency graph as \emph{provisional}.
Passing these checks is evidence rather than a guarantee: testing can refute a candidate but cannot promote it, and only Lean acceptance makes a declaration stable.

%% file: appendix/additional_results.tex
\section{Additional Experiment Results}
\label{app:additional-results}

\subsection{Additional Explanation of Experiment Setup}
\label{app:additional-experiment-setup}

\textbf{Experiment setup.}
Every method runs under the same external limits: each function-level problem is capped at 5 hours of wall-clock time and \$30 of expenditure (including NTP4VC since one task has one obligation in its benchmark design), and each repository-level task at 8 hours and \$60.
All methods access each backbone model through the same official provider: Gemini-3.5-Flash via Google Vertex AI, GPT-5.6-Terra via OpenAI, and DeepSeek-V4-Flash via DeepSeek, so token prices and caching behavior are identical across methods.
Cost is measured the same way for every method: we collect the token counts of every model call by type (uncached input, cached input, cache write, and output tokens, with reasoning tokens billed as output) and convert them to US dollars under the provider's official pricing (\Cref{tab:model-pricing}).
All experiments run on a machine with 16 vCPUs and 256\,GB of memory, which hosts Lean checking and all local computation while model inference runs through the provider APIs, with 12 concurrent prover runs.
We use Lean Finder~\cite{lu2026leanfinder} as the retrieval tool for the methods and deploy it as a service on a machine with 4 vCPUs, 16\,GB of memory and one NVIDIA L4 GPU. After each run, the generated Lean proof is checked by the independent Lean proof reviewer checker described below.

\textbf{Model pricing.}
\Cref{tab:model-pricing} lists the official prices used to convert token counts into the dollar costs reported throughout the paper.
We fetched the GPT and Gemini prices from the providers' official pricing pages on July~19, 2026, and the DeepSeek prices on August~17, 2026.
For GPT and Gemini we use the standard-tier prices; DeepSeek prices vary by time of day, and we use the peak-hour rates (off-peak rates are half), so no DeepSeek cost is understated.
Every run of every method is costed under this single price schedule, so cost comparisons do not depend on when a run was executed.
The three backbones are shared by \ToolName{} and all baselines; GPT-5.6-Sol is used only to generate the CLEVER and AlgoVeri implementations before any prover runs, so its cost is not charged to any method.

\begin{table}[h]
    \centering
    \caption{\textbf{Official model prices}, in US dollars per million tokens.
        \emph{Input} applies to uncached input tokens, \emph{Cache hits} to input tokens read from the provider's prompt cache, and \emph{Output} to output tokens, including reasoning tokens.
        Sources: OpenAI, \url{https://developers.openai.com/api/docs/pricing}, and Google, \url{https://ai.google.dev/gemini-api/docs/pricing}, both fetched with standard tier on July~19, 2026; DeepSeek, \url{https://api-docs.deepseek.com/quick_start/pricing/} with peak-hour rates, fetched on August~17, 2026.}
    \label{tab:model-pricing}
    \small
    \begin{tabular}{llrrrr}
        \toprule
        \textbf{Model}    & \textbf{Provider} & \textbf{Input} & \textbf{Cache hits} & \textbf{Cache write} & \textbf{Output} \\
        \midrule
        GPT-5.6-Terra     & OpenAI            & 2.50           & 0.250               & 3.125                & 15.00           \\
        Gemini-3.5-Flash  & Google Vertex AI  & 1.50           & 0.150               & --                   & 9.00            \\
        DeepSeek-V4-Flash & DeepSeek          & 0.44           & 0.014               & --                   & 1.32            \\
        \bottomrule
    \end{tabular}
\end{table}

\textbf{Excluded baselines.}
We do not compare against four other open-source provers, each for a different reason.

\emph{Hilbert}~\cite{sumanth2026hilbert} was built for MiniF2F and PutnamBench, where a problem is a self-contained theorem: the file header holds only \texttt{import Mathlib}, \texttt{open} declarations, and \texttt{set\_option} lines, and everything the theorem mentions is Mathlib vocabulary the model already knows.
Program verification benchmarks break this assumption.
Their header carries the actual mathematical content---\texttt{def implementation} and \texttt{def problem\_spec}---and the theorem is a single line that refers to these definitions.
Hilbert's decomposition does not treat the header as part of the problem, and this also causes the wasted assemblies we observed: when stating a subgoal, the model adds hypotheses about \texttt{implementation}, not recognizing it as a concrete definition it could unfold, so the subgoal's signature no longer matches its use site and the sketch fails to assemble.

\emph{Goedel-Code-Prover}~\cite{li2026goedelcodeprover} relies mainly on its own fine-tuned model; its agentic framework is a simple pipeline of lemma decomposition and iterative refinement whose effectiveness comes from sampling many model answers.
Under our protocol---one attempt per problem (pass@1) with a frontier backbone such as Gemini-3.5-Flash---it failed to prove problems within a \$5 budget that both \ToolName{} and the coding agents prove for under \$0.50 each.
Given this gap, which stems from a framework designed around a fine-tuned model and repeated sampling rather than from the backbone, a matched-budget comparison would not be informative.

\emph{WybeCoder}~\cite{gloeckle2026wybecoder} and $\mathrm{P}^3$~\cite{li2026p3} target end-to-end verifiable code generation, in which generating the implementation and generating its proof guide each other.
In our setting the implementation and specification are fixed inputs (\Cref{sec:evaluation}), so their central mechanism does not apply to pure proof generation.

\textbf{Benchmark licenses.}
\Cref{tab:benchmark-licenses} lists the license of each benchmark we evaluate on, as stated in its official repository.
We use all five benchmarks only for research evaluation, in accordance with their licenses, and credit each through its original paper; we do not redistribute any benchmark data.
The implementations we generate with GPT-5.6-Sol for CLEVER and AlgoVeri are derived from their specifications and remain subject to those benchmarks' licenses.
NTP4VC's tooling is MIT-licensed, while each verification condition, including its Lean~4 translation, keeps the license of the C project it is derived from.
Each Vero task is a Lean~4 translation of a third-party repository: tasks derived from permissively licensed upstream projects are distributed under Vero's Apache-2.0 license with attribution, while tasks derived from copyleft upstream projects keep their upstream licenses.

\begin{table}[h]
    \centering
    \caption{\textbf{Benchmark licenses.}
        $^\ddagger$NTP4VC's tooling is released under MIT; its verification conditions keep the licenses of their upstream C projects, which include GPL-2.0, GPL-3.0, LGPL-2.1, BSD-3-Clause, and MIT.
        $^\dagger$Vero's pipeline, specifications, and permissively licensed tasks are released under Apache-2.0; tasks derived from copyleft upstream projects keep their upstream licenses: \texttt{flocq} and \texttt{portion} (LGPL-3.0-or-later) and \texttt{huffman} (LGPL-2.1-or-later).}
    \label{tab:benchmark-licenses}
    \footnotesize
    \setlength{\tabcolsep}{4pt}
    \begin{tabular}{lll}
        \toprule
        \textbf{Benchmark}                & \textbf{License}     & \textbf{Repository}                                                                        \\
        \midrule
        CLEVER~\cite{thakur2025clever}   & MIT                  & \href{https://github.com/trishullab/clever}{\texttt{github.com/trishullab/clever}}         \\
        VERINA~\cite{ye2026verina}       & Apache-2.0           & \href{https://github.com/sunblaze-ucb/verina}{\texttt{github.com/sunblaze-ucb/verina}}     \\
        AlgoVeri~\cite{zhao2026algoveri} & Apache-2.0           & \href{https://github.com/haoyuzhao123/algoveri}{\texttt{github.com/haoyuzhao123/algoveri}} \\
        NTP4VC~\cite{xu2026ntp4vc}       & Mixed$^\ddagger$     & \href{https://github.com/xqyww123/NTP4VC}{\texttt{github.com/xqyww123/NTP4VC}}             \\
        Vero~\cite{ye2026vero}           & Apache-2.0$^\dagger$ & \href{https://github.com/sunblaze-ucb/vero}{\texttt{github.com/sunblaze-ucb/vero}}         \\
        \bottomrule
    \end{tabular}
\end{table}

\textbf{Baseline versions.}
\Cref{tab:baseline-versions} pins every baseline to the exact release or commit we ran, so each comparison can be reproduced against the same code.
The three coding agents are installed from npm at the listed versions, each driving the backbone of its own provider; Humanize and Ax-Prover-Base are run from the listed commits of their public repositories.

\textbf{Lean proof reviewer checker.}
A run counts as solved only if its proof survives an independent review, applied identically to \ToolName{} and to every baseline.
The review is needed because a failing agent can end a run in ways that compile but prove nothing: leaving a placeholder, quietly editing the theorem it was asked to prove, or assuming the goal as a new axiom.
After a run finishes, we take the produced Lean file and apply five checks, the first three mechanical, the last two aimed at proofs that pass the mechanical checks yet are not genuine.
\begin{enumerate}[leftmargin=1.5em,itemsep=1pt,topsep=2pt]
    \item \textbf{No unsolved goals.}
          We re-elaborate the file with the Lean checker and reject any run whose target theorems still contain an unproved hole, whether written as \texttt{sorry} or \texttt{admit}, or left open by a placeholder tactic such as \texttt{apply?}; the elaborator's own \texttt{sorry}-containing and error warnings are treated as failures rather than diagnostics.
    \item \textbf{Unmodified statements.}
          The theorem the agent proves must be the theorem it was given.
          We compare each target declaration against the benchmark original by string match on the statement and by re-elaborating it in the benchmark environment, so a weakened hypothesis, a renamed constant, a changed binder, or a restated conclusion is rejected even when the accompanying proof compiles.
    \item \textbf{No new axioms.}
          We ask the Lean checker for the axioms each target theorem's proof actually depends on and require that set to lie within the benchmark's own environment.
          A proof that introduces an \texttt{axiom}, admits a declaration through an axiomatic definition, or otherwise extends the trusted base is rejected.
    \item \textbf{Automated review.}
          The three checks above are syntactic, so we additionally run Codex with GPT-5.6-Sol over each accepted proof, with the original statement and the produced file, and ask it to report placeholders, unproved obligations, and proofs that are vacuous or only appear to establish the statement.
          Any proof it flags is re-examined by hand.
    \item \textbf{Manual audit.}
          Finally, we sample passing proofs at random across benchmarks, methods, and backbones and inspect them manually to confirm that they are genuine and that the preceding checks are not missing a failure mode.
\end{enumerate}
Every solve rate and cost reported in this paper counts only runs that pass all five checks.

\begin{table}[h]
    \centering
    \caption{\textbf{Baseline versions.}
        Coding agents are pinned by npm release; Humanize and Ax-Prover-Base by Git commit, abbreviated to seven characters (each entry links to the full commit).
        Full commit hashes: Humanize \texttt{8e21ff211daa731fa826c491e1f12ed284135c0d}; Ax-Prover-Base \texttt{06dfadc9ab439755af5efcfe0add95bfef2733c7}.}
    \label{tab:baseline-versions}
    \footnotesize
    \setlength{\tabcolsep}{4pt}
    \begin{tabular}{llll}
        \toprule
        \textbf{Baseline} & \textbf{Backbone} & \textbf{Package / GitHub repository}                                                                                                      & \textbf{Version} \\
        \midrule
        Codex             & GPT-5.6-Terra     & \href{https://www.npmjs.com/package/@openai/codex/v/0.146.0}{\texttt{@openai/codex}}                                                      & 0.146.0          \\
        Gemini CLI        & Gemini-3.5-Flash  & \href{https://www.npmjs.com/package/@google/gemini-cli/v/0.55.1}{\texttt{@google/gemini-cli}}                                             & 0.55.1           \\
        DeepSeek Harness  & DeepSeek-V4-Flash & \href{https://www.npmjs.com/package/@deepseek-ai/dsh/v/0.1.2-alpha.2}{\texttt{@deepseek-ai/dsh}}                                          & 0.1.2-alpha.2    \\
        \midrule
        Humanize          & all three         & \href{https://github.com/humanfia/putnambench/tree/8e21ff211daa731fa826c491e1f12ed284135c0d}{\texttt{humanfia/putnambench}}               & \texttt{8e21ff2} \\
        Ax-Prover-Base    & all three         & \href{https://github.com/Axiomatic-AI/ax-prover-base/tree/06dfadc9ab439755af5efcfe0add95bfef2733c7}{\texttt{Axiomatic-AI/ax-prover-base}} & \texttt{06dfadc} \\
        \bottomrule
    \end{tabular}
\end{table}

\subsection{Per-Problem Cost and Time}
\label{app:per-problem-cost-time}

\Cref{fig:clever-cost-time} opens the CLEVER results into one point per solved problem, comparing \ToolName{} with the coding agent of each backbone.
Two patterns hold across backbones.
At comparable wall-clock time its points also lie below the coding agent's and avoid their high-cost tail, which reaches about \$30 per problem under Gemini-3.5-Flash against about \$13 for \ToolName{}.
The average cost also covers the harder problems the baselines never solve so it is higher than the coding agents on GPT model.
\ToolName{} has a higher mean wall-clock time under GPT-5.6-Terra and Gemini-3.5-Flash, but a lower mean under DeepSeek-V4-Flash.
Thus, wall-clock time is the one axis where \ToolName{} does not win: a coding agent writes a whole proof in one shot, whereas \ToolName{} selects, routes, and Lean-checks one goal before purchasing the next.
That sequential, verified progress is what keeps cost low but it takes longer, which is also discussed in \Cref{app:limitations}.

\input{figures/bench_data_points.tex}

\subsection{Discussion of DeepSeek Model on Baselines}
\label{app:discussion-deepseek-baselines}

The gap between \ToolName{} and the baselines is widest under DeepSeek-V4-Flash, and the reason is that the baselines are not built for how this model spends its reasoning budget.
Two of the three baseline families run inside an agent scaffold designed around another provider's model.
Humanize is built on general coding agents such as Codex and Claude Code, so running it on DeepSeek-V4-Flash means driving a Codex-shaped loop with a model the loop was not tuned for. The coding-agent baseline uses DeepSeek's own harness, which is the intended scaffold but still gives the model one large context holding the whole task.
Both arrangements share the same assumption: the model is asked to reason about a long context and produce a proof in one turn.

DeepSeek-V4-Flash usually has long reasoning traces and spends much of its wall-clock budget inside them, and the effect compounds in an agentic setting, where each turn carries the accumulated transcript.
We observe two resulting failure modes.
In the first, a run is still reasoning when the wall-clock cap expires.
In the second, the model finishes reasoning and then declines the task, reporting that the goal is too hard or that it cannot complete the proof.
Both leave the run unsolved after real spending, which is why the cost saving on commonly solved problems is largest there.

However, \ToolName{}'s fine-grained control on proof graphs supplies exactly that: the router fixes a reasoning effort and working budget per invocation, and each specialist sees only its own goals with the prepared context it needs.
This is why the framework can transfer across backbone LLMs.

\subsection{Run-to-Run Variance}
\label{app:run-variance}

Model outputs are stochastic, so a single run can overstate or understate a system's cost.
We measure this variance at two scales: over the three function-level benchmarks in full, which tells us whether the headline results rest on a lucky run, and over the 30-problem ablation subset, which sets the noise floor against which a component ablation must be read.
Both use the same provider, caps, and proof-review procedure as the main evaluation.

\textbf{Full function-level benchmarks.}
We run the complete \ToolName{} three times with Gemini-3.5-Flash on all of CLEVER, VERINA, and AlgoVeri, 427 problems in total.
Total cost over the three benchmarks is \$969.8 on average with a standard deviation of \$89.5, a coefficient of variation of 9.2\%.
The spread is therefore an order of magnitude smaller than the cost gaps the comparison rests on (\Cref{tab:cost-performance}), so the reported savings are not an artifact of which run we report.
The solved set is identical across the three runs on all three benchmarks, so the solve rates of \Cref{tab:provability-performance} are stable and the variance is in cost alone.

\textbf{Ablation subset across backbones.}
We then run the complete \ToolName{} three times with each of the three backbones on the 30 ablation problems of \Cref{sec:eval-ablation}---10 each from CLEVER, VERINA, and AlgoVeri, spanning difficulty levels from easy to hard.
\Cref{tab:stability} reports, for each backbone and benchmark, the mean and standard deviation of cost over the three runs, together with the coefficient of variation (CV, the standard deviation divided by the mean); every run solves all 30 problems, so the spread is in cost alone and not in what was proved.

Six of the nine cells have a CV below 9\%, and the two largest, AlgoVeri with GPT-5.6-Terra (34.2\%) and CLEVER with DeepSeek-V4-Flash (22.0\%), arise for different reasons.
The AlgoVeri figure is genuine cost dispersion: the standard deviation is \$5.29 on a \$15.47 mean, the largest absolute spread in the table, and AlgoVeri's algorithmic goals admit proofs of very different lengths, so a single unlucky decomposition changes the bill substantially.
The DeepSeek figure is an artifact of scale: its standard deviation is only \$0.93 but DeepSeek-V4-Flash is cheap enough (\$3.00--\$7.66 per benchmark) that a small absolute spread is a large relative one.
Absolute variation is therefore modest everywhere, and the ranking of backbones by cost is stable across all three runs.

The Gemini-3.5-Flash means are the complete-system figures used in the component ablation (\Cref{tab:ablation-cost}), and their CVs (2.4\%--8.3\%) indicate how large an ablation difference must be to exceed run-to-run noise.
The router and Lemma Planner ablations, at 39.9\% and 41.6\% above the complete system, are far outside this range; the automation ablation, at 5.8\% overall, is comparable to it, and we therefore read that component's effect as small rather than precisely measured.

\input{tables/stability.tex}

\subsection{More Component Ablation Discussion}
\label{app:more-component-ablations}

The component ablation of \Cref{sec:eval-ablation} removes one action agent at a time, or the router's adaptive choice among them, while keeping the rest of the system intact.
Two further components of \ToolName{} admit no such ablation, and we explain here why we do not report one rather than leaving the omission unexplained.

\textbf{The two-level proof graph.}
The graph is not a component that the rest of the system uses. It is the state the rest of the system is defined over.
Removing it leaves nothing to select from: there is no dependency-ready frontier to maintain, no topological order in which goals become available, no consumers to share a helper lemma with, and no conditional descendants to reopen when a provisional lemma is retracted.
The scheduler, the router's grouping, and the Recovery Reviewer's nearest-safe-point backtracking all read the graph, so a variant without it is not \ToolName{} with one part disabled but a different system---one that proves each obligation in isolation, which is what the baselines already do.
The controlled comparison for the graph is therefore the baseline comparison of \Cref{sec:eval-performance}, and the case studies of \Cref{app:case-study} show concretely what the graph buys on individual tasks.

\textbf{The adaptive routing rules.}
Removing the explicit rules does not produce a router that stops adapting.
The evidence packet still carries the execution history from shared memory so a router without written rules can read the same signals and change its behavior in response to them, only implicitly and without a logged, versioned artifact we can inspect.
Such an ablation would therefore measure the value of \emph{writing the rules down}, not the value of adaptation, and a difference in either direction would be hard to attribute: a null result would not show that adaptation is unnecessary, and a positive one would not separate the rules from the extra tokens that carrying them adds to every packet.
What the experiments do isolate is adaptive action choice as a whole: \emph{\ToolName{} w/o Router} replaces it with one fixed order over the same graph and raises cost by 39.9\% (\Cref{tab:ablation-cost}).

%% file: figures/bench_data_points.tex
\begin{figure}[t]
    \centering
    \captionsetup[subfigure]{justification=centering}
    \begin{subfigure}[t]{0.325\linewidth}
        \centering
        \includegraphics[width=\linewidth]{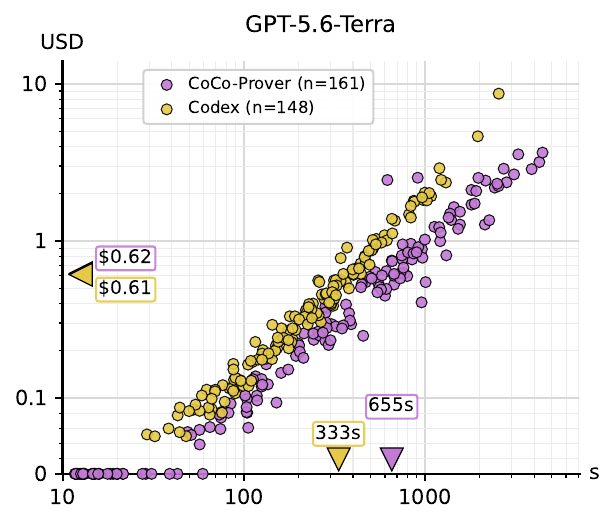}
        \caption{GPT-5.6-Terra: \ToolName{} vs.\ Codex.}
        \label{fig:clever-cost-time-gpt}
    \end{subfigure}\hfill
    \begin{subfigure}[t]{0.325\linewidth}
        \centering
        \includegraphics[width=\linewidth]{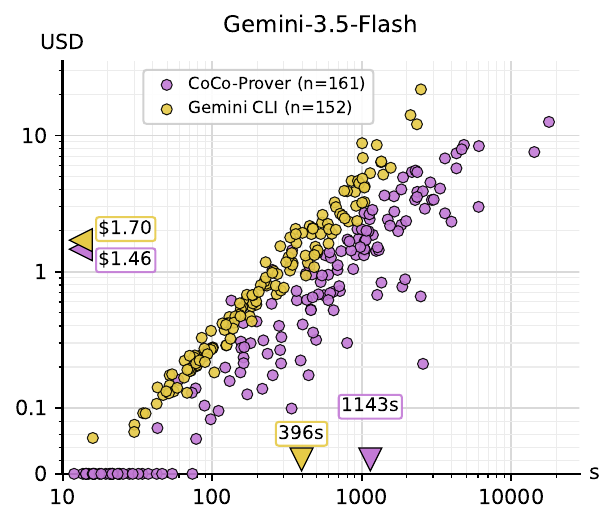}
        \caption{Gemini-3.5-Flash: \ToolName{} vs.\ Gemini CLI.}
        \label{fig:clever-cost-time-gemini}
    \end{subfigure}\hfill
    \begin{subfigure}[t]{0.325\linewidth}
        \centering
        \includegraphics[width=\linewidth]{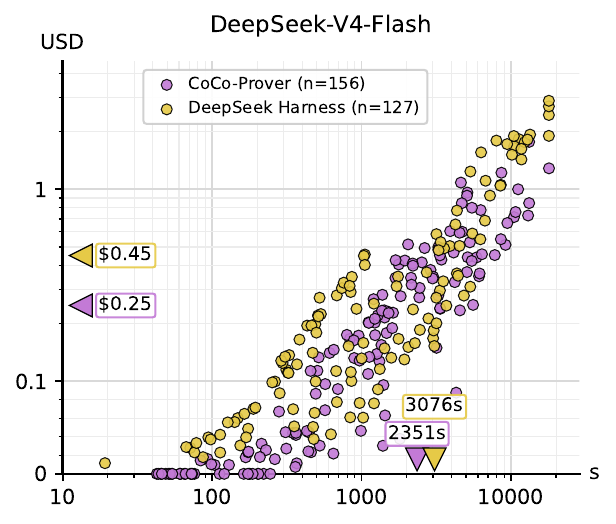}
        \caption{DeepSeek-V4-Flash: \ToolName{} vs.\ DeepSeek Harness.}
        \label{fig:clever-cost-time-deepseek}
    \end{subfigure}
    \caption{\textbf{Per-problem cost and wall-clock time of every solved problem in CLEVER benchmark, for \ToolName{} and each coding agent.}
        Each point is one problem solved and plotted at its wall-clock time and monetary cost.
        Triangles on the axes mark each system's mean cost ($\blacktriangleleft$) and mean time ($\blacktriangledown$) over its own solved set.}
    \label{fig:clever-cost-time}
    \vspace{-0.5em}
\end{figure}

%% file: tables/stability.tex
\begin{table}[h]
    \centering
    \caption{\textbf{Run-to-run variance of \ToolName{}.}
        We run the complete system three times with each backbone on the 30 ablation problems (10 each from CLEVER, VERINA, and AlgoVeri) and report the mean $\pm$ standard deviation of total cost in USD.
        Parenthesized percentages are coefficients of variation (standard deviation divided by the mean).
        Every run solves all 30 problems.
        The Gemini-3.5-Flash means are the complete-system figures used in \Cref{tab:ablation-cost}.}
    \label{tab:stability}
    \small
    \setlength{\tabcolsep}{4pt}
    \renewcommand{\arraystretch}{1.15}
    \resizebox{\linewidth}{!}{        \begin{tabular}{lccc}
            \toprule
            \textbf{Backbone}
                              & \textbf{CLEVER Cost (USD)}
                              & \textbf{VERINA Cost (USD)}
                              & \textbf{AlgoVeri Cost (USD)}                                                        \\
            \midrule
            GPT-5.6-Terra     & 17.85 $\pm$ 0.80 (4.5\%)     & 30.32 $\pm$ 2.54 (8.4\%) & 15.47 $\pm$ 5.29 (34.2\%) \\
            Gemini-3.5-Flash  & 28.64 $\pm$ 1.61 (5.6\%)     & 43.49 $\pm$ 1.06 (2.4\%) & 38.82 $\pm$ 3.22 (8.3\%)  \\
            DeepSeek-V4-Flash & 4.22 $\pm$ 0.93 (22.0\%)     & 7.66 $\pm$ 0.91 (11.9\%) & 3.00 $\pm$ 0.23 (7.7\%)   \\
            \bottomrule
        \end{tabular}    }
\end{table}

%% file: appendix/case_study.tex
\section{Case Studies}
\label{app:case-study}

We walk through two tasks end to end: a function-level task, CLEVER \texttt{prob\_15} (\Cref{app:case-study-clever}), where all systems face one obligation and the question is cost; and a repository-level task, Vero \texttt{prob\_14} (\Cref{app:case-study-vero}), with 56 interdependent obligations, where the question is whether provability can be reached at all.

\subsection{CLEVER \texttt{prob\_15}: a function-level cost comparison}
\label{app:case-study-clever}

\ToolName{} proves CLEVER \texttt{prob\_15} for \$2.36; the strongest baselines prove it for \$8.67 (Codex) and \$10.74 (Humanize), and Ax-Prover-Base spends \$5.48 and fails.
All four systems use the same LLM (GPT-5.6-Terra) at the same provider prices under the caps of the main evaluation.

\subsubsection{The task}
\label{app:case-study-task}

\texttt{prob\_15} is HumanEval's \texttt{string\_sequence}: return the numbers $0..n$ joined by single spaces.

\begin{lstlisting}[style=lean]
def problem_spec
(implementation: Nat → String)
(n: Nat) :=
let spec (result: String) :=
  let result_nums := result.splitOn " ";
  result_nums.length = n + 1 ∧
  ∀ i, i < n + 1 → result_nums[i]! = i.repr;
∃ result, implementation n = result ∧ spec result

def implementation (n: Nat) : String :=
  " ".intercalate ((List.range (n + 1)).map (fun i => i.repr))

theorem correctness (n: Nat) : problem_spec implementation n := sorry
\end{lstlisting}

The specification reads the output back with \texttt{String.splitOn " "}, so the proof must show that splitting on the separator undoes \texttt{String.intercalate} whenever no piece contains the separator and Lean~4 core has no such theorem:
\texttt{String.splitOn} is a byte-position scanner (\texttt{String.splitOnAux}) using well-founded recursion over UTF-8 offsets; \texttt{String.intercalate} is implemented by a private accumulator worker \texttt{String.intercalate.go}; the only ready-made fact is the list-level \texttt{List.splitOn\_intercalate}; and the proof also needs ``\texttt{Nat.repr} never produces a space,'' which requires reasoning about \texttt{Nat.toDigits}.
A proof therefore has to connect three layers: the string API, character lists, and the byte-level scanner. Every baseline named exactly this connection as its obstacle (\Cref{app:case-study-baselines}).

\subsubsection{Head-to-head outcome}
\label{app:case-study-results}

\begin{table}[h]
  \centering
  \caption{\textbf{CLEVER \texttt{prob\_15} under identical LLM, prices, and caps.}
    Ax-Prover-Base's 288{,}669 output tokens include 228{,}209 reasoning tokens.
    Codex and Humanize find proofs of similar size and structure to ours; most of their extra cost is cache reads from re-sending an ever-growing context at every step.
    Ax-Prover-Base keeps its context small but spends \$4.33 on output, mostly reasoning, re-deriving a monolithic proof each iteration, and never compiles a complete proof.}
  \label{tab:case-study-headline}
  \small
  \resizebox{\linewidth}{!}{    \begin{tabular}{lrrrr}
      \toprule
                            & \textbf{\ToolName{}} & Codex                & Humanize              & Ax-Prover-Base          \\
      \midrule
      Outcome               & \textbf{proved}      & proved               & proved                & failed (\texttt{sorry}) \\
      Total cost (USD)      & \textbf{2.36}        & 8.67 ($3.7\times$)   & 10.74 ($4.5\times$)   & 5.48                    \\
      Wall-clock time (s)   & 2{,}816              & 2{,}550              & 4{,}538               & 6{,}893                 \\
      Uncached input tokens & 425{,}737            & 408{,}271            & 1{,}000{,}932         & 452{,}135               \\
      Cache-read tokens     & 2{,}115{,}584        & 26{,}675{,}200       & 23{,}777{,}280        & 77{,}568                \\
      Total prompt tokens   & \textbf{2.54M}       & 27.08M               & 24.78M                & 0.53M                   \\
      Output tokens         & 51{,}310             & 65{,}566             & 152{,}920             & 288{,}669               \\
      Proof size            & 191 lines, 9 lemmas  & 149 lines, 8 helpers & 184 lines, 13 helpers & ---                     \\
      \midrule
      Cost: uncached input  & 1.06                 & 1.02                 & 2.50                  & 1.13                    \\
      Cost: output          & 0.77                 & 0.98                 & 2.29                  & 4.33                    \\
      Cost: cache reads     & 0.53                 & \textbf{6.67}        & \textbf{5.94}         & 0.02                    \\
      \bottomrule
    \end{tabular}  }
\end{table}

\Cref{tab:case-study-headline} summarizes the four runs.
Two observations frame the rest of the section.
First, the two successful baselines reach essentially the same mathematics: Codex's helper lemma set mirrors our lemma graph so the $3.7\times$--$4.5\times$ cost gap is organizational.
Second, the failure is organizational too: Ax-Prover-Base correctly identified every missing fact yet had no mechanism to turn a blocker into a separately provable goal.

\subsubsection{How \ToolName{} proved it}
\label{app:case-study-ours}

The run took 8 controller iterations and 18 routed agent invocations: 7 planning invocations (router reasoning and the Lemma Planner, at a high-thinking configuration $\theta$) and 11 Proof Writer invocations (at a low-thinking $\theta$).
It had zero backtracks, zero retracted lemmas, and zero rejected proof scripts: every lemma was proposed once and proved once, every Proof Writer invocation closed its target, so no target needed a second invocation, and nothing that was proved was ever redone.

\begin{table}[h]
  \centering
  \caption{\textbf{Iteration timeline} (declaration numbers from \Cref{tab:case-study-decls}).
    Each iteration pushes the open gap exactly one layer down: from the specification, to the string API, to character lists, to the byte scanner.
    Deterministic Automation opened iteration~1 by producing the five-tactic skeleton of \texttt{correctness} at zero model cost; on the eight lower-level goals it failed fast (heartbeat timeout, no progress, or a partial script below the quality threshold) and the controller moved on without spending model budget there.}
  \label{tab:case-study-timeline}
  \small
  \resizebox{\linewidth}{!}{    \begin{tabular}{clllr}
      \toprule
      Iter. & Open goal(s) at start            & New helper lemmas proposed                                                          & Closed by Proof Writer                         & Cost (\$) \\
      \midrule
      1     & \texttt{correctness}             & \texttt{implementation\_splitOn}                                                    & both conjuncts of the spec                     & 0.263     \\
      2     & \texttt{implementation\_splitOn} & round-trip (\#4) + no-space (\#3)                                                   & \texttt{implementation\_splitOn}               & 0.262     \\
      3     & \#3, \#4                         & \texttt{toList}/\texttt{splitOn} bridges (\#5, \#6)                                 & \#4 (\texttt{calc} chain); \#3                 & 0.453     \\
      4     & \#5, \#6                         & private worker (\#7) + predicate split (\#8)                                        & \#5, \#6                                       & 0.266     \\
      5     & \#7, \#8                         & byte-scanner agreement (\#9)                                                        & \#7 (induction generalizing \texttt{acc}); \#8 & 0.263     \\
      6     & \#9                              & UTF-8 offset lemma (\#10)                                                           & \#9 (25 calls, \$0.325: hardest step)          & 0.612     \\
      7     & \#10                             & --- (plan via \texttt{String.next\_of\_valid})                                      & \#10                                           & 0.128     \\
      8     & ---                              & \multicolumn{2}{l}{\textbf{proof complete} at 2{,}815\,s (run ended at 2{,}816\,s)} &                                                            \\
      \bottomrule
    \end{tabular}  }
\end{table}

\Cref{tab:case-study-timeline} gives the timeline, and \Cref{fig:lemmas-case-1} the lemma-dependency graph it builds: each iteration extends a nine-lemma chain, pushing the open gap one layer further down, from the specification to the string API, then to character lists, and finally to the byte-level scanner and its UTF-8 offset arithmetic.

\begin{figure}[h]
  \centering
  \includegraphics[width=\linewidth]{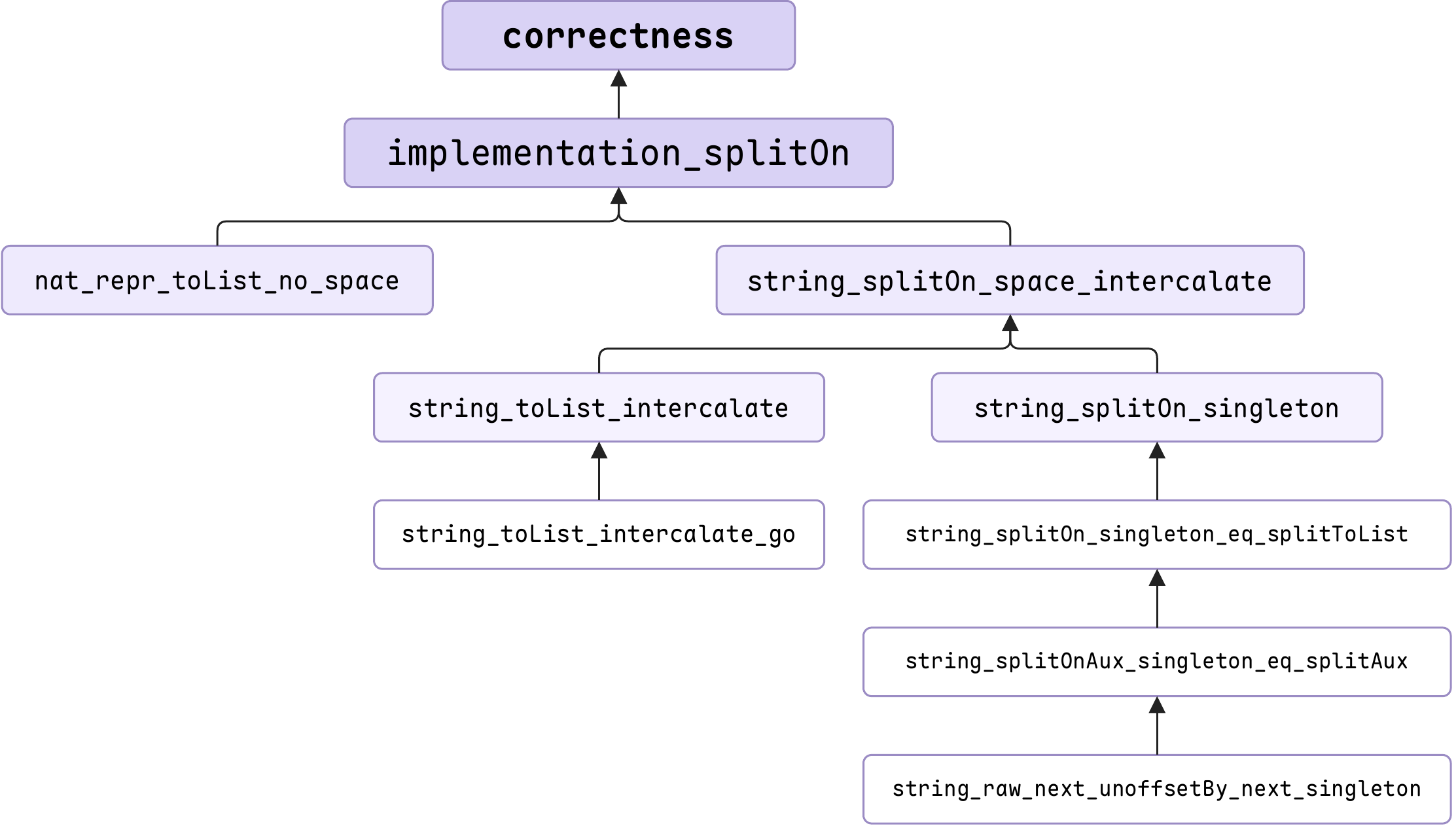}
  \caption{\textbf{Lemma-dependency graph of the final proof of CLEVER \texttt{prob\_15}.}
    The root is the specification-level \texttt{correctness} theorem; every edge points from a helper lemma to the declaration whose proof consumes it.
    Depth corresponds to the layer of abstraction each iteration reached: \texttt{implementation\_splitOn} instantiates the specification to the implementation, the round-trip and no-space lemmas work at the string API, the \texttt{toList} and \texttt{splitOn} bridges move to character lists (including the private accumulator worker), and the lowest branch descends through the byte-level scanner agreement to UTF-8 offset arithmetic.
    Each declaration was proposed with a \texttt{sorry} body, checked for counterexamples, and proved exactly once (\Cref{tab:case-study-decls}).}
  \label{fig:lemmas-case-1}
\end{figure}

\begin{table}[h]
  \centering
  \caption{\textbf{Every declaration in the final proof.}
    Each lemma was proposed with a \texttt{sorry} body and a structured header (\texttt{@description}, \texttt{@lib\_depends\_on}, \texttt{@used\_in}, and \texttt{@depends\_on} once it has generated dependencies), tested for counterexamples before acceptance, and proved exactly once, in the iteration after it was proposed.}
  \label{tab:case-study-decls}
  \scriptsize
  \setlength{\tabcolsep}{4pt}
  \resizebox{\linewidth}{!}{    \begin{tabular}{clllr}
      \toprule
      \# & Lemma/Theorem                                           & Proposed       & Counterexample check        & Lines \\
      \midrule
      1  & \texttt{correctness}                                    & target theorem & ---                         & 14    \\
      2  & \texttt{implementation\_splitOn}                        & iter.~1        & no counterexample           & 9     \\
      3  & \texttt{nat\_repr\_toList\_no\_space}                   & iter.~2        & no counterexample           & 9     \\
      4  & \texttt{string\_splitOn\_space\_intercalate}            & iter.~2        & inconclusive (timeout)      & 28    \\
      5  & \texttt{string\_toList\_intercalate}                    & iter.~3        & no counterexample           & 9     \\
      6  & \texttt{string\_splitOn\_singleton}                     & iter.~3        & no counterexample           & 5     \\
      7  & \texttt{string\_toList\_intercalate\_go}                & iter.~4        & no counterexample           & 11    \\
      8  & \texttt{string\_splitOn\_singleton\_eq\_splitToList}    & iter.~4        & no counterexample           & 5     \\
      9  & \texttt{string\_splitOnAux\_singleton\_eq\_splitAux}    & iter.~5        & inconclusive (not testable) & 55    \\
      10 & \texttt{string\_raw\_next\_unoffsetBy\_next\_singleton} & iter.~6        & inconclusive (not testable) & 15    \\
      \bottomrule
    \end{tabular}  }
\end{table}

\paragraph{Lemma-by-lemma Lean code.}

We list every declaration in discovery order, from the top of the DAG down: its role, the planner's plan (quoted from the transcript), and the final Lean code.

\paragraph{1. \texttt{correctness}: the target theorem.}
Deterministic Automation applied a partial script at zero model cost (static analysis check in rounter), leaving the two conjuncts open; the first Lemma Planner invocation (11 calls, \$0.229) proposed \texttt{implementation\_splitOn} and Proof Writer closed both conjuncts using exactly the planned tactics.

\begin{lstlisting}[style=lean]
theorem correctness
(n: Nat)
: problem_spec implementation n := by
  unfold problem_spec          -- Deterministic Automation skeleton
  unfold implementation        --   (no LLM tokens)
  intro spec
  simp_all [spec]
  apply And.intro
  · rw [implementation_splitOn]                      -- plan, conjunct 1
    simp only [List.length_map, List.length_range]
  · intro i a                                        -- plan, conjunct 2
    rw [implementation_splitOn]
    rw [List.getElem?_map, List.getElem?_range a]
    rfl
\end{lstlisting}

\paragraph{2. \texttt{implementation\_splitOn}: specialize the round-trip to the implementation.}
Every lemma is proposed with a structured header; this one reads:

\begin{lstlisting}[style=lean]
/-
@name: implementation_splitOn
@description: Splitting the space-intercalated decimal representations of the
  numbers in `List.range (n + 1)` recovers exactly that mapped range.
@lib_depends_on: [String.splitOn, String.intercalate, Nat.repr, List.range]
@used_in: [correctness]
-/
\end{lstlisting}

The iteration-2 Lemma Planner proposed the two lemmas the plan relies on (\#3 and \#4).
The Proof Writer closed it in 3 calls:

\begin{lstlisting}[style=lean]
lemma implementation_splitOn (n : Nat) :
    (" ".intercalate ((List.range (n + 1)).map (fun i => i.repr))).splitOn =
      (List.range (n + 1)).map (fun i => i.repr) := by
  apply string_splitOn_space_intercalate
  · intro s hs
    rw [List.mem_map] at hs
    rcases hs with ⟨i, hi, rfl⟩
    exact nat_repr_toList_no_space i
  · simp
\end{lstlisting}

\paragraph{3. \texttt{nat\_repr\_toList\_no\_space}: decimal digits contain no space.}
This is Ax-Prover-Base's \texttt{hnoSpace} blocker (\Cref{app:case-study-baselines}).
The plan argues by contradiction through the \emph{library} fact \texttt{Nat.isDigit\_of\_mem\_toDigits} instead of unfolding \texttt{Nat.toDigits}---unfolding is exactly where Ax-Prover-Base got stuck.
The Proof Writer wrote \texttt{(10 : Nat)} explicitly and added a final \texttt{contradiction}:

\begin{lstlisting}[style=lean]
lemma nat_repr_toList_no_space (n : Nat) : (' ' : Char) ∉ n.repr.toList := by
  intro h
  have hm : (' ' : Char) ∈ (10 : Nat).toDigits n := by
    simpa [Nat.repr] using h
  have hd : (' ' : Char).isDigit = true :=
    Nat.isDigit_of_mem_toDigits (by norm_num) (by norm_num) hm
  norm_num [Char.isDigit] at hd
  contradiction
\end{lstlisting}

\paragraph{4. \texttt{string\_splitOn\_space\_intercalate}: the string round-trip.}
The central string-level fact; Humanize rounds 1--2 reported it as missing from the library.
The Lemma Planner first proposed the two character-list bridge lemmas (\#5 and \#6), then the Proof Writer used the four-step \texttt{calc} structure and fixed three details against Lean: \texttt{rintro} became \texttt{rcases List.mem\_map.mp}, \texttt{String.singleton\_eq} became \texttt{String.singleton\_eq\_ofList}, and a final \texttt{simp [List.map\_map]} that did not close the goal was replaced by an explicit induction using \texttt{String.ofList\_toList}:

\begin{lstlisting}[style=lean]
lemma string_splitOn_space_intercalate (ss : List String)
    (hspace : ∀ s ∈ ss, (' ' : Char) ∉ s.toList) (hne : ss ≠ []) :
    (" ".intercalate ss).splitOn = ss := by
  have hspace' : ∀ l ∈ ss.map String.toList, (' ' : Char) ∉ l := by
    intro l hl
    rcases List.mem_map.mp hl with ⟨s, hs, rfl⟩
    exact hspace s hs
  have hne' : ss.map String.toList ≠ [] := by
    simpa using hne
  calc
    -- (a) String.splitOn → List.splitOn on characters      [lemma #6]
    (" ".intercalate ss).splitOn =
        (List.splitOn (' ' : Char) ((" ".intercalate ss).toList)).map
          String.ofList := by
      simpa [String.singleton_eq_ofList] using
        (string_splitOn_singleton (" ".intercalate ss) (' ' : Char))
    -- (b) String.intercalate → List.intercalate on characters  [lemma #5]
    _ = (List.splitOn (' ' : Char)
        (([' '] : List Char).intercalate (ss.map String.toList))).map
          String.ofList := by
      rw [string_toList_intercalate]
      rfl
    -- (c) the list-level round-trip from the library
    _ = (ss.map String.toList).map String.ofList := by
      rw [List.splitOn_intercalate (ss.map String.toList) (' ' : Char)
        hspace' hne']
    -- (d) back to strings
    _ = ss := by
      rw [List.map_map]
      clear hspace' hne' hspace hne
      induction ss with
      | nil => rfl
      | cons s ss ih =>
        simp only [List.map_cons, Function.comp_apply]
        rw [String.ofList_toList, ih]
\end{lstlisting}

\paragraph{5. \texttt{string\_toList\_intercalate}: \texttt{String.intercalate} as a list operation.}
Ax-Prover-Base's \texttt{hintercalate} blocker.
The Lemma Planner first proposed \#7, a lemma about the private worker, then split on the list of strings; the Proof Writer completed:

\begin{lstlisting}[style=lean]
open private String.intercalate.go from Init.Data.String.Defs in
lemma string_toList_intercalate (sep : String) (ss : List String) :
    (sep.intercalate ss).toList = sep.toList.intercalate (ss.map String.toList) := by
  cases ss with
  | nil =>
    rfl
  | cons a as =>
    simpa only [String.intercalate, List.map_cons] using
      (string_toList_intercalate_go a sep as)
\end{lstlisting}

\paragraph{6. \texttt{string\_splitOn\_singleton}: \texttt{String.splitOn} as a list operation.}
Part of Ax-Prover-Base's \texttt{hbridge} blocker.
The Lemma Planner separated the byte-scanner part into \#8 and connected the rest to the library theorem \texttt{String.splitToList\_of\_valid}. The Proof Writer completed:

\begin{lstlisting}[style=lean]
lemma string_splitOn_singleton (s : String) (c : Char) :
    s.splitOn (String.singleton c) = (s.toList.splitOn c).map String.ofList := by
  rw [string_splitOn_singleton_eq_splitToList]
  rw [String.splitToList_of_valid]
  rfl
\end{lstlisting}

\paragraph{7. \texttt{string\_toList\_intercalate\_go}: a lemma about a private function.}
The Proof Writer strengthened a plan ``induct on \texttt{ss}, generalizing \texttt{acc}, since \texttt{String.intercalate.go} advances precisely by replacing the accumulator with \texttt{acc ++ sep ++ head}.'' with a \texttt{cases} split and two extra \texttt{simp} lemmas:

\begin{lstlisting}[style=lean]
open private String.intercalate.go from Init.Data.String.Defs in
lemma string_toList_intercalate_go (acc sep : String) (ss : List String) :
    (String.intercalate.go acc sep ss).toList =
      sep.toList.intercalate (acc.toList :: ss.map String.toList) := by
  induction ss generalizing acc with
  | nil =>
    simp [String.intercalate.go, List.intercalate]
  | cons s ss ih =>
    rw [String.intercalate.go]
    rw [ih (acc := acc ++ sep ++ s)]
    cases ss <;> simp [String.toList_append, List.intercalate,
      List.intersperse, List.append_assoc]
\end{lstlisting}

\paragraph{8. \texttt{string\_splitOn\_singleton\_eq\_splitToList}: pattern split $=$ predicate split.}
The Lemma Planner pushed the scanner itself into lemma \#9; the Proof Writer used the lemma as written:

\begin{lstlisting}[style=lean]
lemma string_splitOn_singleton_eq_splitToList (s : String) (c : Char) :
    s.splitOn (String.singleton c) = s.splitToList (fun d => d == c) := by
  unfold String.splitOn String.splitToList
  rw [if_neg (by simpa only [beq_iff_eq] using String.singleton_ne_empty)]
  exact string_splitOnAux_singleton_eq_splitAux s c 0 0 []
\end{lstlisting}

\paragraph{9. \texttt{string\_splitOnAux\_singleton\_eq\_splitAux}: the byte-level scanners agree.}
The hardest step.
Ax-Prover-Base's attempts to unfold this recursion hit ``recursion-depth failures''; Humanize's round 1 called the scanner ``irreducible.''
The Lemma Planner invocation separated the offset arithmetic into lemma \#10. Then the Proof Writer planned the four cases of a well-founded \texttt{String.splitOnAux.induct}, ``retaining the invariant only for the reachable singleton-separator state $j = 0$ \ldots\ because the recursor also exposes unreachable states having a nonzero separator position.''
The Proof Writer then turned the prose plan into 55 lines, discovering the exact \texttt{if}/\texttt{dif} rewrite sequences case by case:

\begin{lstlisting}[style=lean]
lemma string_splitOnAux_singleton_eq_splitAux
    (s : String) (c : Char) (b i : String.Pos.Raw) (r : List String) :
    s.splitOnAux (String.singleton c) b i 0 r =
      s.splitAux (fun d => d == c) b i r := by
  refine String.splitOnAux.induct s (String.singleton c)
    (motive := fun b i j r => j = 0 →
      s.splitOnAux (String.singleton c) b i j r =
        s.splitAux (fun d => d == c) b i r)
    ?_ ?_ ?_ ?_ b i 0 r rfl
  -- case 1: at the end of `s`
  · intro b i j r hi hj
    simp [String.splitOnAux.eq_def, String.splitAux.eq_def, hi]
  -- case 2: character matches; the (one-character) separator is complete
  · intro b i j r hi hmatch i1 j1 hend ih hj
    subst j
    have hsingle : String.Pos.Raw.get (String.singleton c) 0 = c := by
      rw [← String.front_eq_get, String.front_eq]
      simp [String.toList_singleton]
    have hnext : String.Pos.Raw.next (String.singleton c) 0 =
        { byteIdx := c.utf8Size } := by
      simpa only [String.singleton_eq_ofList, List.nil_append,
        List.cons_append, String.utf8Len_nil, Nat.zero_add] using
        String.next_of_valid [] c []
    have hsingle_end : String.Pos.Raw.atEnd (String.singleton c)
        (String.Pos.Raw.next (String.singleton c) 0) = true := by
      rw [hnext]
      simp [String.Pos.Raw.atEnd, String.utf8ByteSize_singleton]
    have hget : String.Pos.Raw.get s i = c := by
      simpa [hsingle] using hmatch
    rw [string_raw_next_unoffsetBy_next_singleton s i c hget] at ih
    rw [String.splitOnAux.eq_def]
    rw [String.splitAux.eq_def]
    rw [if_neg hi, if_pos hmatch, if_pos hsingle_end]
    simp only [hget]
    rw [string_raw_next_unoffsetBy_next_singleton s i c hget]
    rw [dif_neg hi, if_pos (by simp)]
    exact ih rfl
  -- case 3: match, but separator incomplete (impossible for a singleton)
  · intro b i j r hi hmatch i1 j1 hend ih hj
    subst j
    have hnext : String.Pos.Raw.next (String.singleton c) 0 =
        { byteIdx := c.utf8Size } := by
      simpa only [String.singleton_eq_ofList, List.nil_append,
        List.cons_append, String.utf8Len_nil, Nat.zero_add] using
        String.next_of_valid [] c []
    have hsingle_end : String.Pos.Raw.atEnd (String.singleton c)
        (String.Pos.Raw.next (String.singleton c) 0) = true := by
      rw [hnext]
      simp [String.Pos.Raw.atEnd, String.utf8ByteSize_singleton]
    exact False.elim (hend hsingle_end)
  -- case 4: no match; separator position resets to 0
  · intro b i j r hi hmatch ih hj
    subst j
    have hsingle : String.Pos.Raw.get (String.singleton c) 0 = c := by
      rw [← String.front_eq_get, String.front_eq]
      simp [String.toList_singleton]
    have hnomatch : ¬(String.Pos.Raw.get s i == c) = true := by
      simpa [hsingle] using hmatch
    have hunoffset : i.unoffsetBy 0 = i := by
      apply String.Pos.Raw.ext
      simp [String.Pos.Raw.unoffsetBy]
    rw [String.splitOnAux.eq_def]
    rw [String.splitAux.eq_def]
    rw [if_neg hi, if_neg hmatch]
    simp only [hsingle, hunoffset]
    rw [dif_neg hi, if_neg hnomatch]
    exact ih rfl
\end{lstlisting}

\paragraph{10. \texttt{string\_raw\_next\_unoffsetBy\_next\_singleton}: UTF-8 offset arithmetic.}
The Proof Writer first reasoned a plan which reduces position equality to byte-index equality, computes the singleton advance with the library theorem \texttt{String.next\_of\_valid}, and cancels the subtraction with \texttt{Nat.add\_sub\_cancel}.
The Proof Writer then added \texttt{String.Pos.Raw.mk\_zero} and \texttt{Nat.zero\_add} to the planned \texttt{simpa only} list:

\begin{lstlisting}[style=lean]
lemma string_raw_next_unoffsetBy_next_singleton
    (s : String) (i : String.Pos.Raw) (c : Char)
    (hget : String.Pos.Raw.get s i = c) :
    (String.Pos.Raw.next s i).unoffsetBy
        (String.Pos.Raw.next (String.singleton c) 0) = i := by
  have hsingleton : String.Pos.Raw.next (String.singleton c) 0 =
      { byteIdx := c.utf8Size } := by
        simpa only [String.singleton_eq_ofList, List.nil_append,
          String.utf8Len_nil, String.Pos.Raw.mk_zero, Nat.zero_add] using
          (String.next_of_valid [] c [])
  apply String.Pos.Raw.ext
  rw [String.Pos.Raw.byteIdx_unoffsetBy, hsingleton]
  simp only [String.Pos.Raw.next, hget, String.Pos.Raw.addChar_eq]
  exact Nat.add_sub_cancel _ _
\end{lstlisting}

\paragraph{What the lemma-by-lemma view shows.}
The Lemma Planner decomposed subgoals. Every blocker became a lemma one layer down including \#7 (the private worker), \#8/\#9 (the byte scanner), \#10 (offset arithmetic). The Proof Writer often reasoned about the mathematics then fixes the Lean details.
And library facts were found and used when they existed (\texttt{Nat.isDigit\_of\_mem\_toDigits}, \texttt{List.splitOn\_intercalate}, \texttt{String.splitToList\_of\_valid}, \texttt{String.splitOnAux.induct}, \texttt{String.next\_of\_valid}). New lemmas were written only for facts missing from the library.

\subsubsection{How the baselines spent their budgets}
\label{app:case-study-baselines}

\paragraph{Codex: \$8.67, proved.}

Codex worked in one continuous turn of about 193 tool steps---101 shell commands (92 of them Lean checks) and 92 scratch-file edits---plus 5 progress messages.
Its messages show it reached the same mathematical plan as ours: ``the remaining work is the string round-trip fact for \texttt{intercalate}/\texttt{splitOn} \ldots\ reducing \texttt{splitOn " "} to the library's character-splitting semantics. The proof is now at the low-level auxiliary recursion, with termination decreasing by UTF-8 position.''
Its final proof has a matching shape (excerpt):

\begin{lstlisting}[style=lean]
theorem splitOn_space_aux (s : String) (b i : String.Pos.Raw) (r : List String) :
    s.splitOnAux " " b i 0 r = s.splitAux (fun c => c == ' ') b i r := by
  ...
termination_by s.utf8ByteSize - i.byteIdx
decreasing_by
  all_goals
    apply dec_next s i
    simp_all [String.Pos.Raw.atEnd]

theorem splitOn_space (s : String) :
    s.splitOn " " = s.splitToList (fun c => c == ' ') := ...

theorem toList_intercalate (sep : String) (xs : List String) :
    (sep.intercalate xs).toList = sep.toList.intercalate (xs.map String.toList) := ...

theorem space_not_mem_repr (n : Nat) : ' ' ∉ n.repr.toList := ...

theorem split_intercalate_repr (n : Nat) :
    (" ".intercalate ((List.range (n + 1)).map (fun i => i.repr))).splitOn " " =
      (List.range (n + 1)).map (fun i => i.repr) := ...
\end{lstlisting}

For the private \texttt{String.intercalate.go}, Codex wrote a custom meta-level tactic that rebuilds the private declaration's internal mangled name by hand:

\begin{lstlisting}[style=lean]
open Lean Elab Tactic Meta
elab "unfoldStringIntercalateGo" : tactic => liftMetaTactic fun goal => do
  let name := Name.anonymous.str "_private" |>.str "Init" |>.str "Data"
    |>.str "String" |>.str "Defs" |>.num 0 |>.str "String"
    |>.str "intercalate" |>.str "go"
  let goal ← unfoldTarget goal name
  return [goal]
\end{lstlisting}

This works but depends on the private name's internal encoding.
Where the money went: every step re-sent the whole transcript, including all earlier Lean outputs and edits---27.08M prompt tokens, of which 26.68M were cache reads (\$6.67, 77\% of its cost), roughly 140k prompt tokens per step against our 15.8k per call (2.54M over 161 calls).
The mathematics was the same and the difference is context management, which demonstrates the advantage of the metalevel decision-making on agent orchestration.

\paragraph{Humanize: \$10.74, proved on round 9 of 9.}

Humanize runs a fresh worker each round, with gate checks and a final reviewer.
Rounds 1--8 all ended without a proof, costing \$9.05 (84\% of the total); each worker restarted the exploration and hit the same wall (\Cref{tab:case-study-humanize-rounds}).

\begin{table}[h]
  \centering
  \caption{\textbf{Humanize's nine rounds.} The eight failed rounds name the same two obstacles that our lemmas \#7--\#10 dissolve.}
  \label{tab:case-study-humanize-rounds}
  \small
  \resizebox{\linewidth}{!}{    \begin{tabular}{clp{9.2cm}}
      \toprule
      Round & Cost                       & Worker's stated blocker (from the run log)                                                                                                                                                  \\
      \midrule
      1     & \$0.83                     & ``\texttt{String.splitOn}'s irreducible positional scanner; the needed semantic bridge is absent''                                                                                          \\
      2     & \$0.98                     & ``\texttt{String.splitOn}'s irreducible UTF-8 scanner; no applicable semantic split/intercalate lemma''                                                                                     \\
      3     & \$2.77                     & established the \texttt{splitOn " "} $\leftrightarrow$ character-splitting bridge, but ``the remaining obstacle is \ldots\ the opaque/internal \texttt{String.intercalate} implementation'' \\
      4     & \$0.99                     & ``the recursive join helper is inaccessible for induction''                                                                                                                                 \\
      5     & \$0.79                     & ``the required semantic bridge \ldots\ is not exported; \texttt{List.splitOn\_intercalate} cannot be applied''                                                                              \\
      6     & \$0.99                     & ``the unresolved piece is a public recursive lemma for this Lean version's private \texttt{String.intercalate} helper''                                                                     \\
      7     & \$0.88                     & ``its recursive worker is private/opaque''                                                                                                                                                  \\
      8     & \$0.83                     & ``an unexported/private recursive helper behind \texttt{String.intercalate}''                                                                                                               \\
      9     & \$1.54 $+$ \$0.15 reviewer & proof completed; gates pass                                                                                                                                                                 \\
      \bottomrule
    \end{tabular}  }
\end{table}

Two structural problems stand out.
First, progress is not kept: round~3 proved the string-to-character splitting bridge but left the candidate file unchanged (``The candidate remains unchanged''), so rounds 4--8 rediscovered it.
Second, a blocker is never turned into a goal: four of the five rounds 4--8 named the private \texttt{String.intercalate.go} as the obstacle, where \ToolName{} stated it as a checked lemma \emph{about} the private function and proved it for \$0.037.
(A second, earlier Humanize run also proved the task, for \$6.01 over 3 worker rounds---still $2.5\times$ our cost.)

\paragraph{Ax-Prover-Base: \$5.48, failed.}

Ax-Prover-Base repeatedly proposes a complete proof of \texttt{correctness} and compiles it: over 50 iterations and 151 LLM calls, 49 of 50 builds failed to compile, and the run ended by exhausting its iteration limit.
Its final self-assessment lists three blocking inline \texttt{have}s that it kept trying to prove inside the body of \texttt{correctness}.
The sketch below shows that the three \texttt{have} statements and the quoted failure reasons are verbatim, and the surrounding lines follow its described strategy:

\begin{lstlisting}[style=lean]
theorem correctness (n: Nat) : problem_spec implementation n := by
  unfold problem_spec
  refine ⟨implementation n, rfl, ?_⟩
  -- stable parts: htoList_inj, hmap_inj, the List.splitOn_intercalate
  -- step, both conjuncts
  have hbridge : (s.splitOn " ").map String.toList = s.toList.splitOn ' ' := by
    ...  -- "recursive unfolding caused recursion-depth failures"
  have hintercalate : (" ".intercalate xs).toList =
      [' '].intercalate (xs.map String.toList) := by
    ...  -- "simplification did not reduce its base and inductive cases"
  have hnoSpace : ' ' ∉ i.repr.toList := by
    ...  -- "got stuck on the conditional implementation of `Nat.toDigits`"
  ...
\end{lstlisting}

Each blocker corresponds to a lemma or chain \ToolName{} proved separately: \texttt{hbridge} to lemmas \#6, \#8, \#9, and \#10; \texttt{hintercalate} to \#5 and \#7; and \texttt{hnoSpace} to \#3.
A monolithic attempt fails if \emph{any} part fails, so the parts that did work (\texttt{htoList\_inj}, \texttt{hmap\_inj}, and the use of \texttt{List.splitOn\_intercalate}) were never kept, and the model spent \$4.33 on output, including 228k reasoning tokens, re-deriving the whole proof each iteration.

\subsubsection{Why \ToolName{} is cheaper}
\label{app:case-study-analysis}

Six mechanisms, each visible in the trace, account for the gap; each is a design component of \Cref{sec:architecture-control-loop} doing exactly the job it was built for.
\emph{(1) Small, focused contexts---two packets, neither of them the whole search.}
The router's packet is frontier-wide but shallow: tactic states, the graph neighborhood, and failure memory restricted to the open goals, with no proof transcripts. The specialist's is narrow and deep: the group the router chose, its preliminary analysis, and the excerpt of the router's conversation the invocation carries. Neither packet grows with the search, so each call carries 15.8k prompt tokens, against $\sim$140k per Codex step and 1.4M--7.6M cache reads per Humanize round. This alone explains the $10\times$ gap in prompt tokens at similar proof sizes.
\emph{(2) Verified progress is never repeated---the persistent two-level proof graph.}
Kernel-checked promotion commits every proved lemma to the dependency graph, and the Kahn frontier hands later invocations only what remains open, so all 18 invocations contributed to the final proof. However, Humanize, restarting a fresh worker per round, spent 84\% of its budget re-exploring, and Ax-Prover-Base, keeping only a memory of recent failures rather than persistent proof state, discarded 49 of 50 builds.
\emph{(3) Blockers become subgoals---the Lemma Planner.}
Every baseline correctly \emph{named} the missing facts but had no agent whose job is to act on them. The Lemma Planner states each blocker as a checked lemma and admits it to the dependency graph only after elaboration and acyclicity, whereupon its proof obligations enter the frontier as ordinary goals small enough for the Proof Writer to close.
\emph{(4) Expensive reasoning where it matters---the router's configuration choice $\theta$.}
Because every action is a separately priced purchase $(G,o,\theta)$, the router bought lemma planning at a high-thinking and script synthesis (proof writer) at a low-thinking $\theta$.
\emph{(5) Free goals stay free---Deterministic Automation under the same-provability cheap-first rule.}
It produced the opening skeleton of \texttt{correctness} at zero model cost and failed fast everywhere else, so no LLM budget was ever spent on a goal that symbolic procedures could open or close.
\emph{(6) Conjectures are tested before they are trusted---the counterexample check under the tests-do-not-certify invariant.}
All 9 proposed lemmas were tested for \$0.12 before anything was built on them. Tests may reject a candidate but never promote one, so refutable statements are filtered before proof effort, while soundness stays with the kernel.

\subsection{Vero \texttt{prob\_14}: interval-set algebra, 56 specifications}
\label{app:case-study-vero}

Vero \texttt{prob\_14} asks for proofs of 56 specification theorems about a small interval-set library.
\ToolName{}, running on GPT-5.6-Terra, proved all 56, and every proof is checked by Lean. It built 55 helper lemmas along the way, organized into a dependency hierarchy up to six levels deep.
The coding-agent baseline on the same backbone LLM proved 5 of the 56.
This subsection explains how full provability was reached: how the targets were split up, which lemmas carry the load, and how the proofs of the hardest facts were constructed.

\subsubsection{The task}
\label{app:vero-task}

An interval \texttt{Iv} is a half-open range $[\mathit{lo},\mathit{hi})$ of integers; a \texttt{List Iv} denotes the union of its intervals.
The library under verification has three operations:

\begin{lstlisting}[style=lean]
structure Iv where
  lo : Int
  hi : Int
deriving DecidableEq, Repr, Inhabited

-- frozen point-set semantics that every spec is written against
def covered (x : Int) (ivs : List Iv) : Bool :=
  ivs.any (fun iv => decide (iv.lo ≤ x ∧ x < iv.hi))

def disjointSortedNonempty (ivs : List Iv) : Prop :=
  (∀ iv ∈ ivs, iv.lo < iv.hi) ∧
  (∀ (i : Nat) (h : i + 1 < ivs.length), (ivs[i]).hi < (ivs[i+1]).lo)

def Intervaltree.overlaps : Intervaltree.OverlapsSig :=
  fun a b => decide (a.lo < b.hi ∧ b.lo < a.hi)

-- mergeOverlaps = insertion sort by `lo` (dropping empty intervals),
-- then coalescing
def Intervaltree.insertByLo (iv : Iv) : List Iv → List Iv
  | [] => if iv.lo < iv.hi then [iv] else []
  | y :: ys =>
    if iv.lo < iv.hi then
      if iv.lo ≤ y.lo then iv :: y :: ys
      else y :: Intervaltree.insertByLo iv ys
    else y :: ys

def Intervaltree.sortByLo (ivs : List Iv) : List Iv :=
  ivs.foldr Intervaltree.insertByLo []

def Intervaltree.coalesce : List Iv → List Iv
  | [] => []
  | [x] => [x]
  | x :: y :: rest =>
    if y.lo ≤ x.hi then
      Intervaltree.coalesce
        ({ lo := x.lo, hi := if x.hi < y.hi then y.hi else x.hi } :: rest)
    else
      x :: Intervaltree.coalesce (y :: rest)
  termination_by l => l.length

def Intervaltree.mergeOverlaps : Intervaltree.MergeOverlapsSig :=
  fun ivs => Intervaltree.coalesce (Intervaltree.sortByLo ivs)

-- chop b e = remove [b, e) from every interval, splitting straddlers
def Intervaltree.chopOne (b e : Int) (iv : Iv) : List Iv :=
  let left  : List Iv :=
    if iv.lo < b then [{ lo := iv.lo, hi := if iv.hi < b then iv.hi else b }] else []
  let right : List Iv :=
    if e < iv.hi then [{ lo := if iv.lo < e then e else iv.lo, hi := iv.hi }] else []
  (left ++ right).filter (fun p => decide (p.lo < p.hi))

def Intervaltree.chop : Intervaltree.ChopSig :=
  fun b e ivs => ivs.flatMap (Intervaltree.chopOne b e)
\end{lstlisting}

The 56 targets fall into three groups: \emph{overlaps} (6 targets: symmetry, self-overlap, a characterization by a common point, and one target relating \texttt{overlaps} to runs of merged intervals); \emph{mergeOverlaps} (39 targets: coverage, canonical form, uniqueness, invariance under permutation and reversal, idempotence, measure bounds, hull endpoints, and point-count properties); and \emph{chop} (11 targets: pointwise removal, idempotence, preservation of canonical form, and interaction with merging).
Each target has the form \texttt{theorem prove\_X : spec\_X canonical}; some representative statements:

\begin{lstlisting}[style=lean]
def spec_merge_coverage (impl : RepoImpl) : Prop :=
  ∀ (ivs : List Iv) (x : Int),
    covered x (impl.intervaltree.mergeOverlaps ivs) = covered x ivs

def spec_merge_disjoint_strict (impl : RepoImpl) : Prop :=
  ∀ (ivs : List Iv),
    let out := impl.intervaltree.mergeOverlaps ivs
    (∀ iv ∈ out, iv.lo < iv.hi) ∧
    (∀ (i : Nat) (h : i + 1 < out.length), (out[i]).hi < (out[i+1]).lo)

def spec_merge_canonical_unique (impl : RepoImpl) : Prop :=
  ∀ (ivs c : List Iv),
    disjointSortedNonempty c →
    (∀ x, covered x c = covered x ivs) →
      impl.intervaltree.mergeOverlaps ivs = c

def spec_merge_perm_eq_any (impl : RepoImpl) : Prop :=
  ∀ (xs ys : List Iv),
    List.Perm xs ys →
      impl.intervaltree.mergeOverlaps xs = impl.intervaltree.mergeOverlaps ys

def spec_chop_canonical_merge_fixpoint (impl : RepoImpl) : Prop :=
  ∀ (b e : Int) (ivs : List Iv),
    b < e →
      disjointSortedNonempty ivs →
        impl.intervaltree.mergeOverlaps (impl.intervaltree.chop b e ivs)
          = impl.intervaltree.chop b e ivs
\end{lstlisting}

Four features make the problem hard.
\emph{Three different recursion schemes}: \texttt{insertByLo} is structural recursion, \texttt{sortByLo} is a fold, and \texttt{coalesce} is well-founded recursion whose merge branch recurses on a list whose head has been changed. Unfolding them by hand gives nested \texttt{match}/\texttt{if} terms with no usable induction hypothesis.
\emph{Semantic specifications}: the specs are phrased through \texttt{covered} (a point set) and \texttt{disjointSortedNonempty} (index-based gap conditions), not through the implementation's structure, so every proof must connect the recursive code to these frozen predicates.
\emph{Genuinely global targets}: uniqueness and permutation-invariance need an extensionality principle (\eg two canonical lists covering the same points are equal) which no single unfolding of \texttt{mergeOverlaps} gives.
\emph{Scale}: 56 targets share structure, and proving each independently would repeat the same inductions dozens of times.

\subsubsection{Outcome}
\label{app:vero-outcome}

\begin{table}[h]
  \centering
  \caption{\textbf{Vero \texttt{prob\_14} outcome.} The final 2{,}471-line proof file passes the checker: all 56 targets proved with no \texttt{sorry} and no axioms beyond Lean's standard ones, and all 56 target statements identical to the originals (no restated targets).}
  \label{tab:vero-outcome}
  \small
  \begin{tabular}{lcc}
    \toprule
                             & \textbf{\ToolName{}} (GPT-5.6-Terra) & Coding-agent baseline (GPT-5.6-Terra) \\
    \midrule
    Targets proved           & \textbf{56 / 56}                     & 5 / 56                                \\
    Full file checks in Lean & \textbf{yes}                         & no                                    \\
    \bottomrule
  \end{tabular}
\end{table}

The complete run used about \$19 of model usage, roughly \$0.35 per target theorem.

\subsubsection{How \ToolName{} proved all 56 targets}
\label{app:vero-how}

The proof was built over 34 controller iterations; the number of proved lemmas/theorems grew steadily: 6 at the start of iteration 2, 25 at iteration 5, 60 at iteration 15, 91 at iteration 26, and all 111 (56 targets $+$ 55 lemmas) by iteration 34.
Deterministic Automation ran first on every open goal: it closed 24 goal nodes outright (counted per goal, not per declaration as in \Cref{tab:vero-closure}) and 59 times applied a partial script that a specialist then completed.
The run made 175 routed invocations---72 planning (router reasoning and the Lemma Planner) and 103 Proof Writer---and accepted 56 helper lemmas, each first passing the counterexample check; 55 appear in the final proof, and one ended up unnecessary.
Unlike the zero-backtrack CLEVER run, this run made \textbf{52 backtracks}, each with a written reason accepted by the Recovery Reviewer: the controller routinely rewound a branch where automation had unfolded a recursive definition into an unmanageable term, and replaced it with an induction or a lemma application at the nearest safe node (\Cref{app:vero-episodes}).

\begin{table}[h]
  \centering
  \caption{\textbf{How each declaration was closed.}
    The 8 targets closed by Deterministic Automation alone are local facts, including \texttt{overlaps} symmetry, self-overlap, disjointness, \texttt{merge\_nil}, \texttt{merge\_singleton}, \texttt{chop\_removes}, and \texttt{chop\_append\_coverage}; everything else is built through lemmas, and 7 of the 46 lemma-built targets are themselves reused as lemmas by later targets.}
  \label{tab:vero-closure}
  \small
  \resizebox{\linewidth}{!}{    \begin{tabular}{lccc}
      \toprule
                         & Closed by Automation & Closed by an LLM script (no new lemma) & Closed using lemmas or earlier targets \\
      \midrule
      56 target theorems & 8                    & 2                                      & \textbf{46}                            \\
      55 helper lemmas   & 12                   & 20                                     & 23                                     \\
      \bottomrule
    \end{tabular}  }
\end{table}

For almost every non-trivial target the planner followed one pattern: (1) avoid unfolding the implementation, instead unfold only the spec and the \texttt{canonical} record, down to \texttt{coalesce (sortByLo ivs)}; (2) split the claim along the implementation's pipeline: a fact about \texttt{insertByLo}, lifted to \texttt{sortByLo}, combined with a fact about \texttt{coalesce} that assumes a sortedness certificate (\texttt{Pairwise ($\cdot$.lo $\le$ $\cdot$.lo)}); (3) prove each layer by the induction matching its recursion: structural induction for \texttt{insertByLo} and \texttt{sortByLo}, the functional induction principle \texttt{coalesce.induct} for \texttt{coalesce}; (4) close the target in a few lines by rewriting with the layer lemmas.
The plan for \texttt{prove\_merge\_coverage} states it directly: ``the existing coverage lemmas exactly factor merge into sorting and coalescing; do not unfold either implementation.''

\subsubsection{Key proof episodes, with Lean code}
\label{app:vero-episodes}

\paragraph{The two central targets: coverage and canonical form.}
Almost every \texttt{mergeOverlaps} target depends on \texttt{prove\_merge\_coverage} (used directly by 15 declarations) or \texttt{prove\_merge\_disjoint\_strict} (used directly by 11); both reduce to lemmas about the implementation's layers:

\begin{lstlisting}[style=lean]
theorem prove_merge_coverage : spec_merge_coverage canonical := by
  intro ivs x
  change covered x (Intervaltree.coalesce (Intervaltree.sortByLo ivs)) = covered x ivs
  rw [coalesce_coverage_of_pairwise_lo (Intervaltree.sortByLo ivs) x
    (sortByLo_pairwise_lo ivs)]
  exact sortByLo_coverage ivs x

theorem prove_merge_disjoint_strict : spec_merge_disjoint_strict canonical := by
  intro ivs
  change disjointSortedNonempty
    (Intervaltree.coalesce (Intervaltree.sortByLo ivs))
  exact coalesce_disjointSortedNonempty (Intervaltree.sortByLo ivs)
    (sortByLo_output_nonempty ivs) (sortByLo_pairwise_lo ivs)
\end{lstlisting}

The sorting layer is handled by structural induction; \texttt{sortByLo\_coverage} lifts the one-step fact \texttt{insertByLo\_coverage} through the fold:

\begin{lstlisting}[style=lean]
lemma sortByLo_coverage (ivs : List Iv) (x : Int) :
    covered x (Intervaltree.sortByLo ivs) = covered x ivs := by
  induction ivs with
  | nil =>
    rfl
  | cons iv ivs ih =>
    calc
      covered x (Intervaltree.sortByLo (iv :: ivs)) =
          covered x (Intervaltree.insertByLo iv (Intervaltree.sortByLo ivs)) := rfl
      _ = covered x (iv :: Intervaltree.sortByLo ivs) :=
        insertByLo_coverage iv (Intervaltree.sortByLo ivs) x
      _ = covered x (iv :: ivs) := by
        simpa only [covered, List.any_cons] using
          congrArg (fun z : Bool => decide (iv.lo ≤ x ∧ x < iv.hi) || z) ih
\end{lstlisting}

The coalescing layer is where the difficulty lies.
The first attempt on \texttt{prove\_merge\_coverage} let automation unfold \texttt{coalesce} and case-split; the Recovery Reviewer rewound it, recording: ``the current branches have unfolded \texttt{Intervaltree.coalesce} repeatedly into arbitrarily deep recursive matches, but carry no induction hypothesis or sortedness invariant; closing them would require reconstructing a global induction invariant from the expanded terms.''
The Lemma Planner then proposed \texttt{coalesce\_coverage\_of\_pairwise\_lo}, planned as a functional induction over \texttt{coalesce} in which $x$ and the sortedness certificate are \emph{reverted}, so the induction hypothesis applies to the merged list:

\begin{lstlisting}[style=lean]
lemma coalesce_coverage_of_pairwise_lo (ivs : List Iv) (x : Int)
    (hsorted : ivs.Pairwise (fun a b => a.lo ≤ b.lo)) :
    covered x (Intervaltree.coalesce ivs) = covered x ivs := by
  revert x hsorted
  induction ivs using Intervaltree.coalesce.induct with
  | case1 =>
    intro x _
    rw [Intervaltree.coalesce.eq_1]
  | case2 a =>
    intro x _
    rw [Intervaltree.coalesce.eq_2]
  | case3 a b rest hab ih =>                       -- merge branch: b.lo ≤ a.hi
    intro x hs
    rcases (List.pairwise_cons.mp hs) with ⟨ha, htail⟩
    have hablo : a.lo ≤ b.lo := ha b (by simp)
    by_cases hhi : a.hi < b.hi
    · have hsMerged :
          ({ lo := a.lo, hi := b.hi } :: rest).Pairwise
            (fun u v => u.lo ≤ v.lo) := by
              rw [List.pairwise_cons]
              constructor
              · intro r hr
                exact ha r (by simp [hr])
              · exact List.Pairwise.of_cons htail
      rw [Intervaltree.coalesce.eq_3, if_pos hab]
      simp only [dif_pos hhi, if_pos hhi] at ih ⊢
      rw [ih x hsMerged]
      by_cases hal : a.lo ≤ x <;>
        by_cases hax : x < a.hi <;>
        by_cases hbl : b.lo ≤ x <;>
        by_cases hbx : x < b.hi <;>
        simp [covered, hal, hax, hbl, hbx] <;> omega
    · have hsMerged :
          ({ lo := a.lo, hi := a.hi } :: rest).Pairwise
            (fun u v => u.lo ≤ v.lo) := by
              rw [List.pairwise_cons]
              constructor
              · intro r hr
                exact ha r (by simp [hr])
              · exact List.Pairwise.of_cons htail
      rw [Intervaltree.coalesce.eq_3, if_pos hab]
      simp only [dif_neg hhi, if_neg hhi] at ih ⊢
      rw [ih x hsMerged]
      by_cases hal : a.lo ≤ x <;>
        by_cases hax : x < a.hi <;>
        by_cases hbl : b.lo ≤ x <;>
        by_cases hbx : x < b.hi <;>
        simp [covered, hal, hax, hbl, hbx] <;> omega
  | case4 a b rest hab ih =>                       -- gap branch: a.hi < b.lo
    intro x hs
    rcases (List.pairwise_cons.mp hs) with ⟨ha, htail⟩
    rw [Intervaltree.coalesce.eq_3, if_neg hab]
    have htailCov := ih x htail
    simpa [covered] using
      congrArg (fun q : Bool =>
        (decide (a.lo ≤ x ∧ x < a.hi) : Bool) || q) htailCov
\end{lstlisting}

\paragraph{The counterexample check rejects a false lemma.}
While planning \texttt{prove\_merge\_disjoint\_strict}, the Lemma Planner proposed that every interval produced by \texttt{insertByLo} is nonempty:

\begin{lstlisting}[style=lean]
lemma insertByLo_output_nonempty (iv : Iv) (ys : List Iv) :
    ∀ o ∈ Intervaltree.insertByLo iv ys, o.lo < o.hi
\end{lstlisting}

This is false: when \texttt{iv} is empty, \texttt{insertByLo} returns \texttt{ys} unchanged, so an empty interval already in \texttt{ys} survives.
The counterexample check rejected the lemma before any proof effort was spent on it:

\begin{lstlisting}[style=lean]
Counterexample 1:
iv := { lo := 0, hi := 0 }, ys := [{ lo := 0, hi := 0 }]
issue: refuted deductively; the statement is not evaluable (no cocheck verdict)
note: verified by tactic ladder, not by cocheck_refute
\end{lstlisting}

In the same iteration the Lemma Planner restated it with the missing invariant as a hypothesis which was exactly the form the fold in \texttt{sortByLo} needs, since the list built so far is known to contain only nonempty intervals:

\begin{lstlisting}[style=lean]
lemma insertByLo_output_nonempty (iv : Iv) (ys : List Iv)
    (hys : ∀ o ∈ ys, o.lo < o.hi) :
    ∀ o ∈ Intervaltree.insertByLo iv ys, o.lo < o.hi := by
  induction ys with
  | nil =>
    intro o ho
    unfold Intervaltree.insertByLo at ho
    split at ho
    · simp_all
    · simp_all
  | cons y ys ih =>
    have hy : y.lo < y.hi := hys y (by simp)
    have htail : ∀ z ∈ ys, z.lo < z.hi := fun z hz => hys z (by simp [hz])
    intro o ho
    unfold Intervaltree.insertByLo at ho
    split at ho
    · rename_i hp
      split at ho
      · rename_i hle
        simp only [List.mem_cons] at ho
        rcases ho with rfl | rfl | ho
        · exact hp
        · exact hy
        · exact htail o ho
      · rename_i hle
        simp only [List.mem_cons] at ho
        rcases ho with rfl | ho
        · exact hy
        · exact ih htail o ho
    · rename_i hp
      exact hys o (by simpa using ho)

lemma sortByLo_output_nonempty (ivs : List Iv) :
    ∀ o ∈ Intervaltree.sortByLo ivs, o.lo < o.hi := by
  induction ivs with
  | nil =>
    simp [Intervaltree.sortByLo]
  | cons iv ivs ih =>
    simpa [Intervaltree.sortByLo, List.foldr_cons] using
      (insertByLo_output_nonempty iv (Intervaltree.sortByLo ivs) ih)
\end{lstlisting}

\texttt{sortByLo\_output\_nonempty} is used directly by 7 declarations.

\paragraph{Canonical form of \texttt{coalesce}.}
\texttt{coalesce\_disjointSortedNonempty} is the canonical-form half of the core: coalescing a sorted list of nonempty intervals gives a canonical list.
The Proof Writer reasoned about a plan which split off a small invariant first, \texttt{coalesce\_lower\_bound} (a lower bound on every input \texttt{lo} is preserved in the output), and then planned a four-case functional induction, e.g., ``in the nonmerge branch, use \texttt{Int.lt\_of\_not\_ge hyx} to obtain \texttt{x.hi < y.lo} \ldots\ for gap index zero, the index hypothesis makes the coalesced tail nonempty; apply the reusable lemma \texttt{coalesce\_lower\_bound (y :: rest) y.lo} to its first element \ldots\ for a successor index, reduce the list indexing with \texttt{simp} and invoke IH.2 at the predecessor index'':

\begin{lstlisting}[style=lean]
lemma coalesce_disjointSortedNonempty (ivs : List Iv)
    (hnonempty : ∀ iv ∈ ivs, iv.lo < iv.hi)
    (hsorted : ivs.Pairwise (fun a b => a.lo ≤ b.lo)) :
    disjointSortedNonempty (Intervaltree.coalesce ivs) := by
  revert hnonempty hsorted
  induction ivs using Intervaltree.coalesce.induct with
  | case1 =>
    intro hnonempty hsorted
    unfold disjointSortedNonempty
    simp [Intervaltree.coalesce]
  | case2 x =>
    intro hnonempty hsorted
    unfold disjointSortedNonempty
    simp [Intervaltree.coalesce, hnonempty x (by simp)]
  | case3 x y rest hyx ih =>                       -- merge branch
    intro hnonempty hsorted
    have hx : x.lo < x.hi := hnonempty x (by simp)
    have hy : y.lo < y.hi := hnonempty y (by simp)
    have hp := List.pairwise_cons.mp hsorted
    have hxy : x.lo ≤ y.lo := hp.1 y (by simp)
    have hp' := List.pairwise_cons.mp hp.2
    let u : Iv := { lo := x.lo, hi := if x.hi < y.hi then y.hi else x.hi }
    have hu : u.lo < u.hi := by
      dsimp [u]
      split <;> omega
    have hnu : ∀ iv ∈ u :: rest, iv.lo < iv.hi := by
      intro iv hiv
      rcases List.mem_cons.mp hiv with rfl | hiv
      · exact hu
      · exact hnonempty iv (by simp [hiv])
    have hpu : List.Pairwise (fun a b : Iv => a.lo ≤ b.lo) (u :: rest) := by
      apply List.pairwise_cons.mpr
      constructor
      · intro z hz
        dsimp [u]
        exact Int.le_trans hxy (hp'.1 z hz)
      · exact hp'.2
    have hres := ih hnu hpu
    simpa [Intervaltree.coalesce, hyx, u] using hres
  | case4 x y rest hyx ih =>                       -- gap branch
    intro hnonempty hsorted
    have hx : x.lo < x.hi := hnonempty x (by simp)
    have hp := List.pairwise_cons.mp hsorted
    have hgap : x.hi < y.lo := Int.lt_of_not_ge hyx
    have htailnon : ∀ iv ∈ y :: rest, iv.lo < iv.hi := by
      intro iv hiv
      exact hnonempty iv (by simp [hiv])
    have hres := ih htailnon hp.2
    have hlower : ∀ iv ∈ y :: rest, y.lo ≤ iv.lo := by
      intro iv hiv
      rcases List.mem_cons.mp hiv with rfl | hiv
      · exact Int.le_refl _
      · exact (List.pairwise_cons.mp hp.2).1 iv hiv
    simp only [Intervaltree.coalesce, hyx, ↓reduceIte]
    unfold disjointSortedNonempty at hres ⊢
    constructor
    · intro iv hiv
      rcases List.mem_cons.mp hiv with rfl | hiv
      · exact hx
      · exact hres.1 iv hiv
    · intro i hi
      cases i with
      | zero =>
        simp only [List.getElem_cons_zero, List.getElem_cons_succ]
        simp only [List.length_cons] at hi
        have hzero : 0 < (Intervaltree.coalesce (y :: rest)).length := by
          omega
        have hzlo := coalesce_lower_bound (y :: rest) y.lo hlower
          (Intervaltree.coalesce (y :: rest))[0] (List.getElem_mem hzero)
        exact lt_of_lt_of_le hgap hzlo
      | succ i =>
        simp only [List.length_cons] at hi
        simpa using hres.2 i (by omega)
\end{lstlisting}

\paragraph{Uniqueness of canonical form: the extensionality lemma.}
The deepest mathematical step is independent of the implementation: two canonical interval lists that cover the same points are equal.
The Lemma Planner stated it as a standalone lemma in iteration 13. With it, the uniqueness, determination, and permutation-invariance targets follow in a few lines each.
The proof plan in Proof Writer is an induction on the first list, generalizing the second: (1) the heads' lower endpoints agree, because the smallest covered point of each canonical list is its head's \texttt{lo}; (2) the upper endpoints agree, because the point at a canonical head's \texttt{hi} is \emph{uncovered}, which rules out either head extending past the other; (3) the tails cover the same points, so the induction hypothesis applies.

\begin{lstlisting}[style=lean]
lemma disjointSortedNonempty_ext_of_coverage (c d : List Iv)
    (hc : disjointSortedNonempty c) (hd : disjointSortedNonempty d)
    (hcov : ∀ x : Int, covered x c = covered x d) :
    c = d := by
  induction c generalizing d with
  | nil =>
    cases d with
    | nil =>
      rfl
    | cons b ds =>
      exfalso
      have hb : b.lo < b.hi := hd.1 b (by simp)
      have hdtrue : covered b.lo (b :: ds) = true :=
        (covered_eq_true_iff b.lo (b :: ds)).2 ⟨b, by simp, le_rfl, hb⟩
      have heq := hcov b.lo
      simp [covered] at heq
      omega
  | cons a cs ih =>
    cases d with
    | nil =>
      exfalso
      have ha : a.lo < a.hi := hc.1 a (by simp)
      have hctrue : covered a.lo (a :: cs) = true :=
        (covered_eq_true_iff a.lo (a :: cs)).2 ⟨a, by simp, le_rfl, ha⟩
      have heq := hcov a.lo
      simp [covered] at heq
      omega
    | cons b ds =>
      have ha := hc.1 a (by simp)
      have hb := hd.1 b (by simp)
      have hcgap := disjointSortedNonempty_pairwise_gap (a :: cs) hc
      have hdgap := disjointSortedNonempty_pairwise_gap (b :: ds) hd
      -- every member's lo is bounded below by the head's lo
      have hclo : ∀ v ∈ a :: cs, a.lo ≤ v.lo := by
        intro v hv
        rcases List.mem_cons.mp hv with rfl | hv
        · exact le_rfl
        · have hgapv := (List.pairwise_cons.mp hcgap).1 v hv
          omega
      have hdlo : ∀ v ∈ b :: ds, b.lo ≤ v.lo := by
        intro v hv
        rcases List.mem_cons.mp hv with rfl | hv
        · exact le_rfl
        · have hgapv := (List.pairwise_cons.mp hdgap).1 v hv
          omega
      -- (1) the heads have the same lower endpoint
      have hba : b.lo ≤ a.lo := by
        have hatrue : covered a.lo (a :: cs) = true :=
          (covered_eq_true_iff a.lo (a :: cs)).2 ⟨a, by simp, le_rfl, ha⟩
        have hbtrue : covered a.lo (b :: ds) = true := by
          rw [← hcov a.lo]
          exact hatrue
        obtain ⟨v, hv, hvlo, hvhi⟩ := (covered_eq_true_iff a.lo (b :: ds)).1 hbtrue
        have hlow := hdlo v hv
        omega
      have hab : a.lo ≤ b.lo := by
        have hbtrue : covered b.lo (b :: ds) = true :=
          (covered_eq_true_iff b.lo (b :: ds)).2 ⟨b, by simp, le_rfl, hb⟩
        have hatrue : covered b.lo (a :: cs) = true := by
          rw [hcov b.lo]
          exact hbtrue
        obtain ⟨v, hv, hvlo, hvhi⟩ := (covered_eq_true_iff b.lo (a :: cs)).1 hatrue
        have hlow := hclo v hv
        omega
      have hlo : a.lo = b.lo := le_antisymm hab hba
      -- (2) a canonical head's upper endpoint is uncovered
      have hcaHi : covered a.hi (a :: cs) = false := by
        apply Bool.eq_false_of_not_eq_true
        intro ht
        obtain ⟨v, hv, hvlo, hvhi⟩ := (covered_eq_true_iff a.hi (a :: cs)).1 ht
        rcases List.mem_cons.mp hv with rfl | hv
        · omega
        · have hgapv := (List.pairwise_cons.mp hcgap).1 v hv
          omega
      have hdbHi : covered b.hi (b :: ds) = false := by
        apply Bool.eq_false_of_not_eq_true
        intro ht
        obtain ⟨v, hv, hvlo, hvhi⟩ := (covered_eq_true_iff b.hi (b :: ds)).1 ht
        rcases List.mem_cons.mp hv with rfl | hv
        · omega
        · have hgapv := (List.pairwise_cons.mp hdgap).1 v hv
          omega
      have hbia : b.hi ≤ a.hi := by
        by_contra h
        have hlt : a.hi < b.hi := lt_of_not_ge h
        have hbt : covered a.hi (b :: ds) = true :=
          (covered_eq_true_iff a.hi (b :: ds)).2 ⟨b, by simp, by omega, hlt⟩
        have hct : covered a.hi (a :: cs) = true := by
          rw [hcov a.hi]
          exact hbt
        rw [hcaHi] at hct
        simp at hct
      have haib : a.hi ≤ b.hi := by
        by_contra h
        have hlt : b.hi < a.hi := lt_of_not_ge h
        have hat : covered b.hi (a :: cs) = true :=
          (covered_eq_true_iff b.hi (a :: cs)).2 ⟨a, by simp, by omega, hlt⟩
        have hdt : covered b.hi (b :: ds) = true := by
          rw [← hcov b.hi]
          exact hat
        rw [hdbHi] at hdt
        simp at hdt
      have hhi : a.hi = b.hi := le_antisymm haib hbia
      have habeq : a = b := by
        cases a
        cases b
        simp_all
      subst b
      -- (3) the tails cover the same points; apply the IH
      have hcTail : disjointSortedNonempty cs := disjointSortedNonempty_tail a cs hc
      have hdTail : disjointSortedNonempty ds := disjointSortedNonempty_tail a ds hd
      have htailcov : ∀ x : Int, covered x cs = covered x ds := by
        intro x
        by_cases hx : a.lo ≤ x ∧ x < a.hi
        · have hcsfalse : covered x cs = false := by
            apply Bool.eq_false_of_not_eq_true
            intro ht
            obtain ⟨v, hv, hvlo, hvhi⟩ := (covered_eq_true_iff x cs).1 ht
            have hgapv := (List.pairwise_cons.mp hcgap).1 v hv
            omega
          have hdsfalse : covered x ds = false := by
            apply Bool.eq_false_of_not_eq_true
            intro ht
            obtain ⟨v, hv, hvlo, hvhi⟩ := (covered_eq_true_iff x ds).1 ht
            have hgapv := (List.pairwise_cons.mp hdgap).1 v hv
            omega
          rw [hcsfalse, hdsfalse]
        · have hhead : decide (a.lo ≤ x ∧ x < a.hi) = false :=
            decide_eq_false_iff_not.mpr hx
          have hhcov := hcov x
          change (decide (a.lo ≤ x ∧ x < a.hi) || covered x cs) =
            (decide (a.lo ≤ x ∧ x < a.hi) || covered x ds) at hhcov
          rw [hhead, Bool.false_or, Bool.false_or] at hhcov
          exact hhcov
      have htail := ih ds hcTail hdTail htailcov
      exact congrArg (fun t : List Iv => a :: t) htail
\end{lstlisting}

Combined with the two central targets, extensionality gives several targets almost for free, and the agent proved permutation and reversal invariance \emph{semantically}: a permutation preserves coverage (\ie \texttt{covered\_perm}, closed by Deterministic Automation), and extensionality turns equal coverage into equal canonical outputs, avoiding any reasoning about how \texttt{insertByLo} handles ties:

\begin{lstlisting}[style=lean]
theorem prove_merge_coverage_determines : spec_merge_coverage_determines canonical := by
  unfold spec_merge_coverage_determines
  intro a b hab
  apply disjointSortedNonempty_ext_of_coverage
  · exact prove_merge_disjoint_strict a
  · exact prove_merge_disjoint_strict b
  · intro x
    calc
      covered x (canonical.intervaltree.mergeOverlaps a) = covered x a :=
        prove_merge_coverage a x
      _ = covered x b := hab x
      _ = covered x (canonical.intervaltree.mergeOverlaps b) :=
        (prove_merge_coverage b x).symm

lemma mergeOverlaps_perm (xs ys : List Iv) (hperm : xs.Perm ys) :
    Intervaltree.mergeOverlaps xs = Intervaltree.mergeOverlaps ys := by
  apply disjointSortedNonempty_ext_of_coverage
  · simpa [spec_merge_disjoint_strict, canonical] using
      (prove_merge_disjoint_strict xs)
  · simpa [spec_merge_disjoint_strict, canonical] using
      (prove_merge_disjoint_strict ys)
  · intro x
    calc
      covered x (Intervaltree.mergeOverlaps xs) = covered x xs := by
        simpa [spec_merge_coverage, canonical] using (prove_merge_coverage xs x)
      _ = covered x ys := covered_perm x xs ys hperm
      _ = covered x (Intervaltree.mergeOverlaps ys) := by
        symm
        simpa [spec_merge_coverage, canonical] using (prove_merge_coverage ys x)

theorem prove_merge_reverse_eq : spec_merge_reverse_eq canonical := by
  unfold spec_merge_reverse_eq
  unfold canonical
  unfold Intervaltree.overlaps Intervaltree.mergeOverlaps Intervaltree.chop
  try unfold Intervaltree.coalesce
  try unfold Intervaltree.chopOne
  unfold Intervaltree.sortByLo
  intro ivs
  rw [← Intervaltree.coalesce.eq_def, ← Intervaltree.coalesce.eq_def]
  exact mergeOverlaps_perm ivs ivs.reverse (List.reverse_perm ivs).symm
\end{lstlisting}

\paragraph{The deepest chain: \texttt{chop} preserves merge fixpoints.}
\path{prove_chop_canonical_merge_fixpoint} sits at dependency depth 6, the deepest point in the proof; it combines both halves of the library (chopping a canonical list keeps it canonical; merging a canonical list is the identity):
\Cref{fig:lemmas-case-2} shows its dependency graph.

\begin{figure}[h]
  \centering
  \includegraphics[width=\linewidth]{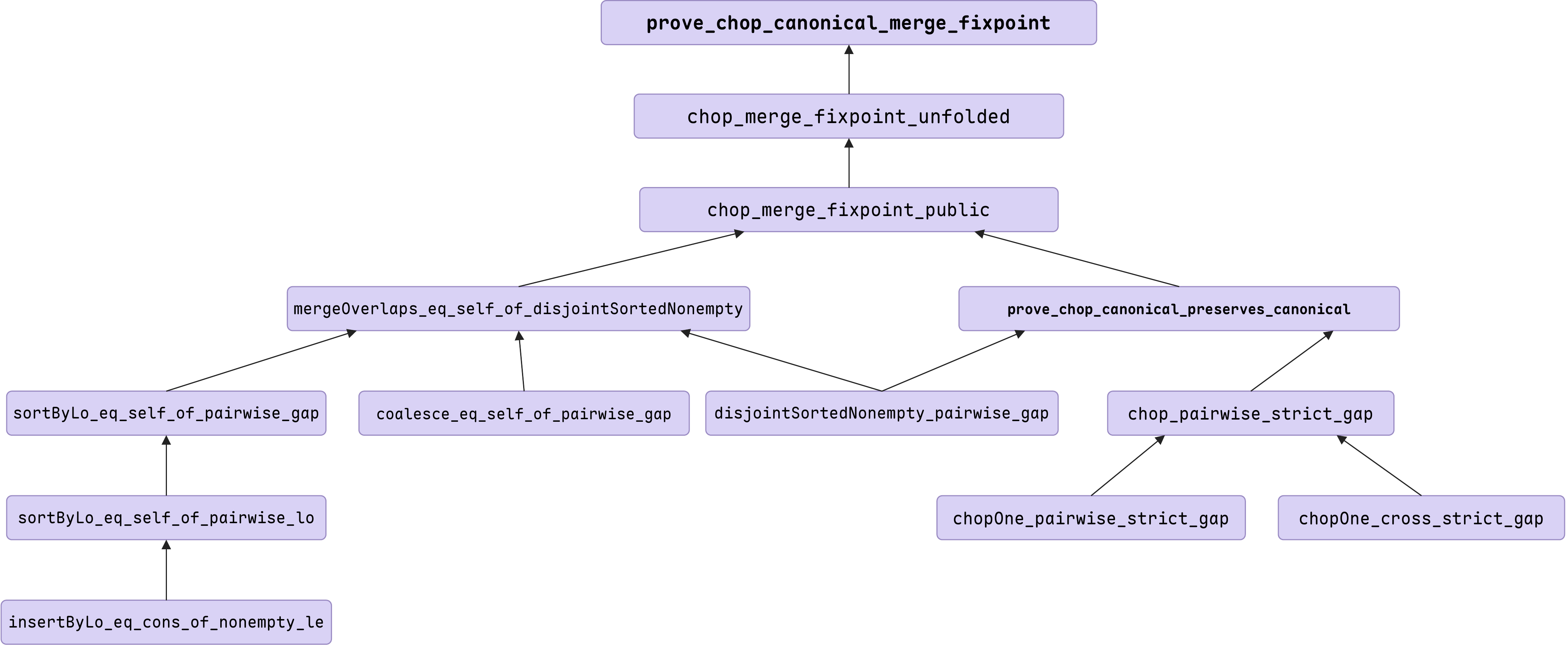}
  \caption{\textbf{Lemma-dependency graph of \texttt{prove\_chop\_canonical\_merge\_fixpoint}}, the deepest target in the Vero \texttt{intervaltree} run; edges point from a helper lemma to the declaration whose proof consumes it.
    Below the target, \texttt{chop\_merge\_fixpoint\_unfolded} bridges to the unfolded goal shape left by the automation prefix and \texttt{chop\_merge\_fixpoint\_public} restates it over the public definitions; the proof then splits into the two halves of the library, merging a canonical list is the identity (left) and chopping preserves canonicity (right).
    Two forms of reuse are visible.
    \texttt{prove\_chop\_canonical\_preserves\_canonical} (bold) is another benchmark target that this proof consumes as a lemma, and \texttt{disjointSortedNonempty\_pairwise\_gap} is a single node feeding both halves rather than a fact re-derived in each.}
  \label{fig:lemmas-case-2}
\end{figure}

Here the agent kept the mathematical statement and the syntactic shape of the goal separate.
\texttt{chop\_merge\_fixpoint\_public} is stated over the public definitions and proved in two lines from two verified results:

\begin{lstlisting}[style=lean]
lemma chop_merge_fixpoint_public (b e : Int) (ivs : List Iv)
    (hbe : b < e)
    (hnonempty : ∀ iv ∈ ivs, iv.lo < iv.hi)
    (hconsecutive : ∀ (i : Nat) (h : i + 1 < ivs.length),
      (ivs[i]).hi < (ivs[i + 1]).lo) :
    Intervaltree.mergeOverlaps (Intervaltree.chop b e ivs) =
      Intervaltree.chop b e ivs := by
  apply mergeOverlaps_eq_self_of_disjointSortedNonempty
  exact prove_chop_canonical_preserves_canonical b e ivs hbe ⟨hnonempty, hconsecutive⟩
\end{lstlisting}

The target's goal, however, had already been unfolded by the automation prefix into an explicit \texttt{match}/\texttt{flatMap} term, and an earlier attempt to \texttt{change} the goal back to the public form had failed (recorded in the plan as ``the failed earlier approach'').
The Lemma Planner therefore proposed a bridge lemma, \texttt{chop\_merge\_fixpoint\_unfolded}, stated in exactly that unfolded form and proved by unfolding the public lemma's statement the same way, which reduces the target to a single \texttt{exact}:

\begin{lstlisting}[style=lean]
theorem prove_chop_canonical_merge_fixpoint : spec_chop_canonical_merge_fixpoint canonical := by
  unfold spec_chop_canonical_merge_fixpoint
  unfold disjointSortedNonempty canonical
  unfold Intervaltree.overlaps Intervaltree.mergeOverlaps Intervaltree.chop
  try unfold Intervaltree.coalesce
  try unfold Intervaltree.chopOne
  unfold Intervaltree.sortByLo
  intro b e ivs a a_1
  simp_all
  exact chop_merge_fixpoint_unfolded b e ivs a a_1.1 a_1.2
\end{lstlisting}

\paragraph{Backtracking to an induction point.}
All 52 backtracks in the run follow one pattern: automation or an eager script unfolds a recursive definition several levels deep and loses the induction structure; the router and Recovery Reviewer recognize this, rewinds to a node where the recursive object is still abstract (\ie the nearest safe point) and proposes an induction there or a lemma to apply.
Two representative recorded reasons: for \texttt{insertByLo\_length}, ``the repeated \texttt{unfold}/\texttt{split} path has exposed several fixed recursive layers of \texttt{insertByLo} but introduced no induction hypothesis for the arbitrary remaining tail \texttt{ys}; closing the displayed goal would require proving the same length property for \texttt{ys}, whereas node 66234 permits structural induction on \texttt{ys}, yielding exactly that induction hypothesis''; and for \texttt{prove\_merge\_coverage}, ``at node 65680 both \texttt{covered} and \texttt{sortByLo} have already been unfolded, so the available coverage lemmas no longer syntactically match \ldots\ at node 65541, unfolding only the specification and the canonical record preserves it sufficiently to apply the three existing semantic lemmas directly.''
Because every backtrack keeps a record of why the abandoned path failed, later planning passes over the same region do not repeat it.

\subsubsection{Dependency relationships}
\label{app:vero-deps}

The 111 lemmas/theorems form the dependency graph of maximum depth 6, layered from single-step facts up through the pipeline:

\begin{lstlisting}[style=lean]
Layer 0: facts about one interval / one insertion step
         insertByLo_coverage, insertByLo_output_nonempty, mem_insertByLo,
         chopOne_coverage, covered_eq_true_iff, covered_perm,
         disjointSortedNonempty_tail, disjointSortedNonempty_pairwise_gap, ...
Layer 1: lifted through the fold / flatMap; functional induction on coalesce
         sortByLo_coverage, sortByLo_output_nonempty,
         coalesce_coverage_of_pairwise_lo, coalesce_disjointSortedNonempty,
         chop_pairwise_strict_gap, disjointSortedNonempty_ext_of_coverage
Layer 2: sortedness certificate
         sortByLo_pairwise_lo, mem_sortByLo_iff, sortByLo_eq_self_of_pairwise_gap
Layer 3: the central targets
         prove_merge_coverage, prove_merge_disjoint_strict, prove_merge_is_canonical,
         mergeOverlaps_eq_self_of_disjointSortedNonempty
Layer 4: consequences of coverage + canonicality (+ extensionality)
         prove_merge_canonical_unique, prove_merge_coverage_determines,
         mergeOverlaps_perm, ...
Layer 5: consequences of determination / permutation invariance
         prove_merge_reverse_eq, prove_merge_perm_eq_any, prove_merge_chop_merge, ...
Layer 6: prove_chop_canonical_merge_fixpoint
\end{lstlisting}

\begin{table}[h]
  \centering
  \caption{\textbf{The most reused declarations.}
    A handful of layer lemmas (sortedness, nonemptiness, the strict-gap conversion) feed more than half of the 56 targets, and seven proved targets are themselves reused as lemmas---this is why decomposing into lemmas scales here: each hard induction is done once and then reused.}
  \label{tab:vero-reuse}
  \small
  \resizebox{\linewidth}{!}{    \begin{tabular}{llcc}
      \toprule
      Declaration                                        & Kind   & Direct users & Targets depending on it (transitively) \\
      \midrule
      \texttt{prove\_merge\_coverage}                    & target & 15           & 18                                     \\
      \texttt{disjointSortedNonempty\_pairwise\_gap}     & lemma  & 11           & 24                                     \\
      \texttt{prove\_merge\_disjoint\_strict}            & target & 11           & 18                                     \\
      \texttt{sortByLo\_pairwise\_lo}                    & lemma  & 7            & 31                                     \\
      \texttt{sortByLo\_output\_nonempty}                & lemma  & 7            & 28                                     \\
      \texttt{prove\_chop\_pointwise}                    & target & 5            & 6                                      \\
      \texttt{prove\_merge\_coverage\_determines}        & target & 5            & 5                                      \\
      \texttt{coalesce\_disjointSortedNonempty}          & lemma  & 4            & 23                                     \\
      \texttt{disjointSortedNonempty\_ext\_of\_coverage} & lemma  & 4            & 10                                     \\
      \texttt{covered\_eq\_true\_iff}                    & lemma  & 4            & 14                                     \\
      \texttt{disjointSortedNonempty\_tail}              & lemma  & 4            & 14                                     \\
      \bottomrule
    \end{tabular}  }
\end{table}

Of the 56 targets, 10 close at depth 0 (no helper lemma), and the rest sit at depths 1--6 with direct dependency lists of one to four declarations each.

\subsubsection{What made all 56 provable}
\label{app:vero-analysis}

Five mechanisms, each visible in the trace, account for this outcome. As in the CLEVER study, each is a design component of \Cref{sec:architecture-control-loop} doing its intended job, yet here the components carry \emph{provability} rather than cost.
\emph{(1) Decomposition along the implementation---the Lemma Planner writing into the dependency DAG.}
Every \texttt{mergeOverlaps} fact was split into an \texttt{insertByLo} step, a \texttt{sortByLo} lift, and a \texttt{coalesce} fact taking a sortedness certificate. Because each layer lemma is a DAG node served to every consumer, the certificate (\texttt{sortByLo\_pairwise\_lo}) and the nonemptiness invariant (\texttt{sortByLo\_output\_nonempty}) were proved once and reused by 31 and 28 of the 56 targets, respectively.
\emph{(2) The right induction for each recursion scheme---the router's planning.}
The agent built for structural obstructions chose structural induction for insertion and the fold and the functional induction principle \texttt{coalesce.induct} for the well-founded coalescing, reverting or generalizing hypotheses so the induction hypothesis applies to the modified recursive argument; its plans then guided the Proof Writer's scripts.
\emph{(3) Semantic lemmas instead of syntactic unfolding---structure-before-repetition.}
Where redrawing on an unfolded goal could only fail again, the router escalated to statement-level structure: the extensionality lemma turns uniqueness, determination, and permutation-invariance targets into short applications of the two central results, with no reasoning about the sorting algorithm.
\emph{(4) Proved targets reused as lemmas---the DAG makes targets first-class nodes.}
Nothing in the graph distinguishes a stable target from a stable helper, so seven proved targets serve as dependencies of later targets and the proof reaches depth 6 without duplicating work (\ie cross-target reuse that no per-theorem pool, sketch, or blueprint expresses).
\emph{(5) Checking and backtracking---the counterexample check and the Recovery Reviewer.}
The check refuted a false candidate lemma before any proof effort was staked on it, and the Reviewer accepted 52 backtracks to the nearest safe node, each with a recorded reason entering the failure memory, so later planning passes over the same region never repeated a dead end.

%% file: appendix/prompts.tex
\section{Prompt Templates}
\label{app:prompts}

This section lists the prompt templates of the agents shown in \Cref{fig:router-agent-loop} and described in \Cref{sec:specialist-agents}.

\promptbox{System Prompt}{appendix/prompts/system.txt}

\promptbox{Agentic Router Prompt}{appendix/prompts/router.txt}

\promptbox{Lemma Planner Prompt}{appendix/prompts/lemma_planner.txt}

\promptbox{Proof Writer Prompt}{appendix/prompts/proof_writer.txt}

\promptbox{Recovery Reviewer Prompt}{appendix/prompts/recovery_reviewer.txt}

%% file: appendix/limitations.tex
\section{Limitations}
\label{app:limitations}

\textbf{Benchmark coverage.}
Our evaluation covers five Lean~4 benchmarks, but only two of them, NTP4VC and Vero, are repository-level, so the evidence on repository-scale lemma reuse rests on fewer tasks than the function-level comparison.
All benchmarks are in Lean~4 since we don't implement \ToolName{} on Rocq or Isabelle/HOL.

\textbf{Proof-only setting.}
Implementations and specifications are fixed inputs since \ToolName{} only operates on proofs as a theorem prover in program verification.

\textbf{Conditional theory.}
In theory, we do not prove that the implemented LLM router must attain low regret. The analysis fixes the target of the design and explains when orchestration pays off. And our experiment results demonstrate that the implemented agentic router can attain low regret.

\textbf{Reliance on LLM judgment.}
Routing, lemma proposal, planning, and recovery all depend on the semantic judgment of the backbone model.
The symbolic guards are syntactic and do not judge whether a decision is good.

\textbf{Wall-clock time and cost accounting.}
\ToolName{} is slower than the coding agents in mean wall-clock time under GPT-5.6-Terra and Gemini-3.5-Flash, which we attribute to checking each bounded invocation with Lean before choosing the next one.
Our cost model counts only model tokens at official API prices; Lean checking, retrieval, embedding, and other local computation run on our own machine and are not charged.
Dollar costs also depend on each provider's caching behavior, which means the baselines like coding agents using one single conversation will take advantage of the caching cost.